\documentclass{article}
\PassOptionsToPackage{numbers, compress}{natbib}
\usepackage[final]{neurips_2023}

\usepackage{commands}

\usepackage{microtype}
\usepackage{graphicx}
\usepackage{algorithm}
\usepackage{booktabs}
\usepackage{enumitem}
\setlist[itemize]{itemsep=2pt, topsep=0pt, partopsep=0pt, parsep=0pt}

\usepackage{adjustbox}

\usepackage[font=small,labelfont=bf]{caption}
\usepackage{subcaption}
\usepackage{wrapfig}
\usepackage{overpic}
\usepackage{array}

\usepackage{multirow}

\usepackage{hyperref}

\usepackage{tikz}
\usetikzlibrary{tikzmark,calc}

\usepackage{amsmath}
\usepackage{etoolbox}

\AtBeginEnvironment{align}{\small}
\AtBeginEnvironment{align*}{\small}

\usepackage{amssymb}
\usepackage{mathtools}
\usepackage{amsthm}

\usepackage[capitalize,noabbrev]{cleveref}

\usepackage{scalerel}

\theoremstyle{plain}
\newtheorem{theorem}{Theorem}

\newtheorem{lemma}[theorem]{Lemma}
\newtheorem{corollary}[theorem]{Corollary}
\theoremstyle{definition}

\newtheorem{assumption}[theorem]{Assumption}
\newtheorem{remark}[theorem]{Remark}

\title{Variance-Reduced Gradient Estimation via Noise-Reuse in Online Evolution Strategies} %

\author{
Oscar Li$^\S$$^1$, James Harrison$^\Diamond$, Jascha Sohl-Dickstein$^\Diamond$, Virginia Smith$^\S$, Luke Metz$^\Diamond$$^2$ \\
{\small $^\S$Machine Learning Department, School of Computer Science \;\; Carnegie Mellon University}\\
{\small $^\Diamond$Google DeepMind}
}

\begin{document}
\maketitle
\addtocounter{footnote}{1}
\footnotetext{Correspondence to: $<$oscarli@cmu.edu$>$. $^2$Now at OpenAI.}
\addtocounter{footnote}{1}

\begin{abstract}
Unrolled computation graphs are prevalent throughout machine learning but present challenges to automatic differentiation (AD) gradient estimation methods when their loss functions exhibit extreme local sensitivtiy, discontinuity, or blackbox characteristics. In such scenarios, online evolution strategies methods are a more capable alternative, while being more parallelizable than vanilla evolution strategies (ES) by interleaving partial unrolls and gradient updates.
In this work, we propose a general class of unbiased online evolution strategies methods. 
We analytically and empirically characterize the variance of this class of gradient estimators and identify the one with the least variance, which we term Noise-Reuse Evolution Strategies (NRES). Experimentally\footnote{Code available at \href{https://github.com/OscarcarLi/Noise-Reuse-Evolution-Strategies}{https://github.com/OscarcarLi/Noise-Reuse-Evolution-Strategies}.}, we show NRES results in faster convergence than existing AD and ES methods in terms of wall-clock time and number of unroll steps across a variety of applications, including learning dynamical systems, meta-training learned optimizers, and reinforcement learning.

\end{abstract}

\addtocontents{toc}{\protect\setcounter{tocdepth}{0}}

\section{Introduction}

First-order optimization methods are a foundational tool in machine learning. 
With many such methods (e.g., SGD, Adam) available in existing software, ML training often amounts to specifying a computation graph of learnable parameters and computing some notion of gradients to pass into an off-the-shelf optimizer. Here, \textit{unrolled computation graphs} (UCGs), where the same learnable parameters are repeatedly applied to transition a dynamical system’s inner state, have found their use in various applications such as recurrent neural networks \citep{hochreiter1997long, cho2014learning}, meta-training learned optimizers \citep{metz2019understanding, harrison2022a}, hyperpameter tuning \citep{maclaurin2015gradient, pmlr-v70-franceschi17a}, dataset distillation \citep{wang2018dataset, cazenavette2022dataset}, and reinforcement learning \citep{sutton1999policy, pmlr-v37-schulman15}.

While a large number of automatic differentiation (AD) methods exist to estimate gradients in UCGs~\citep{baydin2018automatic}, they often perform poorly over loss landscapes with extreme local sensitivity and cannot handle black-box computation dynamics or discontinuous losses~\citep{pmlr-v80-parmas18a, metz2019understanding, metz2021gradients}. To handle these shortcomings, evolution strategies (ES) have become a popular alternative to produce gradient estimates in UCGs \citep{salimans2017evolution}. ES methods convolve the (potentially pathological or discontinuous) loss surface with a Gaussian distribution in the learnable parameter space, making it smoother and infinitely differentiable. Unfortunately, vanilla ES methods cannot be applied online\footnote{\textit{Online} here means a method can produce gradient estimates using only \textit{a truncation window} of an unrolled computation graph \textit{instead of the full graph}, thus allowing the interleaving of partial unrolls and gradient updates.} --- the computation must reach the end of the graph to produce a gradient update, thus incurring large update latency for long UCGs.
To address this, a recently proposed approach, Persistent Evolution Strategies \citep{vicol21unbiased} ($\PES$), samples a new Gaussian noise in every truncation unroll and accumulates the past sampled noises to get rid of the estimation bias in its online application. %

In this work, we investigate the coupling of the noise sampling frequency and the gradient estimation frequency in $\PES$. By decoupling these two values, we arrive at a more general class of unbiased, online ES gradient estimators. Through a variance characterization of these estimators, we find that the one which provably has the lowest variance in fact reuses the same noise for the entire time horizon (instead of over a single truncation window as in $\PES$). We name this method \textit{Noise-Reuse Evolution Strategies} ($\NRES$). In addition to being simple to implement, $\NRES$ converges faster than $\PES$ across a wide variety of applications due to its reduced variance. Overall, we make the following contributions: 
\begin{itemize}[leftmargin=*]
    \item We propose a class of unbiased online evolution strategies gradient estimators for unrolled computation graphs that generalize Persistent Evolution Strategies~\citep{vicol21unbiased}.
    \item We analytically and empirically characterize the variance of this class of estimators and identify the lowest-variance estimator which we name Noise-Reuse Evolution Strategies ($\NRES$).
    \item We identify the connection between $\NRES$ and the existing offline ES method $\FullES$ and show that $\NRES$ is a better alternative to $\FullES$ both in terms of parallelizability and variance.
    \item We demonstrate that $\NRES$ can provide optimization convergence speedups (up to 5-60$\times$) over AD/ES baselines in terms of wall-clock time and number of unroll steps in applications of 1) learning dynamical systems, 2) meta-training learned optimizers, and 3) reinforcement learning. 
\end{itemize}

\section{Online Evolution Strategies: Background and Related Work}
\label{sec:background}

\begin{figure}[t]
\begin{minipage}{0.05\textwidth}
    ~
\end{minipage}
\begin{minipage}{0.3\textwidth}
\begin{overpic}[width=\textwidth]{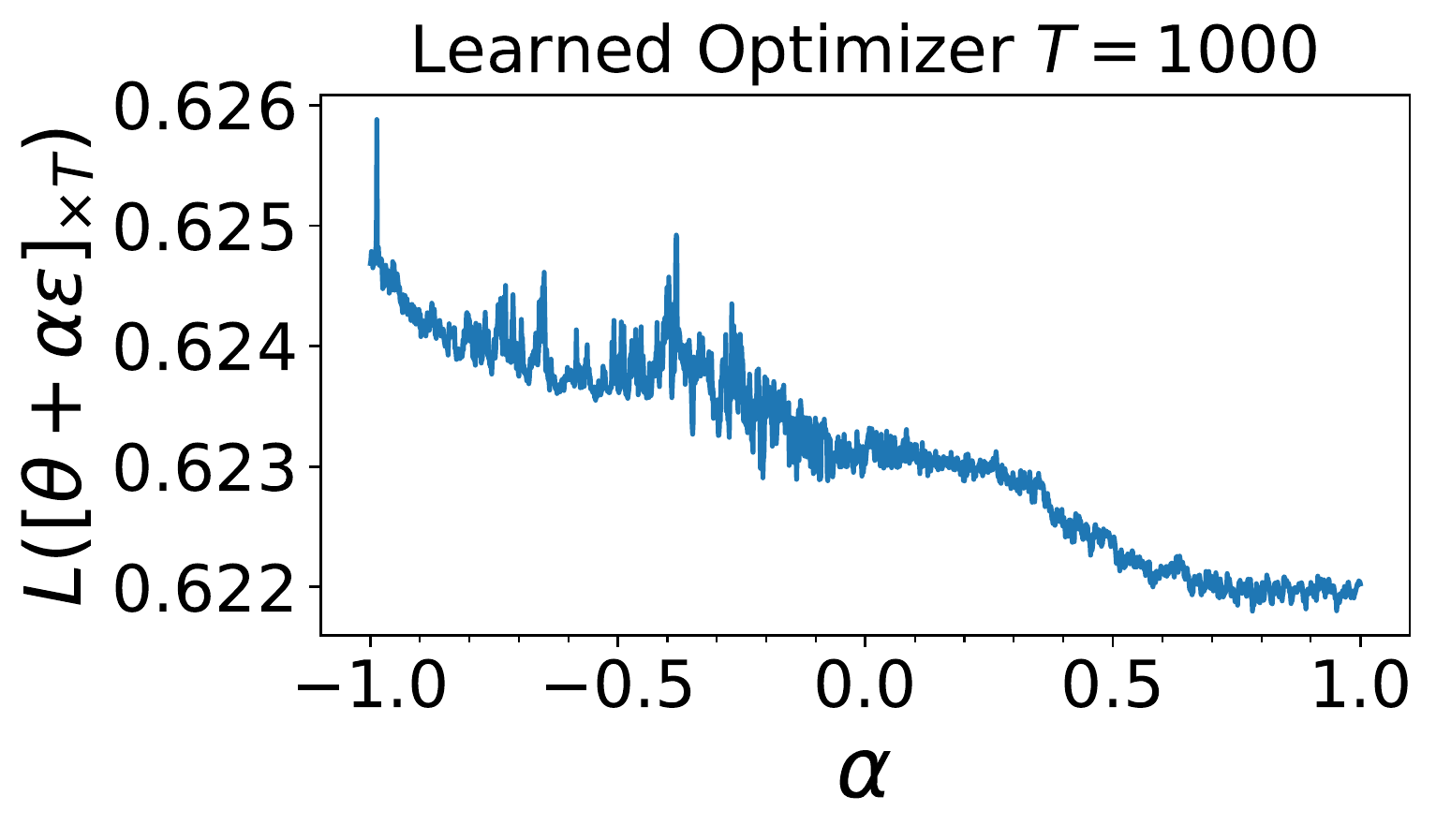}
    \put (-10,52.0) {\textbf{(a)}}
    \put (110, 52.0) {\textbf{(b)}}
\end{overpic}
\end{minipage}
\hfill
\begin{minipage}{0.6\textwidth}
    \;\;
    \small
    \begin{tabular}{l||c|c|c}
        method & online? & unbiased? & low variance?\\
     \hline

        $\FullES$ & \xmark & \cmark & \cmark \\
        $\TES$ \citep{metz2019understanding} & \cmark & \xmark & \cmark \\
        $\PES$ \citep{vicol21unbiased} & \cmark & \cmark & \xmark\\
        $\NRES$ (ours) & \cmark & \cmark & \cmark \\
    \end{tabular}
    ~

    ~

    ~
\end{minipage}
\vspace{-.05in}
\caption{(a) The pathological loss surface in the learned optimizer task (Sec.~\ref{exp:lopt}) along a random $\bepsilon$ direction; such surfaces are common in UCGs but can make automatic differentiation methods unusable, leading to the recent development of evolution strategies methods. (b) Comparison of properties of different evolution strategies methods. Unlike prior online ES methods, $\NRES$ produces both unbiased and low-variance gradient estimates.}
\label{fig:loss}
\end{figure}

\paragraph{Problem setup.} 
Unrolled computation graphs (UCGs) \cite{vicol21unbiased} are common in applications such as training recurrent neural networks, meta-training learned optimizers, and learning reinforcement learning policies, where the same set of parameters are repeatedly used to update the inner state of some system. 
We consider general UCGs where the inner state $s_t \in \Real^p$ is updated with learnable parameters $\theta \in \Real^d$ through transition functions: $\{f_t: \Real^p \times \Real^d \rightarrow \Real^p\}_{t=1}^T$, $s_{t} = f_{t}(s_{t-1}; \theta)$ for $T$ time steps starting from an initial state $s_0$.
At each time step $t \in \cbrck{1, \ldots, T}$, the state $s_t$ incurs a loss $L^s_t(s_t)$. As the loss $L^s_t$ depends on $t$ applications of $\theta$, we make this dependence more explicit with a loss function {\small $L_t: \Real^{dt} \rightarrow \Real$} and {\small $L_t([\theta]_{\times t}) \coloneqq L^s_t(s_t)$}\footnote{$[\theta]_{\times a}$, $a \in \Integer_{\ge0}$ denotes $a$ copies of $\theta$, for example: {\scriptsize $L_t(\brck{\theta}_{\times a}, \brck{\theta'}_{\times {t-a}}) \coloneqq L_t(\underbrace{\theta,\ldots, \theta}_{\textrm{$a$ times}}, \underbrace{\theta',\ldots, \theta'}_{\textrm{$(t-a)$ times}})$.}}. We aim to minimize the average loss over all $T$ time steps unrolled under the same $\theta$, {\small $\min_\theta L([\theta]_{\times{T}})$}, where
\small
\vspace{-0.5em}
\begin{align}
\label{eq:loss_definition}
    L(\theta_1, \ldots, \theta_T) \coloneqq \frac{1}{T} \sum_{t=1}^T L_t(\theta_1, \ldots, \theta_t).
\end{align}
\normalsize
\paragraph{Loss properties.} \vspace{-1.0em} 
Despite the existence of many automatic differentiation (AD) techniques 
to estimate gradients in UCGs~\citep{baydin2018automatic}, there are common scenarios where they are undesirable: 
\textbf{1)}~\textit{Loss surfaces with extreme local sensitiviy}: With large number of unrolls in UCGs, the induced loss surface is prone to high degrees of sharpness and many suboptimal local minima (see Figure~\ref{fig:loss}). This issue is particularly prevalent when the underlying dynamical system is chaotic under the parameter ($\theta$) of interest \cite{metz2021gradients}, e.g. in model-based control \citep{pmlr-v80-parmas18a} and meta-learning \cite{metz2019understanding}. In such cases, naively following the gradient may either \textit{a)} fail to converge under the normal range of learning rates (because of the conflicting gradient directions) or \textit{b)} converge to highly suboptimal solutions using a tuned, yet much smaller learning rate. \textbf{2)} \textit{Black-box or discontinuous losses}: As AD methods require defining a Jacobian-vector product (forward-mode) or a vector-Jacobian product (reverse-mode) for every elementary operation in the computation graph, they cannot be applied  when the UCG's inner dynamics are inaccessible (e.g., model-free reinforcement learning) or the loss objectives (e.g. accuracy) are piecewise constant (zero gradients).

\paragraph{Evolution Strategies.} %
Due to the issues with AD methods described above, a common alternative is to use evolution strategies (ES) to estimate gradients. Here, the original loss function is convolved with an isotropic Gaussian distribution in the space of $\theta$, resulting in an infinitely differentiable loss function with lower sharpness and fewer local minima than before ($\sigma >0 $ is a hyperparameter):
\begin{align}
    \theta \mapsto \E_{\bepsilon \sim \calN(\bzero, \sigma^2 I_{d\times d})} L([\theta + \bepsilon]_{\times T}).
    \label{eq:es_smoothing}
\end{align}
An unbiased gradient estimator of \eqref{eq:es_smoothing} is given by the likelihood ratio gradient estimator~\cite{glynn1990likelihood}: $\frac{1}{\sigma^2} L([\theta + \bepsilon]_{\times T}) \bepsilon$. This estimator only requires the loss evaluation (hence is zeroth-order) but not an explicit computation of the gradient, thus being applicable in cases when the gradients are noninformative (chaotic or piecewise constant loss) or not directly computable (black-box loss).
To reduce the variance, antithetic sampling is used and we call this estimator $\FullES$ (Algorithm~\ref{alg:fulles} in the Appendix):
{
\small
\begin{align}
    \mathrm{FullES}(\theta) \coloneqq \frac{1}{2\sigma^2} \bigg [ \frac{1}{T} \sum_{i=1}^T (L_i([\theta + \bepsilon]_{\times i}) - L_i([\theta - \bepsilon]_{\times i}) \bigg] \bepsilon.\label{eq:fulles_estimator}
\end{align}
}%
The term $\mathrm{Full}$ highlights that this estimator can only produce a gradient estimate after a \textit{full} sequential unroll from $t=0$ to $T$. We call such a full unroll an \textit{episode} following the reinforcement learning terminology. $\FullES$ can be parallelized \citep{salimans2017evolution} by averaging $N$ parallel gradient estimates using \textit{i.i.d.} $\bepsilon$'s, but is not online and can result in substantial latency between gradient updates when $T$ is large.

\paragraph{Truncated Evolution Strategies.}
To make $\mathrm{FullES}$ online, \citet{metz2019understanding} take inspiration from truncated backpropagation through time ($\TBPTT$) and propose the algorithm $\mathrm{TES}$ (see Algorithm \ref{alg:tes} in the Appendix). Unlike the stateless estimator $\FullES$, $\TES$ is stateful: $\TES$ starts from a saved state $s$ (from the previous iteration) and draws a new $\bepsilon_i$ for antithetic unrolling. To make itself online, $\TES$ only unrolls for a truncation window of $W$ steps for every gradient estimate, thus reducing the latency from $O(T)$ to $O(W)$. Analytically,
\small
\begin{align}
    \TES(\theta) \coloneqq \frac{1}{2\sigma^2W} \sum_{i=1}^W \big [L_{\btau + i}([\textcolor{gray}{\theta}]_{\scaleto{\times \btau}{4pt}}, [\orange{\theta} + \bepsilon_{(\btau / W) + 1}]_{\scaleto{\times i}{4pt}}) - L_{\btau + i}([\textcolor{gray}{\theta}]_{\scaleto{\times \btau}{4pt}}, [\orange{\theta} - \bepsilon_{(\btau / W) + 1}]_{\scaleto{\times i}{4pt}}) \big] \bepsilon_{(\btau / W) + 1}. \label{eq:tes_estimator}
\end{align}
\normalsize
Here, besides the Gaussian random variables {\small $\bepsilon_i \big \vert_{i=1}^{T/W} \iid \calN(\bzero, \sigma^2 I_{d\times d})$}, the time step $\btau$ which $\TES$ starts from is also a random variable drawn from the uniform distribution $\btau \sim \Unif\{0, W, \ldots, T - W\}$\footnote{The random variable $\btau$ will always be sampled from this uniform distribution. In addition, we will assume the time horizon $T$ can be evenly divided into truncation windows of length $W$, i.e. $T \equiv 0\; (\mathrm{mod} \;W)$.}. It is worth noting that \citet{vicol21unbiased} who also analyze online ES estimators do not take the view that the time step $\btau$ an online ES estimator starts from is a random variable; as such their analyses do not fully reflect the ``online'' nature captured in our work. When multiple online ES Workers (e.g., $\TES$ workers) run in parallel, different workers will work at different \textit{i.i.d.} time steps $\btau_i$ (which we call \textit{step-unlocked} workers) (see Figure \ref{fig:sharingstrategies}(a)). We provide the pseudocode for creating step-unlocked workers and for general online ES learning in Algorithm \ref{alg:oes} and \ref{alg:oes_training_algorithm} in the Appendix.

\begin{figure*}[t]
     \centering
     \begin{overpic}[width=0.48\textwidth]{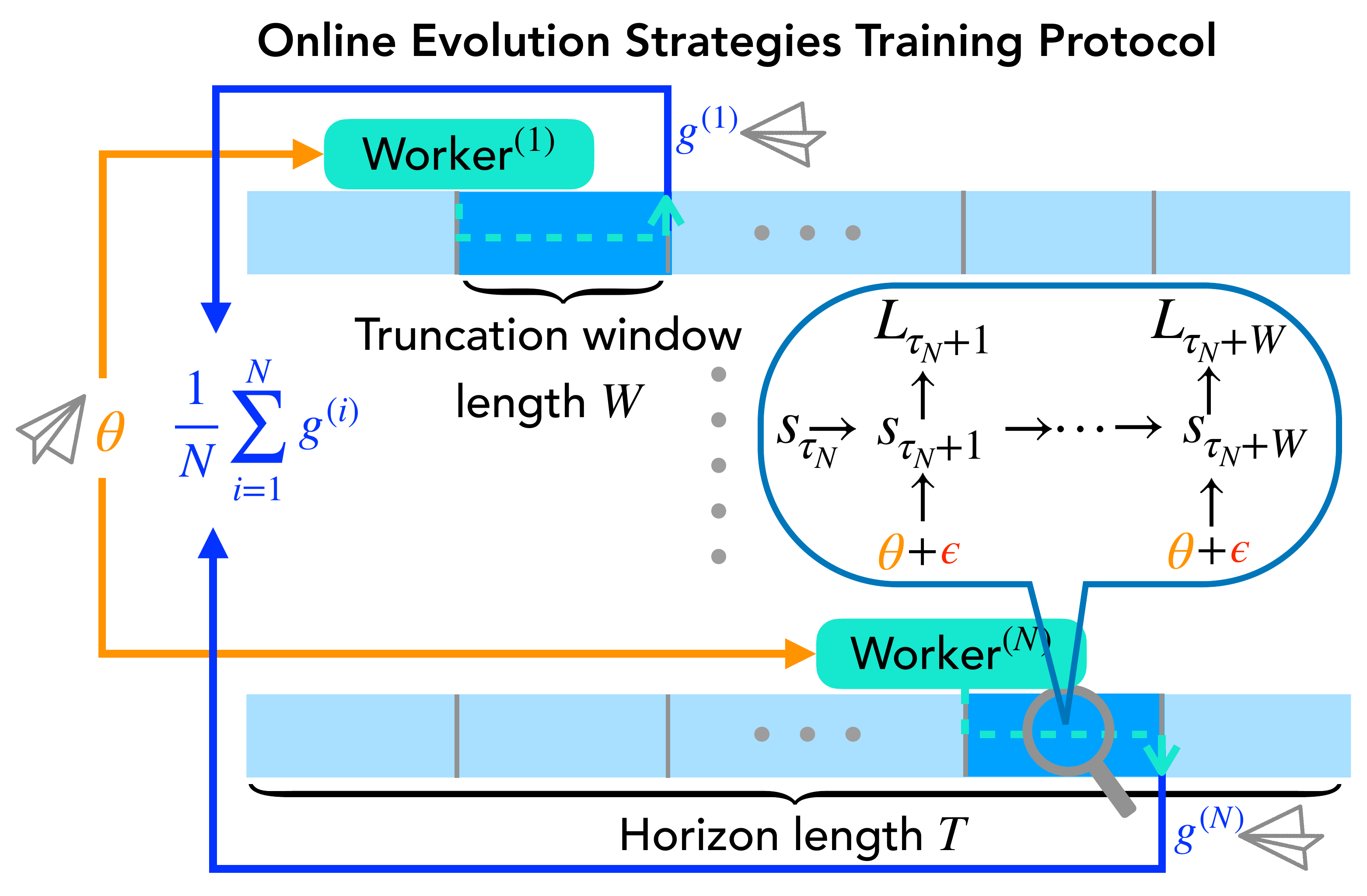}
     \put (3,61.0) {\textbf{(a)}}
    \end{overpic}
    \begin{overpic}[width=0.48\textwidth]{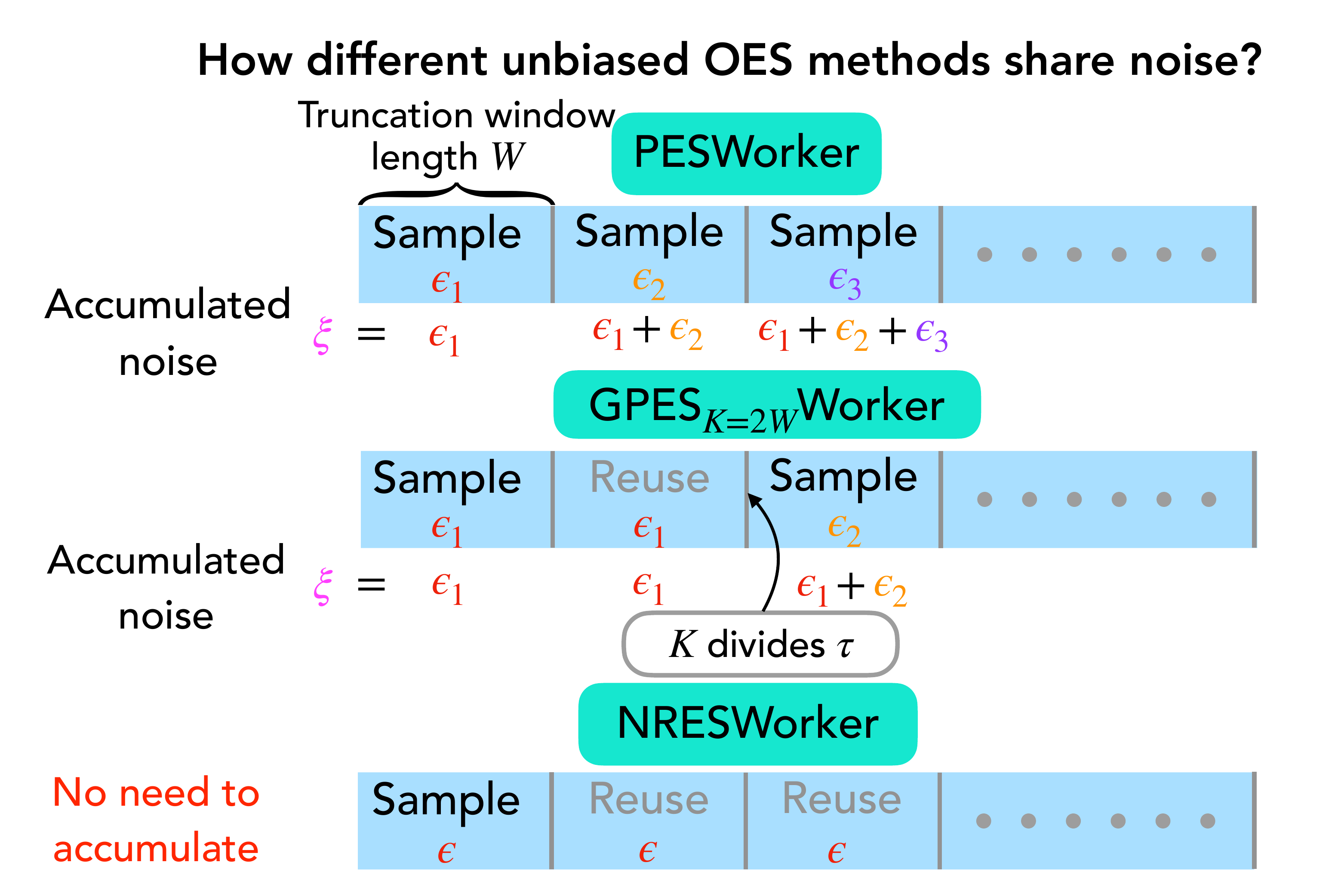}
     \put (3,61.0) {\textbf{(b)}}
    \end{overpic}
    \caption{(a) Illustration of \textit{step-unlocked} online ES workers working independently at different truncation windows. Here a central server sends \orange{$\theta$} (whose gradient to be estimated) to each worker and receives the estimates over partial unrolls from each. The \blue{averaged gradient} can then be used in a first-order optimization algorithm. (b) Comparison of the noise sharing mechanisms of $\PES$, $\GPES_{K}$, and $\NRES$ (ours). Unlike $\PES$ (and $\GPES_{K\neq T}$) which samples a new noise in every (some) truncation window and needs to accumulate the noise, $\NRES$ only samples noise once at the beginning of an episode and reuses the noise for the full episode.}
    \label{fig:protocol}
    \label{fig:sharingstrategies}
\end{figure*}
\paragraph{TES is a biased gradient estimator of \eqref{eq:es_smoothing}.} Note that in \eqref{eq:tes_estimator}, only the $\orange{\theta}$'s in the current length-$W$ truncation window receive antithetic perturbations, thus ignoring the impact of the earlier $\textcolor{gray}{\theta}$'s up to time step $\btau$. Due to this bias, optimization using $\TES$ 
typically doesn't converge to optimal solutions.

\paragraph{Persistent Evolution Strategies.}
To resolve the bias of $\TES$, \citet{vicol21unbiased} recognize that $\TES$ samples a new noise in every truncation window and modifies the smoothing objective into:
\small
\begin{align}
    \theta \mapsto \E_{\{\bepsilon_i\}}  L([\theta + \bepsilon_1]_{\scaleto{\times W}{4pt}}, \ldots, [\theta + \bepsilon_{T/W}]_{\scaleto{\times W}{4pt}}) \label{eq:pes_smoothing}.
\end{align}
\normalsize
They show that an unbiased gradient estimator of \ref{eq:pes_smoothing} is given by (see Algorithm \ref{alg:pes}):
\small
{
\setlength{\jot}{-3.0ex}
\begin{align}
    \PES(\theta) \coloneqq \frac{1}{2\sigma^2W}  \sum_{i=1}^W \bigg [ L_{\btau+i}&([\theta + \bepsilon_1]_{\scaleto{\times W}{4pt}}, \ldots, [\theta + \textcolor{violet}{\bepsilon_{(\btau/W) + 1}}]_{\scaleto{\times i}{4pt}}) \nonumber \\
    -&  L_{\btau+i}([\theta - \bepsilon_1]_{\scaleto{\times W}{4pt}}, \ldots, [\theta - \textcolor{violet}{\bepsilon_{(\btau/W) + 1}}]_{\scaleto{\times i}{4pt}}) \bigg ] \paren{(\sum_{j=1}^{\btau/W} \epsilon_j) + \textcolor{violet}{\epsilon_{\btau/W + 1}}} \label{eq:pes_estimator},
\end{align}
}%
\normalsize
with randomness in both $\{\bepsilon_i\}$ and $\btau$.  To eliminate the bias of $\TES$, instead of multiplying only with the current epsilon $\textcolor{violet}{\epsilon_{\btau/W + 1}}$, $\PES$
multiplies it with the cumulative sum of all the different \textit{iid} noise sampled so far (\self.$\boldxi$ in PESWorker). As we shall see in the next section, this accumulation of noise terms provably results in higher variance, making $\PES$ less desirable in practice. We contrast the noise sampling properties of our proposed methods with $\PES$ in Figure~\ref{fig:loss}(b).

\paragraph{Hysteresis.} 
When online gradient estimators are used in training, they often suffer from \textit{hysteresis}, or history dependence, as a result of the parameters $\theta$ changing between adjacent unrolls. 
That is, the parameter value $\theta_{0}$ that a worker uses in the current truncation window is not the same parameter $\theta_{-1}$ that was used in the previous window.
This effect is often neglected \cite{vicol21unbiased}, under an assumption that $\theta$ is updated slowly. 
To the best of our knowledge, \citep{masse2020convergence} is the only work to analyze the convergence of an online gradient estimator under hysteresis.
In the following theoretical analysis, we assume all online gradient estimates are computed without hysteresis in order to isolate the problem. However, in Section~\ref{sec:exps}, we show empirically that even under the impact of hysteresis, our proposed online estimator $\NRES$ can outperform non-online methods (e.g., $\FullES$) which don't suffer from hysteresis.

\begin{figure}[t]
    \centering
    \begin{minipage}[t]{0.48\linewidth}
    \begin{algorithm}[H]
\small
\caption{Persistent Evolution Strategies \cite{vicol21unbiased}}
\label{alg:pes}
\textbf{class} PESWorker(OnlineESWorker):

\quad \deff ~ \init(\self, $W$):

\qquad \self.$\btau = 0$; \; \self.$s^+ = s_0$; \; \self.$s^- = s_0$

\begin{tikzpicture}[remember picture, overlay]
        \draw[line width=0pt, draw=orange!30, rounded corners=2pt, fill=orange!30, fill opacity=0.3]
            (2.6, -0.12) rectangle (4.8, 0.28);
\end{tikzpicture}
\;\;\;\;\; \self.$W = W$; \; \self.$\boldxi = \bzero \in \Real^d$
\smallskip

\quad \deff gradient\_estimate(\self, $\; \theta$):

\begin{tikzpicture}[remember picture, overlay]
        \draw[line width=0pt, draw=green!30, rounded corners=2pt, fill=green!30, fill opacity=0.3]
            (0.6, -0.03) rectangle (5.2, 0.26);
\end{tikzpicture}
\;\;\;\;\; \texttt{\scriptsize \# sample at every truncation window}

\qquad $\bepsilon \sim \calN(\bzero, \sigma^2I_{d\times d})$ \texttt{\scriptsize \; \# this is \textcolor{violet}{$\epsilon_{(\btau/W)+1}$}}

\begin{tikzpicture}[remember picture, overlay]
        \draw[line width=0pt, draw=orange!30, rounded corners=2pt, fill=orange!30, fill opacity=0.3]
            (0.52, -0.12) rectangle (6.7, 0.28);
\end{tikzpicture}
\;\;\;\;\; \self.$\boldxi$ $\mathrel{+}= ~\bepsilon$ \texttt{\scriptsize \;\; \# now \self.$\boldxi = \sum_{i=1}^{\btau/W} \epsilon_i + \textcolor{violet}{\epsilon_{\btau/W + 1}}$}

~

\qquad ($s^+, s^{-}$) $=$ (\self.$s^+$, \self.$s^-$)

\qquad $L_{\mathrm{sum}}^+ = 0$; \;\; $L_{\mathrm{sum}}^- = 0$

\qquad \for $i$ \inn range(1, \self.$W$+1):

\qquad \quad $s^+ = f_{\textrm{\self}.\btau + i}(s^+, \theta + \bepsilon)$

\qquad \quad $s^- = f_{\textrm{\self}.\btau + i}(s^-, \theta - \bepsilon)$

\qquad \quad $L_{\mathrm{sum}}^+ \mathrel{+}= L^s_{\self.\btau + i}(s^+)$

\qquad \quad $L_{\mathrm{sum}}^- \mathrel{+}= L^s_{\self.\btau + i}(s^-)$

~

\begin{tikzpicture}[remember picture, overlay]
        \draw[line width=0pt, draw=orange!30, rounded corners=2pt, fill=orange!30, fill opacity=0.3]
            (5.6, -0.12) rectangle (6.4, 0.3);
\end{tikzpicture}
\;\;\;\;\; $g = (L_{\mathrm{sum}}^+ - L_{\mathrm{sum}}^-) / (2\sigma^2 \cdot \textrm{\self}.W) \cdot ~$\self.$\boldxi$

\qquad \self.$s^+ = s^+$; \self.$s^- = s^-$

\qquad \self.$\btau$ = \self.$\btau + W$

\qquad \ifff \self.$\btau \ge T$: \texttt{\scriptsize \;\# reset at the end}

\qquad \quad \self.$\btau = 0$; \self.$s^+ = s_0$; \self.$s^- = s_0$

\begin{tikzpicture}[remember picture, overlay]
        \draw[line width=0pt, draw=orange!30, rounded corners=2pt, fill=orange!30, fill opacity=0.3]
            (0.9, -0.12) rectangle (2.3, 0.28);
\end{tikzpicture}
\;\;\;\;\; \quad \self.$\boldxi = \bzero$

\qquad \return $g$

\end{algorithm}
    \end{minipage}
    \begin{minipage}[t]{0.48\linewidth}
    \input{sections/nres_algorithm.tex}
    \end{minipage}
\end{figure}
\section{A New Class of Unbiased Online Evolution Strategies Methods}
\label{sec:gpes}
As shown in Section~\ref{sec:background}, $\TES$ and $\PES$ both sample a new noise perturbation $\bepsilon$ for every truncation window to produce gradient estimates. Here we note that the \textit{frequency of noise-sharing} (new noise every truncation window of size $W$) is fixed to the \textit{frequency of gradient estimates} (a gradient estimate every truncation window of size $W$). However, the former is a choice of the smoothing objective \eqref{eq:pes_smoothing}, while the latter is often a choice of how much gradient update latency the user can tolerate. In this section \textit{we break this coupling} by introducing a general class of gradient estimators that encompass $\PES$. We then analyze these estimators' variance to identify the one with the least variance.

\paragraph{Generalized Persistent Evolution Strategies (GPES).}
For a given fixed truncation window size $W$, we consider all
\textit{noise-sharing periods} $K$ that are multiples of $W$, 
$K=cW$ for
{\small $c \in \Integer^+,\; c \le T/W $}.
$K$ being a
multiple of $W$ ensures that within each truncation window, only a single $\bepsilon$ is used. When $K=W$, we recover the $\PES$ algorithm. However, when $K$ is larger than $W$, the same noise will 
be used across adjacent truncation windows (Figure \ref{fig:sharingstrategies}(b)). With a new noise sampled every $K$ unroll steps, we define the \textit{$K$-smoothed loss objective} as the function:
{
\small
\begin{align}
\label{eq:k-smooth-objective}
\theta \mapsto \E_{\{\bepsilon_i\}}  L([\theta + \bepsilon_1]_{\scaleto{\times K}{4pt}}, \ldots, [\theta + \bepsilon_{\lceil T/K \rceil}]_{\times \remainder(T, K)}),
\footnotemark
\end{align}
\normalsize
}

where {\small $\remainder: (\Integer^+)^2 \rightarrow \Integer^+$} is the modified remainder function such that $\remainder(x, y)$ is the unique integer $n \in [1, y]$ where $x = qy + n$ for some integer $q$. This extra notation allows for the possibility that $T$ is not divisible by $K$ and the last noise $\bepsilon_{\lceil T/K \rceil}$ is used for only $\remainder(T, K) < K$ steps.\footnotetext{$\lceil x \rceil$ is smallest integer $\ge x$; $\lfloor x \rfloor$ is the largest integer $\le x$.}

We now give the analytic form of an unbiased gradient estimator of the resulting smoothed loss.\footnote{Proofs for all the Lemmas and Theorems are in Appendix \ref{app_sec:theory}.}
\begin{lemma}
\label{lemma:gpesk_form}
An unbiased gradient estimator for the $K$-smoothed loss is given by
\small
{
\setlength{\jot}{-2.0ex}
\addtolength{\belowdisplayskip}{-3.0ex}
\begin{align*}
    \GPES_{K}(\theta) \coloneqq \frac{1}{2\sigma^2W} \sum_{j=1}^W \bigg [L_{\btau + j}&([\theta + \bepsilon_1]_{\scaleto{\times K}{4pt}}, \ldots, [\theta + \darkgreen{\bepsilon_{\lfloor \btau/K \rfloor + 1}}]_{\times \scaleto{\remainder(\btau + j, K)}{5pt}})   \\
    - & L_{\btau + j}([\theta - \bepsilon_1]_{\scaleto{\times K}{4pt}}, \ldots, [\theta - \darkgreen{\bepsilon_{\lfloor \btau/K \rfloor + 1}}]_{\times \scaleto{\remainder(\btau + j, K)}{5pt}}) \bigg ] \cdot \scaleto{\paren{\sum_{i=1}^{\lfloor \btau/K \rfloor + 1} \bepsilon_i}}{25pt},
\end{align*}
}

\normalsize
with randomness in $\btau$ and $\{\bepsilon_i\}_{i=1}^{\lceil T/K \rceil}$.
\end{lemma}

\paragraph{GPES$_K$ algorithm.} Here, for the truncation window starting at step $\btau$, the noise $\darkgreen{\bepsilon_{\lfloor \btau/K \rfloor + 1}}$ is used as the antithetic perturbation to unroll the system. If $\btau$ is not divisible by $K$, then this noise has already been sampled at time step $t = \lfloor \btau/K \rfloor \cdot K$ in an earlier truncation window. Therefore, to know what noise to apply at this truncation window, we need to remember the last used $\bepsilon$ and update it when $\btau$ becomes divisble by $K$. We provide the algorithm for the $\GPES_K$ gradient estimator in Algorithm \ref{alg:gpes} in the Appendix. Note that $\GPES_{K=W}$ is the same as the $\PES$ algorithm.

\paragraph{Variance Characterization of GPES$_K$.} With this generalized class of gradient estimators $\GPES_K$, one might wonder how to choose the value of $K$. Since each estimator is an unbiased gradient estiamtor with respect to its smoothed objective, we compare the variance of these estimators as a function of $K$. To do this analytically, we make some simplifying assumptions: 

\begin{assumption}
\label{assumption:linearity}
For a given 
$\theta \in \Real^d$ and $t \in [T] \coloneqq \{1, \ldots, T\}$, there exists a set of vectors $\{g^t_i \in \Real^d\}_{i=1}^t$, such that for any $\{v_i \in \Real^d\}_{i=1}^t$, the following equality holds:
\small
\begin{align}
    L_t(\theta + v_1, \theta + v_2, \ldots, \theta + v_t) - L_t(\theta - v_1, \theta - v_2, \ldots, \theta - v_t) = 2\sum_{i=1}^t (v_i)^\top (g^t_i)
\end{align}
\normalsize
\end{assumption}
\begin{remark}
This assumption is more general than the quadratic $L_t$ assumption made in \citep{vicol21unbiased} (explanation see Appendix~\ref{app_sec:theory}). Here one can roughly understand $g^t_i$ as time step $t$'s smoothed loss's partial derivative with respect to the $i$-th application of $\theta$. For notational convenience, we let $g^t \coloneqq \sum_{i=1}^t g^t_i$ (roughly the total derivative of smoothed step-$t$ loss with respect to $\theta$) and $g^t_{K, j} \coloneqq \sum_{i=K\cdot(j-1) + 1}^{\min\{t,\; K\cdot j\}} g^t_i$ for $j \in \{1, \ldots, \lceil t/K \rceil \}$ (roughly the sum of partial derivatives of smoothed step-$t$ loss with respect to all $\theta$'s in the $j$-th noise-sharing window of size $K$ (the last window might be shorter)).
\end{remark}

With this assumption in place, we first consider the case when $W=1$ and $K=cW = c$ for $c \in [T]$. In this case, the $\GPES_K$ estimator can be simplified into the following form:
\begin{lemma}
\label{lemma:gpesk_form_under_assumption}
Under Assumption \ref{assumption:linearity}, when $W=1$, {\small $\GPES_{K=c}(\theta) = \frac{1}{\sigma^2}\sum_{j=1}^{\lfloor \btau/c \rfloor + 1} \paren{\sum_{i=1}^{\lfloor \btau/c \rfloor + 1} \bepsilon_i} \bepsilon_j^\top g_{c,j}^{\btau + 1}$}.
\end{lemma}

With this simplified form, we can now characterize the variance of the estimator $\GPES_{K=c}(\theta)$. Since it's a random vector, we analytically derive its total variance (trace of covariance matrix) $\tr(\Cov[\GPES_{K=c}(\theta)])$.

\begin{theorem}
\label{thm:variance}
When $W=1$ and under Assumption \ref{assumption:linearity}, for integer $\magenta{c} \in [T]$,
\small
\begin{align}
\tr(\Cov[\GPES_{K=c}(\theta)]) = \scaleto{\frac{(d+2)}{T} \sum_{t=1}^T \paren{ \|g^t\|_2^2} - \norm{\frac{1}{T} \sum_{t=1}^T g^t}_2^2 + \frac{1}{T} \sum_{t=1}^T \paren{\frac{d}{2}\sum_{j=1, j'=1}^{\lceil t/\magenta{c} \rceil} \norm{g^t_{\magenta{c},j} - g^t_{\magenta{c},j'}}_2^2}}{30pt}. \label{eq:variance}
\end{align}
\normalsize
\end{theorem}
To understand how the value of $K=\magenta{c}$ changes the total variance, we notice that only the nonnegative third term in \eqref{eq:variance} depends on it. This term measures the pairwise squared distance between non-overlapping partial sums $g^t_{\magenta{c}, j}$ for all $j$. When $c=T$, for every $t \in [T]$, there is only a single such partial sum as $\lceil t/c \rceil = 1$. In this case, this third term reduces to its smallest value of $0$.
Thus:
\begin{corollary}
\label{corollary:w=1}
Under Assumption \ref{assumption:linearity}, when $W=1$, the gradient estimator $\GPES_{K=T}(\theta)$ has the smallest total variance among all $\{\GPES_K: K \in [T]\}$ estimators.
\end{corollary}

\begin{remark} To understand Corollary~\ref{corollary:w=1} intuitively,  notice that at a given time step $t$ (i.e., a length-1 truncation window), any $\GPES_{K=c}$ gradient estimator ($c \in [T]$) aims to unbiasedly estimate the total derivative of the smoothed loss at this step with respect to $\theta$, which we have denoted by $g^t$. By applying a new Gaussian noise perturbation every $c < T$ steps, the $\GPES_{K=c}$ estimators \textit{indirectly estimate $g^t$} by first unbiasedly estimating the gradients inside each size-$c$ noise-sharing window: {\small $\{g^t_{c, j}\}_{j=1}^{\lceil t / c\rceil}$} and then summing up the result (notice {\small $g^t = \sum_{j=1}^{\lceil t / c\rceil} g^t_{c, j}$}). To obtain this extra (yet unused) information about the intermediate partial derivatives, these estimators require more randomness and thus suffer from a larger total variance than the $\GPES_{K=T}$ estimator which directly estimates $g^t$.
\end{remark}

\sidecaptionvpos{figure}{c}
\begin{SCfigure}
\includegraphics[width=0.4\textwidth]{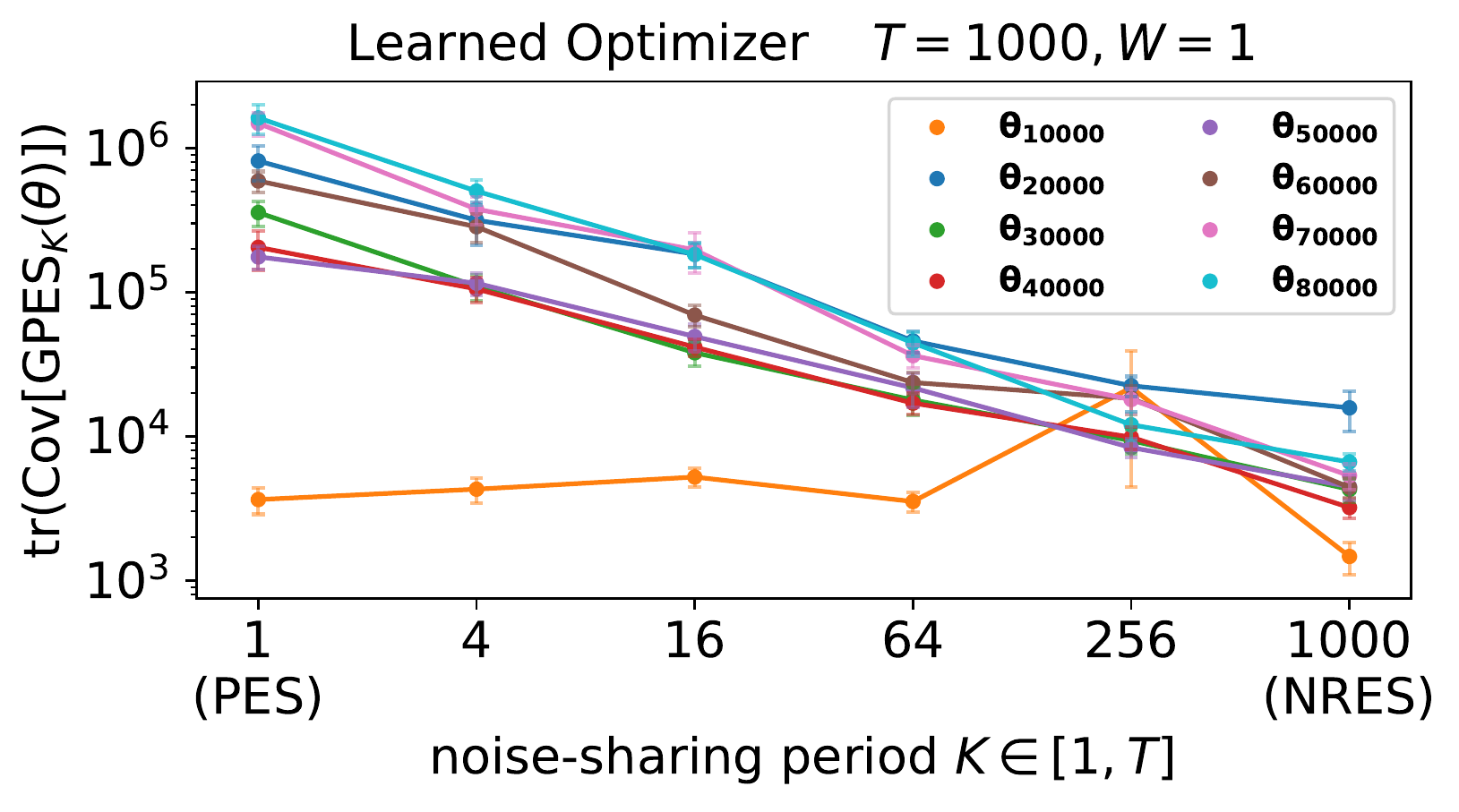}
\caption{Total variance of $\GPES_K$ vs. noise-sharing period $K$ for different $\theta_i$'s from the learned trajectory of $\PES$. $\GPES_{K=T}$ ($\NRES$) has the lowest total variance among estimators of its class (including $\PES$) for each $\theta_i$.}
\vspace{-.1in}
\label{fig:variance}
\end{SCfigure}
\paragraph{Experimental Verification of Corollary \ref{corollary:w=1}.}
We empirically verify Corollary \ref{corollary:w=1} on a meta-training learned optimizer task ($T=1000$; see additional details in Section \ref{exp:lopt}). Here we save a trajectory of $\theta_i$ learned by $\PES$ ($i$ denotes training iteration) and compute the total variance of the estimated gradients (without hysteresis) by $\GPES$ with different values of $K$ in Figure \ref{fig:variance}(a) ($W=1$). 
In agreement with theory, $K=T$ has the lowest variance.

\section{Noise-Reuse Evolution Strategies}
\label{sec:nres}
\paragraph{NRES has lower variance than PES.} As variance reduction is desirable in stochastic optimization \citep{wang2013variance}, by Corollary \ref{corollary:w=1}, the gradient estimator $\GPES_{K=T}$ is particularly attractive and \textit{can serve as a variance-reduced replacement for $\PES$}. When $K=T$, we only need to sample a single $\bepsilon$ once at the beginning of an episode (when $\btau=0$) and reuse the same noise for the entirety of that episode before it resets. This removes the need to keep track of the cumulative applied noise ($\boldxi$) (Figure~\ref{fig:sharingstrategies}(b)), 
making the algorithm simpler and more memory efficient 
than $\PES$. Due to its noise-reuse property, we name this gradient estimator $\GPES_{K=T}$ the \textit{Noise-Reuse Evolution Strategies} ($\NRES$) (pseudocode in Algorithm~\ref{alg:nres}). Concurrent with our work, \citet{vicol2023low} independently proposes a similar algorithm with different analyses. We discuss in detail how our work differs from \citep{vicol2023low} in Appendix~\ref{app_sec:related_work}. %
Despite Theorem \ref{thm:variance} assuming $W=1$, one can relax this assumption to any $W$ that divides the horizon length $T$. By defining a ``mega'' UCG whose single transition step is equivalent to $W$ steps in the original UCG, we can apply Corollary \ref{corollary:w=1} to this mega UCG and arrive at the following result.
\begin{corollary}
\label{corollary:w>1}
    Under Assumption \ref{assumption:linearity}, when $W$ divides $T$, the NRES gradient estimator has the smallest total variance among all $\GPES_{K=cW}$ estimators $c \in [T/W]$.
\end{corollary}

\paragraph{NRES is a replacement for FullES.}
By sharing the same noise over the entire horizon, the smoothing objective of $\NRES$ is the same as $\FullES$'s. Thus, we can think of $\NRES$ as the 
online counterpart to the offline algorithm $\FullES$. 
Hence $\NRES$ can act as a drop-in replacement to $\FullES$ in UCGs.
A single $\FullES$ worker runs $2T$ unroll steps for each gradient estimate, while a single $\NRES$ runs only $2W$ steps. 
Motivated by this, 
we compare the average of $T/W$ \textit{i.i.d.} $\NRES$ gradient estimates with 1 $\FullES$ gradient estimate as they require the same amount of compute.

\paragraph{NRES is more parallelizable than FullES.} Because the $T/W$ $\NRES$ gradient estimators are independent of each other, we can run them in parallel. Under perfect parallelization, the entire $\NRES$ gradient estimation would require $O(W)$ time to complete. In contrast, the single $\FullES$ gradient estimate has to traverse the UCG from start to finish, thus requiring $O(T)$ time. Hence, $\NRES$ is $T/W$ times more parallelizable than $\FullES$ under the same compute budget (Figure~\ref{fig:nres_vs_fulles}(a)).

\begin{figure}[t]
\centering
\begin{minipage}{0.52\textwidth}
    \begin{overpic}[width=\textwidth]{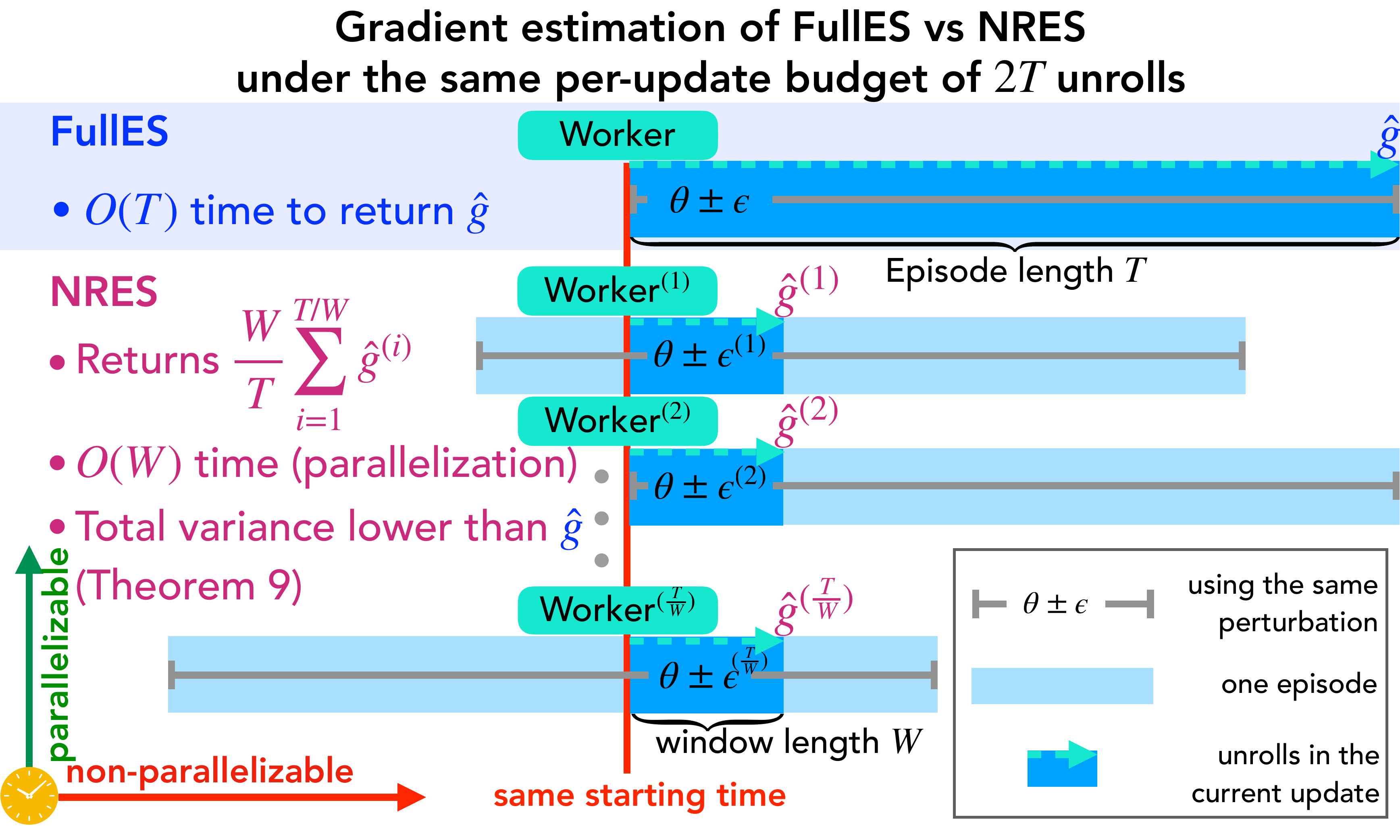}
     \put (17.0,57.0) {\textbf{(a)}}
    \end{overpic}
\end{minipage}
\begin{minipage}{0.46\textwidth}
\fontsize{7pt}{8pt}\selectfont
\begin{tabular}{c|c|c}
    \multicolumn{3}{c}{\scaleto{\textbf{(b)}}{8pt} \quad $T=1000$, $W=1$ \; \; ($\pm 95\%$ ci)} \\
    \hline
     & \multicolumn{2}{c}{total variance $\tr(\Cov[\;\cdot\;])$} \\
    \cline{2-3}
    \fontsize{6.5pt}{15pt}\selectfont
    iter. $j$ & \multirow{3}{*}{$\frac{W}{T} \scaleto{\sum\limits_{i=1}^{T/W} \NRES_i(\theta_j)}{19pt}$} & \multirow{3}{*}{$\scaleto{\FullES(\theta_j)}{7pt}$} \\
    of $\theta_j$ & & \\
   ($\scaleto{\times}{4pt}$10$\scaleto{^4}{6pt}$) & & \\
\hline
\hline
    \fontsize{6.5pt}{10pt}\selectfont
     1 & $\mathbf{1.47} \pm 0.38$ & $864.34 \pm 385.01$\\
     2 & $\mathbf{15.74} \pm 4.90$ & $513.50 \pm 74.71$ \\
     3 & $\mathbf{4.28} \pm 0.78$ & $201.38 \pm 30.18$ \\
     4 & $\mathbf{3.20} \pm 0.50$ & $684.85 \pm 86.90$ \\
     5 & $\mathbf{4.47} \pm 0.88$ & $181.06 \pm 24.97$ \\
     6 & $\mathbf{4.42} \pm 0.68$ & $288.61 \pm 46.77$ \\
     7 & $\mathbf{5.33} \pm 1.08$ & $448.30 \pm 62.44$ \\
     8 & $\mathbf{6.62} \pm 0.89$ & $154.86 \pm 28.21$ \\
\end{tabular}
\end{minipage}
\caption{
(a) Comparison of $\FullES$ and $\NRES$ gradient estimation under the same unroll budget. Unlike $\FullES$ which can only use a single noise perturbation $\bepsilon$ to unroll sequentially for an entire episode of length $T$, $\NRES$ can use $T/W$ parallel step-unlocked workers each unrolling inside its random truncation windows of length $W$ with independent perturbations {\scriptsize $\bepsilon^{(i)}$}. This results in a $T/W\times$ speed-up and variance reduction (Theorem~\ref{thm:nres_fulles_comparison}) over $\FullES$. (b) The total variance of $\NRES$ and $\FullES$ estimators under the same compute budget at the same set of $\theta_i$ checkpoints in Figure~\ref{fig:variance}(a). $\NRES$ achieves significantly lower total covariance.}
\label{fig:nres_vs_fulles}
\end{figure}

\paragraph{NRES can often have lower variance than FullES.} We next compare the variance of 
the average of $T/W$ \textit{i.i.d.} $\NRES$ gradient estimates with the variance of 1 $\FullES$ gradient estimate:
\begin{theorem}
\label{thm:nres_fulles_comparison}
Under Assumption \ref{assumption:linearity}, for any $W$ that divides $T$, if
{
\small
\begin{align}
\label{eq:direction_assumption}
    \scaleto{\sum_{k=1}^{T/W} \norm{\sum_{t=W\cdot (k-1) + 1}^{W\cdot k} g^t}_2^2 \le \frac{d+1}{d+2}\norm{\sum_{j=1}^{T/W} \sum_{t=W\cdot (k-1) + 1}^{W\cdot k} g^t}_2^2}{30pt},
\end{align}
}
then for iid $\{\NRES_i(\theta)\}_{i=1}^{T/W}$  estimators,\;\; {\small $\tr(\Cov(\frac{1}{T/W}  \sum_{i=1}^{T/W} \NRES_i(\theta)) \le \tr(\Cov(\FullES(\theta))$}.
\end{theorem}
\begin{remark} To understand the inequality assumption in \eqref{eq:direction_assumption}, we notice that it relates the sum of the squared $2$-norm of vectors {\small $\{\sum_{t=W\cdot (k-1) + 1}^{W\cdot k} g^t\}_{k=1}^{T/W}$} with the squared $2$-norm of their sum. When these vectors are pointing in similar directions, this inequality would hold (to see this intuitively, consider the more extreme case when all these vectors are exactly in the same direction). Because each term {\small $\sum_{t=W\cdot (k-1) + 1}^{W\cdot k} g^t$} can be understood as the total derivative of the sum of smoothed losses in the $k$-th truncation window with respect to $\theta$, we see that inequality \eqref{eq:direction_assumption} is satisfied when, roughly speaking, different truncation windows' gradient contributions are
pointing in similar directions. This is often the case for real-world applications because if we can decrease the losses within a truncation window by changing the parameter $\theta$, we likely will also decrease other truncation windows' losses.
At a high-level, Theorem \ref{thm:nres_fulles_comparison} shows that for many practical unrolled computation graphs, $\NRES$ is not only \textit{better than $\FullES$ due to its better parallelizability} but also \textit{better due to its lower variance} given the same computation budget.
\end{remark}

\paragraph{Empirical Verification of Theorem \ref{thm:nres_fulles_comparison}.}
We empirically verify Theorem \ref{thm:nres_fulles_comparison} in Figure \ref{fig:nres_vs_fulles}(b) using the same set up of the meta-training learned optimizer task used in Figure \ref{fig:variance}(a). Here we compare the total variance of averaging $T/W=1000$ \textit{i.i.d.} $\NRES$ estimators versus using $1$ $\FullES$ gradient estimator (same total amount of compute). We see that $\NRES$ has a significantly lower total variance than $\FullES$ while also allowing $T/W=1000$ times wall-clock speed up due to its parallelizability.

\section{Experiments}
\label{sec:exps}
$\NRES$ is particularly suitable for optimization in UCGs in two scenarios: 1) when the loss surface exhibits extreme local sensitivity; 2) when automatic differentiation of the loss is not possible/gives noninformative (e.g., zero) gradients. In this section, we focus on three applications exhibiting these properties: a) learning Lorenz system's parameters (sensitive), b) meta-training learned optimizers (sensitive), and c) reinforcement learning (nondifferentiable), and show that $\NRES$ outperforms existing AD and ES methods for these applications. When comparing online gradient estimation methods, we keep the number of workers $N$ used by all methods the same for a fair comparison. For the offline method $\FullES$, we choose its number of workers to be $W/T \times$ the number of $\NRES$ workers on all tasks (in order to keep the same number of unroll steps per-update) except for the learned optimizer task in Section~\ref{exp:lopt} where we show that $\NRES$ can solve the task faster while using much fewer per-update steps than $\FullES$.

\vspace{-0.75em}
\subsection{Learning dynamical system parameters}
\label{subsec:lorenz}
\vspace{-0.5em}

In this application, we consider learning the parameters of a Lorenz system, a canonical chaotic dynamical system. Here the state $s_t = (x_t, y_t, z_t) \in \Real^3$ is unrolled with two learnable parameters $a, r$\footnote{We don't learn the third parameter, fixed at $8/3$, so that we can easily visualize a 2-d loss surface.} with the discretized transitions ($dt = 0.005$) starting at $s_0 = (x_0, y_0, z_0) = (1.2, 1.3, 1.6)$:
\small
\begin{align}
    x_{t+1} = x_{t} + a(y_t - x_t)dt; \;\;
    y_{t+1} = y_{t} + [x_t \cdot (r - z_t) - y_t]dt; \;\;
    z_{t+1} = z_t + [x_t \cdot y_t - 8/3 \cdot z_t]dt. \nonumber 
\end{align}
\normalsize
Due to the positive constraint on $r >0$ and $a>0$, we parameterize them 
as
$\theta = (\ln(r), \ln(a)) \in \Real^2$ and exponentiate the values in each application. 
We assume we observe the
\underline{g}round \underline{t}ruth $z$-coordinate {\small $z^{\textrm{gt}}_t$} for {\small $t\in [T], T=2000$} steps unrolled by the default parameters {\small $(r^{\textrm{gt}}, a^{\textrm{gt}})=(28, 10)$}. For each step $t$, we measure the 
squared loss {\small $L^s_t(s_t) \coloneqq (z_t - z^{\textrm{gt}}_t)^2$}. Our goal is to recover the ground truth parameters {\small$\theta_{\mathrm{gt}} = (\ln(28), \ln(10))$} by optimizing the average loss over all time steps using vanilla SGD. We first visualize the training loss surface in the left panel of Figure \ref{fig:lorenz}(a) (also see Figure~\ref{app_fig:lorenz_loss_surface} in the Appendix) and notice that it has extreme sensitivity to small changes in the parameter $\theta$.

\begin{figure*}[t]
    \centering
    \begin{overpic}[width=0.6\textwidth,percent]{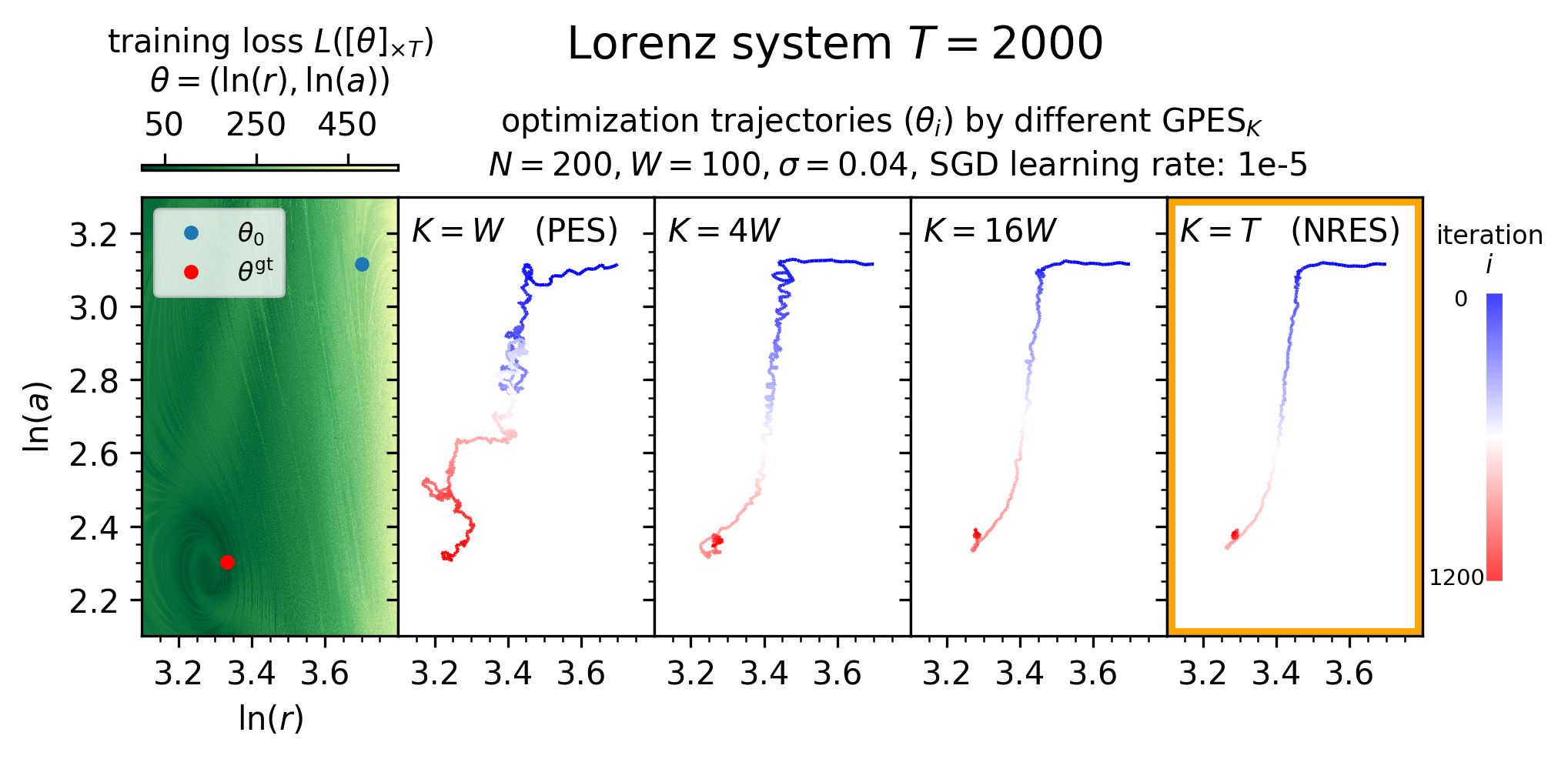}
        \put(0.0,47){\textbf{(a)}}
    \end{overpic}
    \begin{overpic}[width=0.38\textwidth,percent]{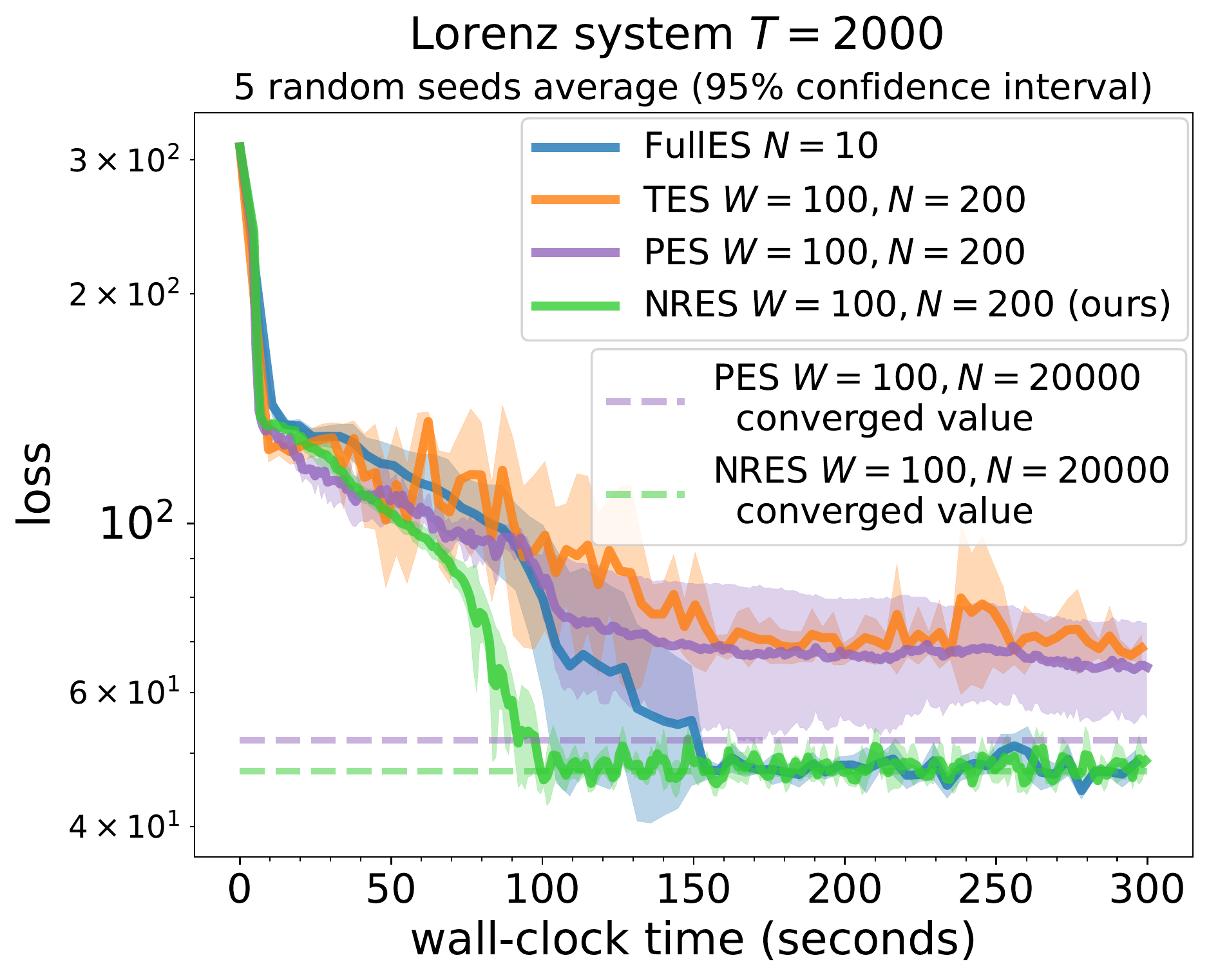}
        \put(5.0,75){\textbf{(b)}}
    \end{overpic}
    \caption{(a) The pathological training loss surface of the Lorenz system problem (left) and the optimization trajectory of different $\GPES_K$ gradient estimators (right). $\NRES$'s trajectory is the smoothest because of its lowest variance. (b) Different ES methods' loss convergence on the same problem. $\NRES$ converges the fastest.}
    \label{fig:lorenz}
\end{figure*}

To illustrate the superior variance of $\NRES$ over other $\GPES_K$ estimators, we plot in the right panels of Figure \ref{fig:lorenz}(a) the optimization trajectory of $\theta$ using gradient estimator $\GPES_K$ with different values of $K$ under the same SGD learning rate. We see that $\NRES$'s trajectory exhibits the least amount of oscillation due to its lowest variance.
In contrast, we notice that $\PES$'s trajectory is highly unstable, thus requiring a smaller learning rate than $\NRES$ to achieve a possibly slower convergence. Hence, we take extra care in tuning each method's constant learning rate and additionally allow $\PES$ to have a decay schedule. We plot the convergence of different ES gradient estimators in wall-clock time using the same hardware in Figure \ref{fig:lorenz}(b). (We additionally compare against automatic differentiation methods in Figure~\ref{app_fig:lorenz_ad_es_comparison} in the Appendix; they all perform worse than the ES methods shown here.) 

In terms of the result, we see that $\NRES$ outperforms \textbf{1)} $\TES$, as $\NRES$ is unbiased and can better capture long-term dependencies; \textbf{2)} $\PES$, as $\NRES$ has provably lower variance, which aids convergence in stochastic optimization; \textbf{3)} $\FullES$, as $\NRES$ can produce more gradient updates in the same amount of wall clock time than $\FullES$ (with parallelization, each $\NRES$ update takes $O(W)$ time instead of $\FullES$'s $O(T)$ time). Additionally, we plot the asymptotically converged loss value when we train with a significantly larger number of particles ($N=20000$) for $\PES$ and $\NRES$. We see that by only using $N=200$ particles, $\NRES$ can already converge around its asymptotic limit, while $\PES$ is still far from reaching its limit within our experiment time.

\subsection{Meta-training learned optimizers}
\label{exp:lopt}

In this application \cite{metz2019understanding}, the meta-parameters $\theta$ of a learned optimizer control the gradient-based updates of an inner model's parameters. The inner state $s_t$ is the optimizer state which consists of both the inner model's parameters and its current gradient momentum statistics. The transition function $f_t$ computes an additive update vector to the inner parameters using $\theta$ and a random training batch and outputs the next optimizer state $s_{t+1}$.
Each time step $t$'s meta-loss $L_t^s$ evaluates the updated inner parameters' generalization performance using a sampled validation batch.

We consider meta-training the learned optimizer model given in \citep{metz2019understanding} ($d=1762$) to optimize a 3-layer MLP on the Fashion MNIST dataset for $T=1000$ steps. (We show results on the same task with higher-dimension and longer horizon in Appendix \ref{app_subsubsec:lopt}.) This task is used in the training task distribution of the state of the art learned optimizer VeLO \citep{metz2022velo}\footnote{We show the performance of ES methods on another task from this distribution in Appendix~\ref{app_subsubsec:lopt}.}.
The loss surface for this problem has high sharpness and many suboptimal minima as previously shown in Figure \ref{fig:loss}(a). We meta-train with Adam using different gradient estimation methods with the same hardware and tune each gradient estimation method's meta learning rate individually.
Because AD methods all perform worse than the ES methods, we defer their results to Figure~\ref{app_fig:lopt_ad_es_comparison_fashion_mnist} in the Appendix and 
only plot the convergence of the ES methods in wall-clock time in Figure \ref{fig:lopt_full_comparison}(a).

\begin{figure*}[t]
     \centering
    \begin{overpic}[width=0.4\textwidth, percent]{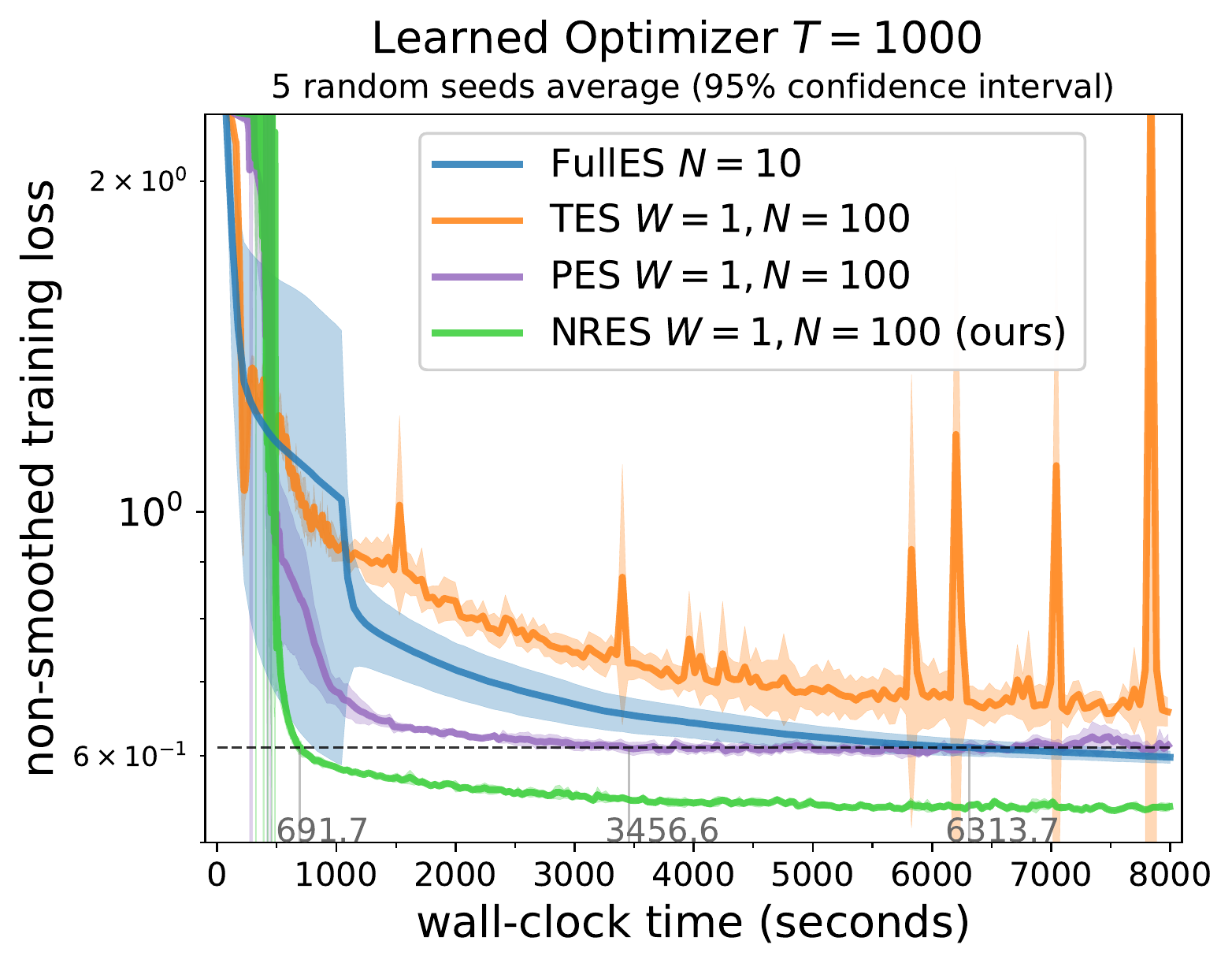}
     \put(7.0,74){\textbf{(a)}}
    \end{overpic}
    \begin{overpic}[width=0.4\textwidth, percent]{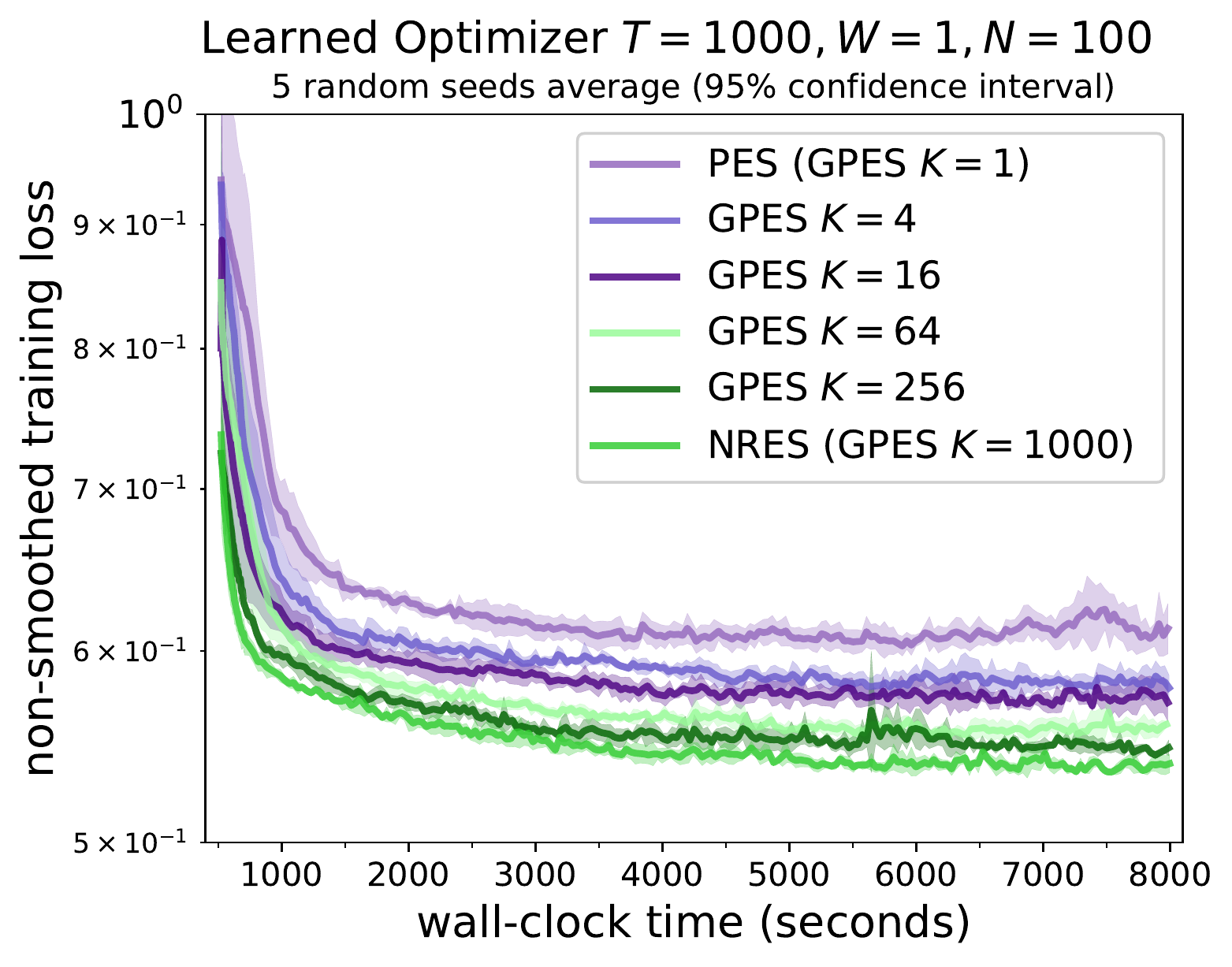}
     \put(7.0,74){\textbf{(b)}}
    \end{overpic}
    \caption{(a) Different ES gradient estimators' training loss convergence on the learned optimizer task in wall-clock time. $\NRES$ reaches the lowest loss fastest. (b) The loss convergence of $\GPES_K$ gradient estimators with difference $K$ values on the same task. $\NRES$ converges the fastest due to its reduced variance.}
    \label{fig:lopt}
    \label{fig:lopt_gpes_comparison}      \label{fig:lopt_full_comparison}
\end{figure*}

Here we see that $\NRES$ reaches the lowest loss value in the same amount of time.
In fact, $\PES$ and $\FullES$ would require $5$ and $9\times$ (respectively) longer than $\NRES$ to reach a loss $\NRES$ reaches early on during its training, while $\TES$ couldn't even reach that loss within our experiment time. It is worth noting that, for this task, $\NRES$ only require $2 \cdot N \cdot W = 200$ unrolls to produce an unbiased, low-variance gradient estimate, which is even smaller than the length of a single episode ($T=1000)$. In addition, we situate $\NRES$'s performance within our proposed class of $\GPES_K$ estimators in Figure \ref{fig:lopt_gpes_comparison}(b).
In accordance with Corollary \ref{corollary:w>1}, $\NRES$ converges fastest due to its reduced variance.
\subsection{Reinforcement Learning}
\label{exp:rl}

It has been shown that ES is a scalable alternative to policy gradient and value function methods for solving reinforcement learning tasks \citep{salimans2017evolution}. In this application, we learn a linear policy \footnote{We additionally compare the ES methods on learning a non-linear ($d=726$) policy on the Half-Cheetah task in Appendix~\ref{app_subsubsec:rl}.} (following \citep{mania2018simple}) using different ES methods on the Mujoco \citep{todorov2012mujoco} Swimmer ($d=16$) and Half Cheetah task ($d=102$). We minimize the average of negative per-step rewards over the horizon length of $T=1000$, which is equivalent to maximizing the undiscounted sum rewards. Unlike \citep{salimans2017evolution,mania2018simple,vicol21unbiased}, we don't use additional heuristic tricks such as \textbf{1)} rank conversion of rewards, \textbf{2)} scaling by loss standard deviation, or \textbf{3)} state normalization. Instead, we aim to compare the pure performance of different ES methods assuming perfect parallel implementations. To do this, we measure a method's performance as a function of the number of \textit{sequential environment steps} it used. Sequential environment steps are steps that \textit{have to} happen one after another (e.g., the environment steps within the same truncation window). However, steps that are parallelizable don't count additionally in the sequential steps. Hence, the wall-clock time under perfect parallel implementation is linear with respect to the number of sequential environment steps used. As all the methods we compare are iterative update methods, we additionally require that each method use the same number of environment steps per update when measuring each method’s required number of sequential steps to solve a task. We tune the SGD learning rate individually for each method and plot their total rewards progression on both Mujoco tasks in Figure \ref{fig:rl_reward_progression}.

Here we see that $\TES$ fails to solve both tasks due to the short horizon bias \cite{wu2018understanding}, making it unable to capture the long term dependencies necessary to solve the tasks. On the other hand, $\PES$, despite being unbiased, suffers from high variance, making it take longer (or unable in the case of Half Cheetah) to solve the task than $\NRES$. As for $\FullES$, despite using the same amount of compute per gradient update as $\NRES$, it's much less parallelizable as discussed in Section \ref{sec:nres} -- it takes much longer time ($10\times$ and $60\times$) than $\NRES$ assuming perfect parallelization. In addition to the number of sequential steps, we additionally show the total number of environment steps used by each method in Table \ref{app_tab:rl_total_env_steps} in the Appendix -- $\NRES$ \textit{also uses the least total number of steps} to solve both tasks, making it the most sample efficient.

\begin{figure*}[t]
    \centering
    \begin{overpic}[width=0.38\textwidth, percent]{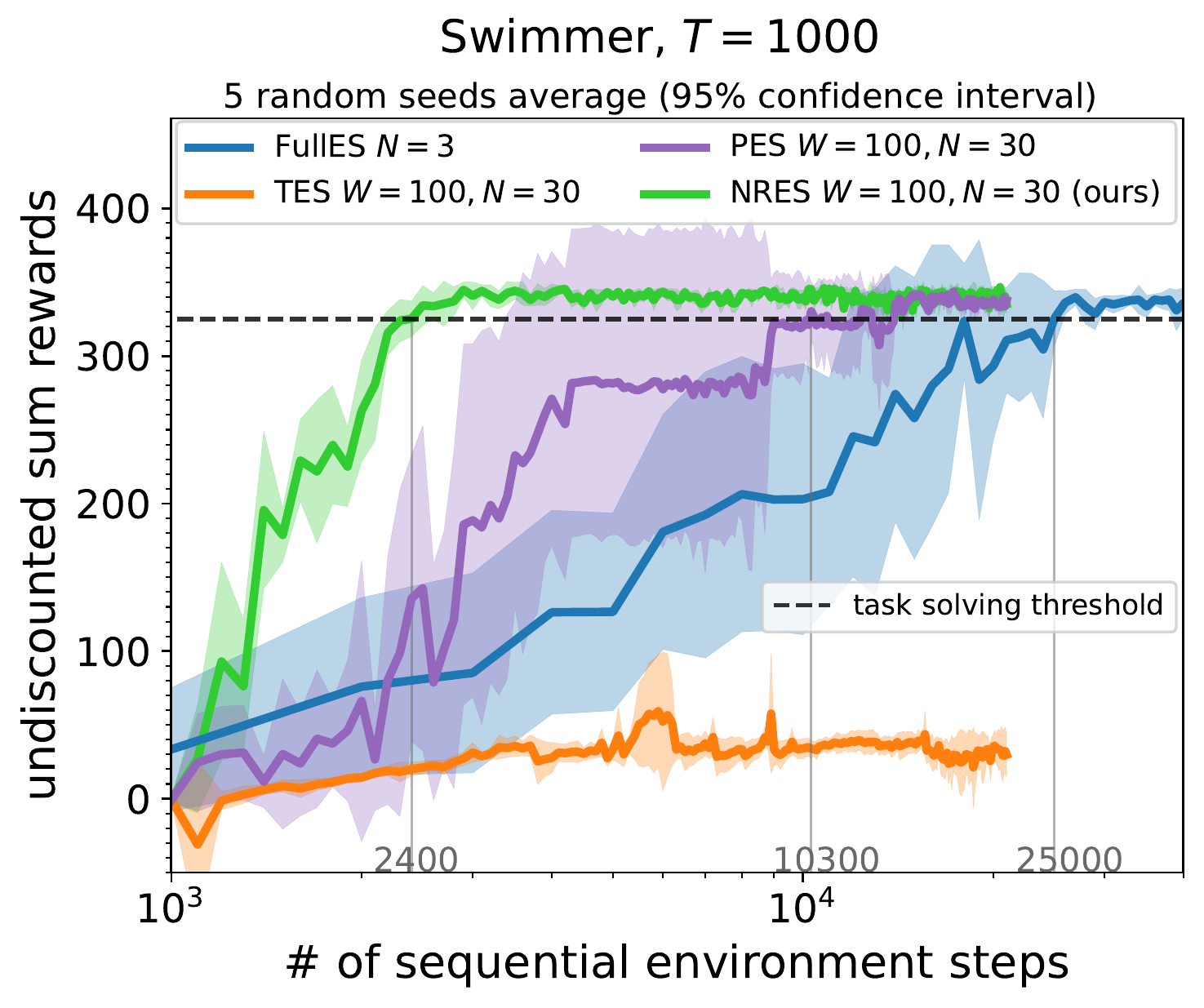}
        \put(27.0, 79){\textbf{(a)}}
    \end{overpic}
    \begin{overpic}[width=0.41\textwidth, percent]{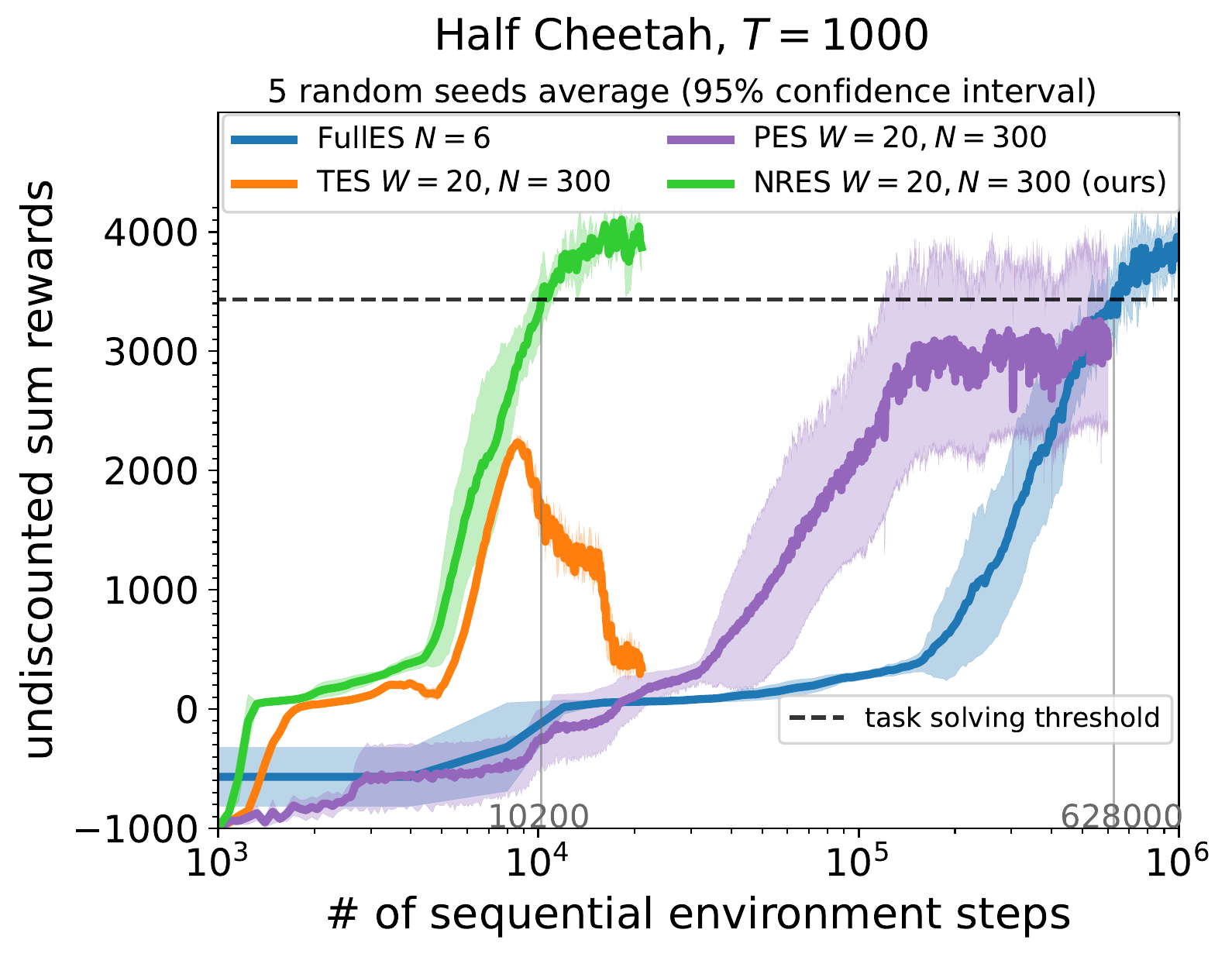}
        \put(26.0, 73){\textbf{(b)}}
    \end{overpic}
    \caption{ES methods' performance vs. the number of sequential environment steps used in solving the Mujoco (a) Swimmer task and (b) Half Cheetah task. $\NRES$ solves the tasks fastest under perfect parallelization.}
    \label{fig:rl_reward_progression}
\end{figure*}

\section{Additional Related Work}
\label{sec:relwork}

Beyond the most related work in Section~\ref{sec:background}, in this section, we further position 
$\NRES$
relative to existing zeroth-order gradient estimation methods. We also provide additional related work on automatic differentiation (AD) methods for unrolled computation graphs in Appendix~\ref{app_sec:related_work}.

\paragraph{Zeroth-Order Gradient Estimators.} 
In this work, we focus on zeroth-order methods that can estimate continuous parameters' gradients to be plugged into any first-order optimizers, unlike other zeroth-order optimization methods such as Bayesian Optimization \citep{frazier2018tutorial}, random search \citep{bergstra2012random}, or Trust Region methods \citep{maggiar2018derivative,liu2019trust}. We also don't compare against policy gradient methods \citep{sutton1999policy}, because they assume internal stochasticity of the unrolling dynamics, which may not hold for deterministic policy learning \citep[e.g.,][]{todorov2012mujoco}. Within the space of evolution strategies methods, many works have focused on improving the vanilla ES method's variance by changing the perturbation distribution \citep{choromanski2018structured, maheswaranathan2019guided, agapie2021spherical,gao22generalizing}, considering the covariance structure \citep{hansen2016cma}, and using control variates \citep{tang2020variance}. However, these works do not consider the unrolled structure of UCGs and are offline methods. In contrast, we reduce the variance by incorporating this unrolled aspect through online estimation and noise-reuse. As the aforementioned variance reduction methods work orthogonally to $\NRES$, it is conceivable that these techniques can be used in conjunction with $\NRES$ to further reduce the variance.
\section{Discussion, Limitations, and Future Work}
\label{sec:conclusion}

In this work, we improve online evolution strategies for unbiased gradient estimation in unrolled computation graphs by analyzing the best noise-sharing strategies. By generalizing an existing unbiased method, Persistent Evolution Strategies, to a broader class, we analytically and empirically identify the best estimator with the smallest variance and name this method Noise-Reuse Evolution Strategies ($\NRES$). We demonstrate the convergence benefits of $\NRES$ over other automatic differentiation and evolution strategies methods on a variety of applications. 

\paragraph{Limitations.} As $\NRES$ is both an online method and an ES method, it naturally inherits some limitations shared by all methods of these two classes, such as hysteresis and variance's linear dependence on the dimension $d$. We provide a detailed discussion of these limitations in Appendix~\ref{app_sec:limitations}.

\paragraph{Future Work.} There are some natural open questions: \textbf{1)} \textit{choosing a better sampling distribution for $\NRES$}. Currently the isotropic Gaussian's variance $\sigma^2$ is tuned as a hyperparameter. Whether there are better ways to leverage the sequential structure in unrolled computation graphs to automate the selection of this distribution is an open question. \textbf{2)} \textit{ Incorporating hysteresis.} Our analysis assumes no hysteresis in the gradient estimates and we haven't observed much impact of it in our experiments. However, understanding when and how to correct for hysteresis is an interesting direction.

\section*{Acknowledgments}
We thank Kevin Kuo, Jingnan Ye, Tian Li, and the anonymous reviewers for their helpful feedback.

\bibliographystyle{unsrtnat}
\bibliography{references}

\begin{thebibliography}{49}
\providecommand{\natexlab}[1]{#1}
\providecommand{\url}[1]{\texttt{#1}}
\expandafter\ifx\csname urlstyle\endcsname\relax
  \providecommand{\doi}[1]{doi: #1}\else
  \providecommand{\doi}{doi: \begingroup \urlstyle{rm}\Url}\fi

\bibitem[Hochreiter and Schmidhuber(1997)]{hochreiter1997long}
Sepp Hochreiter and J{\"u}rgen Schmidhuber.
\newblock Long short-term memory.
\newblock \emph{Neural Computation}, 9\penalty0 (8):\penalty0 1735--1780, 1997.

\bibitem[Cho et~al.(2014)Cho, Merrienboer, Gulcehre, Bougares, Schwenk, and
  Bengio]{cho2014learning}
Kyunghyun Cho, Bart Merrienboer, Caglar Gulcehre, Fethi Bougares, Holger
  Schwenk, and Yoshua Bengio.
\newblock Learning phrase representations using rnn encoder-decoder for
  statistical machine translation.
\newblock In \emph{EMNLP}, 2014.

\bibitem[Metz et~al.(2019)Metz, Maheswaranathan, Nixon, Freeman, and
  Sohl-Dickstein]{metz2019understanding}
Luke Metz, Niru Maheswaranathan, Jeremy Nixon, Daniel Freeman, and Jascha
  Sohl-Dickstein.
\newblock Understanding and correcting pathologies in the training of learned
  optimizers.
\newblock In \emph{International Conference on Machine Learning}, 2019.

\bibitem[Harrison et~al.(2022)Harrison, Metz, and
  Sohl-Dickstein]{harrison2022a}
James Harrison, Luke Metz, and Jascha Sohl-Dickstein.
\newblock A closer look at learned optimization: Stability, robustness, and
  inductive biases.
\newblock In Alice~H. Oh, Alekh Agarwal, Danielle Belgrave, and Kyunghyun Cho,
  editors, \emph{Advances in Neural Information Processing Systems}, 2022.

\bibitem[Maclaurin et~al.(2015)Maclaurin, Duvenaud, and
  Adams]{maclaurin2015gradient}
Dougal Maclaurin, David Duvenaud, and Ryan Adams.
\newblock Gradient-based hyperparameter optimization through reversible
  learning.
\newblock In \emph{International Conference on Machine Learning}, 2015.

\bibitem[Franceschi et~al.(2017)Franceschi, Donini, Frasconi, and
  Pontil]{pmlr-v70-franceschi17a}
Luca Franceschi, Michele Donini, Paolo Frasconi, and Massimiliano Pontil.
\newblock Forward and reverse gradient-based hyperparameter optimization.
\newblock In \emph{International Conference on Machine Learning}, 2017.

\bibitem[Wang et~al.(2018)Wang, Zhu, Torralba, and Efros]{wang2018dataset}
Tongzhou Wang, Jun-Yan Zhu, Antonio Torralba, and Alexei~A Efros.
\newblock Dataset distillation.
\newblock \emph{arXiv preprint arXiv:1811.10959}, 2018.

\bibitem[Cazenavette et~al.(2022)Cazenavette, Wang, Torralba, Efros, and
  Zhu]{cazenavette2022dataset}
George Cazenavette, Tongzhou Wang, Antonio Torralba, Alexei~A Efros, and
  Jun-Yan Zhu.
\newblock Dataset distillation by matching training trajectories.
\newblock In \emph{Proceedings of the IEEE/CVF Conference on Computer Vision
  and Pattern Recognition}, pages 4750--4759, 2022.

\bibitem[Sutton et~al.(1999)Sutton, McAllester, Singh, and
  Mansour]{sutton1999policy}
Richard~S Sutton, David McAllester, Satinder Singh, and Yishay Mansour.
\newblock Policy gradient methods for reinforcement learning with function
  approximation.
\newblock \emph{Advances in Neural Information Processing Systems}, 1999.

\bibitem[Schulman et~al.(2015)Schulman, Levine, Abbeel, Jordan, and
  Moritz]{pmlr-v37-schulman15}
John Schulman, Sergey Levine, Pieter Abbeel, Michael Jordan, and Philipp
  Moritz.
\newblock Trust region policy optimization.
\newblock In \emph{International Conference on Machine Learning}, 2015.

\bibitem[Baydin et~al.(2018)Baydin, Pearlmutter, Radul, and
  Siskind]{baydin2018automatic}
Atilim~Gunes Baydin, Barak~A Pearlmutter, Alexey~Andreyevich Radul, and
  Jeffrey~Mark Siskind.
\newblock Automatic differentiation in machine learning: a survey.
\newblock \emph{Journal of Machine Learning Research}, 18:\penalty0 1--43,
  2018.

\bibitem[Parmas et~al.(2018)Parmas, Rasmussen, Peters, and
  Doya]{pmlr-v80-parmas18a}
Paavo Parmas, Carl~Edward Rasmussen, Jan Peters, and Kenji Doya.
\newblock {PIPPS}: Flexible model-based policy search robust to the curse of
  chaos.
\newblock In \emph{International Conference on Machine Learning}, 2018.

\bibitem[Metz et~al.(2021)Metz, Freeman, Schoenholz, and
  Kachman]{metz2021gradients}
Luke Metz, C~Daniel Freeman, Samuel~S Schoenholz, and Tal Kachman.
\newblock Gradients are not all you need.
\newblock \emph{arXiv preprint arXiv:2111.05803}, 2021.

\bibitem[Salimans et~al.(2017)Salimans, Ho, Chen, Sidor, and
  Sutskever]{salimans2017evolution}
Tim Salimans, Jonathan Ho, Xi~Chen, Szymon Sidor, and Ilya Sutskever.
\newblock Evolution strategies as a scalable alternative to reinforcement
  learning.
\newblock \emph{arXiv preprint arXiv:1703.03864}, 2017.

\bibitem[Vicol et~al.(2021)Vicol, Metz, and Sohl-Dickstein]{vicol21unbiased}
Paul Vicol, Luke Metz, and Jascha Sohl-Dickstein.
\newblock Unbiased gradient estimation in unrolled computation graphs with
  persistent evolution strategies.
\newblock In \emph{International Conference on Machine Learning}, 2021.

\bibitem[Glynn(1990)]{glynn1990likelihood}
Peter~W Glynn.
\newblock Likelihood ratio gradient estimation for stochastic systems.
\newblock \emph{Communications of the ACM}, 33\penalty0 (10):\penalty0 75--84,
  1990.

\bibitem[Mass{\'e} and Ollivier(2020)]{masse2020convergence}
Pierre-Yves Mass{\'e} and Yann Ollivier.
\newblock Convergence of online adaptive and recurrent optimization algorithms.
\newblock \emph{arXiv preprint arXiv:2005.05645}, 2020.

\bibitem[Wang et~al.(2013)Wang, Chen, Smola, and Xing]{wang2013variance}
Chong Wang, Xi~Chen, Alexander~J Smola, and Eric~P Xing.
\newblock Variance reduction for stochastic gradient optimization.
\newblock \emph{Advances in neural information processing systems}, 26, 2013.

\bibitem[Vicol(2023)]{vicol2023low}
Paul Vicol.
\newblock Low-variance gradient estimation in unrolled computation graphs with
  es-single.
\newblock In \emph{International Conference on Machine Learning}, 2023.

\bibitem[Metz et~al.(2022{\natexlab{a}})Metz, Harrison, Freeman, Merchant,
  Beyer, Bradbury, Agrawal, Poole, Mordatch, Roberts, et~al.]{metz2022velo}
Luke Metz, James Harrison, C~Daniel Freeman, Amil Merchant, Lucas Beyer, James
  Bradbury, Naman Agrawal, Ben Poole, Igor Mordatch, Adam Roberts, et~al.
\newblock Velo: Training versatile learned optimizers by scaling up.
\newblock \emph{arXiv preprint arXiv:2211.09760}, 2022{\natexlab{a}}.

\bibitem[Mania et~al.(2018)Mania, Guy, and Recht]{mania2018simple}
Horia Mania, Aurelia Guy, and Benjamin Recht.
\newblock Simple random search of static linear policies is competitive for
  reinforcement learning.
\newblock In \emph{Advances in Neural Information Processing Systems}, 2018.

\bibitem[Todorov et~al.(2012)Todorov, Erez, and Tassa]{todorov2012mujoco}
Emanuel Todorov, Tom Erez, and Yuval Tassa.
\newblock Mujoco: A physics engine for model-based control.
\newblock In \emph{International Conference on Intelligent Robots and Systems},
  2012.

\bibitem[Wu et~al.(2018)Wu, Ren, Liao, and Grosse.]{wu2018understanding}
Yuhuai Wu, Mengye Ren, Renjie Liao, and Roger Grosse.
\newblock Understanding short-horizon bias in stochastic meta-optimization.
\newblock In \emph{International Conference on Learning Representations}, 2018.

\bibitem[Frazier(2018)]{frazier2018tutorial}
Peter~I Frazier.
\newblock A tutorial on bayesian optimization.
\newblock \emph{arXiv preprint arXiv:1807.02811}, 2018.

\bibitem[Bergstra and Bengio(2012)]{bergstra2012random}
James Bergstra and Yoshua Bengio.
\newblock Random search for hyper-parameter optimization.
\newblock \emph{Journal of Machine Learning Research}, 13\penalty0 (2), 2012.

\bibitem[Maggiar et~al.(2018)Maggiar, Wachter, Dolinskaya, and
  Staum]{maggiar2018derivative}
Alvaro Maggiar, Andreas Wachter, Irina~S Dolinskaya, and Jeremy Staum.
\newblock A derivative-free trust-region algorithm for the optimization of
  functions smoothed via gaussian convolution using adaptive multiple
  importance sampling.
\newblock \emph{SIAM Journal on Optimization}, 28\penalty0 (2):\penalty0
  1478--1507, 2018.

\bibitem[Liu et~al.(2019)Liu, Zhao, Yang, Bian, Qin, Yu, and Liu]{liu2019trust}
Guoqing Liu, Li~Zhao, Feidiao Yang, Jiang Bian, Tao Qin, Nenghai Yu, and
  Tie-Yan Liu.
\newblock Trust region evolution strategies.
\newblock In \emph{AAAI Conference on Artificial Intelligence}, 2019.

\bibitem[Choromanski et~al.(2018)Choromanski, Rowland, Sindhwani, Turner, and
  Weller]{choromanski2018structured}
Krzysztof Choromanski, Mark Rowland, Vikas Sindhwani, Richard Turner, and
  Adrian Weller.
\newblock Structured evolution with compact architectures for scalable policy
  optimization.
\newblock In \emph{International Conference on Machine Learning}, 2018.

\bibitem[Maheswaranathan et~al.(2019)Maheswaranathan, Metz, Tucker, Choi, and
  Sohl-Dickstein]{maheswaranathan2019guided}
Niru Maheswaranathan, Luke Metz, George Tucker, Dami Choi, and Jascha
  Sohl-Dickstein.
\newblock Guided evolutionary strategies: augmenting random search with
  surrogate gradients.
\newblock In \emph{International Conference on Machine Learning}, 2019.

\bibitem[Agapie(2021)]{agapie2021spherical}
Alexandru Agapie.
\newblock Spherical distributions used in evolutionary algorithms.
\newblock \emph{Mathematics}, 9\penalty0 (23), 2021.
\newblock ISSN 2227-7390.

\bibitem[Gao and Sener(2022)]{gao22generalizing}
Katelyn Gao and Ozan Sener.
\newblock Generalizing {G}aussian smoothing for random search.
\newblock In \emph{International Conference on Machine Learning}, 2022.

\bibitem[Hansen(2016)]{hansen2016cma}
Nikolaus Hansen.
\newblock The cma evolution strategy: A tutorial.
\newblock \emph{arXiv preprint arXiv:1604.00772}, 2016.

\bibitem[Tang et~al.(2020)Tang, Choromanski, and Kucukelbir]{tang2020variance}
Yunhao Tang, Krzysztof Choromanski, and Alp Kucukelbir.
\newblock Variance reduction for evolution strategies via structured control
  variates.
\newblock In \emph{International Conference on Artificial Intelligence and
  Statistics}, 2020.

\bibitem[Abadi(2016)]{abadi2016tensorflow}
Mart{\'\i}n Abadi.
\newblock Tensorflow: learning functions at scale.
\newblock In \emph{Proceedings of the 21st ACM SIGPLAN International Conference
  on Functional Programming}, pages 1--1, 2016.

\bibitem[Paszke et~al.(2019)Paszke, Gross, Massa, Lerer, Bradbury, Chanan,
  Killeen, Lin, Gimelshein, Antiga, et~al.]{paszke2019pytorch}
Adam Paszke, Sam Gross, Francisco Massa, Adam Lerer, James Bradbury, Gregory
  Chanan, Trevor Killeen, Zeming Lin, Natalia Gimelshein, Luca Antiga, et~al.
\newblock Pytorch: An imperative style, high-performance deep learning library.
\newblock \emph{Advances in Neural Information Processing Systems}, 2019.

\bibitem[Chen et~al.(2016)Chen, Xu, Zhang, and Guestrin]{chen2016training}
Tianqi Chen, Bing Xu, Chiyuan Zhang, and Carlos Guestrin.
\newblock Training deep nets with sublinear memory cost.
\newblock \emph{arXiv preprint arXiv:1604.06174}, 2016.

\bibitem[Tallec and Ollivier(2017{\natexlab{a}})]{tallec2017unbiasing}
Corentin Tallec and Yann Ollivier.
\newblock Unbiasing truncated backpropagation through time.
\newblock \emph{arXiv preprint arXiv:1705.08209}, 2017{\natexlab{a}}.

\bibitem[Gomez et~al.(2017)Gomez, Ren, Urtasun, and
  Grosse]{gomez2017reversible}
Aidan~N Gomez, Mengye Ren, Raquel Urtasun, and Roger~B Grosse.
\newblock The reversible residual network: Backpropagation without storing
  activations.
\newblock \emph{Advances in Neural Information Processing Systems}, 2017.

\bibitem[Williams and Zipser(1989)]{williams1989learning}
Ronald~J Williams and David Zipser.
\newblock A learning algorithm for continually running fully recurrent neural
  networks.
\newblock \emph{Neural Computation}, 1\penalty0 (2):\penalty0 270--280, 1989.

\bibitem[Silver et~al.(2021)Silver, Goyal, Danihelka, Hessel, and van
  Hasselt]{silver2021learning}
David Silver, Anirudh Goyal, Ivo Danihelka, Matteo Hessel, and Hado van
  Hasselt.
\newblock Learning by directional gradient descent.
\newblock In \emph{International Conference on Learning Representations}, 2021.

\bibitem[Tallec and Ollivier(2017{\natexlab{b}})]{tallec2017unbiased}
Corentin Tallec and Yann Ollivier.
\newblock Unbiased online recurrent optimization.
\newblock \emph{arXiv preprint arXiv:1702.05043}, 2017{\natexlab{b}}.

\bibitem[Mujika et~al.(2018)Mujika, Meier, and Steger]{mujika2018approximating}
Asier Mujika, Florian Meier, and Angelika Steger.
\newblock Approximating real-time recurrent learning with random kronecker
  factors.
\newblock \emph{Advances in Neural Information Processing Systems}, 31, 2018.

\bibitem[Benzing et~al.(2019)Benzing, Gauy, Mujika, Martinsson, and
  Steger]{benzing2019optimal}
Frederik Benzing, Marcelo~Matheus Gauy, Asier Mujika, Anders Martinsson, and
  Angelika Steger.
\newblock Optimal kronecker-sum approximation of real time recurrent learning.
\newblock In \emph{International Conference on Machine Learning}, 2019.

\bibitem[Krizhevsky et~al.(2009)]{krizhevsky2009learning}
Alex Krizhevsky et~al.
\newblock Learning multiple layers of features from tiny images.
\newblock 2009.

\bibitem[Hendrycks and Gimpel(2016)]{hendrycks2016gaussian}
Dan Hendrycks and Kevin Gimpel.
\newblock Gaussian error linear units (gelus).
\newblock \emph{arXiv preprint arXiv:1606.08415}, 2016.

\bibitem[Xiao et~al.(2017)Xiao, Rasul, and Vollgraf]{xiao2017fashion}
Han Xiao, Kashif Rasul, and Roland Vollgraf.
\newblock Fashion-mnist: a novel image dataset for benchmarking machine
  learning algorithms.
\newblock \emph{arXiv preprint arXiv:1708.07747}, 2017.

\bibitem[Metz et~al.(2022{\natexlab{b}})Metz, Freeman, Harrison,
  Maheswaranathan, and Sohl-Dickstein]{metz2022practical}
Luke Metz, C~Daniel Freeman, James Harrison, Niru Maheswaranathan, and Jascha
  Sohl-Dickstein.
\newblock Practical tradeoffs between memory, compute, and performance in
  learned optimizers.
\newblock In \emph{Conference on Lifelong Learning Agents}, 2022{\natexlab{b}}.

\bibitem[Brockman et~al.(2016)Brockman, Cheung, Pettersson, Schneider,
  Schulman, Tang, and Zaremba]{brockman2016openai}
Greg Brockman, Vicki Cheung, Ludwig Pettersson, Jonas Schneider, John Schulman,
  Jie Tang, and Wojciech Zaremba.
\newblock Openai gym.
\newblock \emph{arXiv preprint arXiv:1606.01540}, 2016.

\bibitem[Bradbury et~al.(2018)Bradbury, Frostig, Hawkins, Johnson, Leary,
  Maclaurin, Necula, Paszke, Vander{P}las, Wanderman-{M}ilne, and
  Zhang]{jax2018github}
James Bradbury, Roy Frostig, Peter Hawkins, Matthew~James Johnson, Chris Leary,
  Dougal Maclaurin, George Necula, Adam Paszke, Jake Vander{P}las, Skye
  Wanderman-{M}ilne, and Qiao Zhang.
\newblock {JAX}: composable transformations of {P}ython+{N}um{P}y programs,
  2018.
\newblock URL \url{http://github.com/google/jax}.

\end{thebibliography}

\newpage
\appendix
\onecolumn
\addtocontents{toc}{\protect\setcounter{tocdepth}{2}}
\section*{Appendix}

\renewcommand{\contentsname}{Appendix Outline}
\tableofcontents

\clearpage

\section{Notation}

In this section we provide two tables (Table~\ref{tab:notation1} and \ref{tab:notation2}) that sumarize all the notations we use in this paper.

\newcommand{\tablebreak}{[0.5em]}
\begin{table}[h]
\centering
\small
    \caption{Notations used in this paper (Part I)}
    \label{tab:notation1}
    \renewcommand{\arraystretch}{1.6} %
    \begin{tabular}{p{0.2\linewidth} p{0.6\linewidth}}
    \hline
    $T$ & the length of the unrolled computation graph. \\
    $[T]$ & the set of integers $\{1, \ldots, T\}$. \\
    $t \in \Integer \cap [0, T]$ & a time step in the dynamical system. \\
    $\theta \in \Real^d$ & the learnable parameter that unrolls the dynamical system at each time step. \\
    $\theta_i \in \Real^d$, $i \in [T]$ & the parameter that unrolls the dynamical system at the $i$-th time step. \\
    $d$ & dimension of the learnable unroll parameter $\theta$. \\
    $s \in \Real^p$ & an inner state of the dynamical system. \\
    $s_t \in \Real^p$ & the inner state of the dynamical system at time step $t$. \\
    $p$ & the dimension of the inner state $s$ in the dynamical system. \\
    $f_t: \Real^p \times \Real^d \rightarrow \Real^p$ & the transition dynamics from state at time step $t-1$ to time step $t$. The state to be transitioned into at time step $t$ is $f_t(s_{t-1}, \theta)$. $f_t$ doesn't need to be the same for all $t \in [T]$. For example, different $f_t$ could implicitly use different data as part of the computation.\\
    $L_t^s: \Real^p \rightarrow \Real$ & the loss function of the state $s_t$ at time step $t\in [T]$, which gives loss as $L_t^s(s_t)$.\\
    $L_t: \Real^{dt} \rightarrow \Real$ & the loss at time step $t \in [T]$ as a function of all the $\theta_i$'s applied up to time step $t$, $L_t(\theta_1, \ldots, \theta_t)$. Here $\theta_i$ doesn't need to be all the same.\\
    $L: \Real^{dT} \rightarrow \Real$ & the average loss over all $T$ steps incurred by unrolling the system from $s_0$ to $s_T$ using the sequence of $\{\theta_i\}_{i=1}^T$. $L(\theta_1, \ldots, \theta_T) \coloneqq \frac{1}{T} \sum_{t=1}^T L_t(\theta_1, \ldots, \theta_t)$.\\
    $L_t(\brck{\theta}_{\times a}, \brck{\theta'}_{\times {t-a}})$ & the loss incurred at time step $t$ by first unrolling with $\theta$ for $a$ steps, then unrolling with $\theta'$ for $t-a$ steps. $L_t(\brck{\theta}_{\times a}, \brck{\theta'}_{\times {t-a}})\coloneqq f(\underbrace{\theta,\ldots, \theta}_{\textrm{$a$ times}}, \underbrace{\theta',\ldots, \theta'}_{\textrm{$(t-a)$ times}})$.\\
    $W$ & the length of an unroll truncation window (we always assume $T$ is divisible by $W$ for proof cleanness). \\
    $I_{d\times d}$ & the $d$ by $d$ identity matrix.\\
    $\sigma$ & a positive hyperparameter controlling the standard deviation in the isotropic Gaussian distribution $\calN(\bzero, \sigma^2 I_{d\times d})$.\\
    $\bepsilon$ & a random perturbation vector in $\Real^d$ sampled from $\calN(\bzero, \sigma^2 I_{d\times d})$. \\
    $\bepsilon_i$ & the $i$-th Gaussian random vector sampled by an online evolution strategies worker in a given episode. The total number of $\epsilon_i$ in an episode might be strictly smaller than the number of truncation windows (which is $T/W$) for $\GPES_{K}$ when $K > W$.  \\
    $\btau$ & a random variable sampled from the uniform distribution $\Unif\{0, W, \ldots, T - W\}.$ $\btau$ denotes the starting time step of a truncation window by an online evolution strategies worker. \\
    $K$ & the noise-sharing period for the algorithm $\GPES_K$. $K$ is always a multiple of $W$, i.e. $K=cW$ for some positive integer $c$. \\
    $c$ & the integer ratio $K / W$. \\
    \end{tabular}
\end{table}
\begin{table}[ht]
\small
\centering
    \caption{Notations used in this paper (continued, Part II)}
    \label{tab:notation2}
    \renewcommand{\arraystretch}{1.6} %
    \begin{tabular}{p{0.2\linewidth} p{0.6\linewidth}}
    \hline
    $\lceil x \rceil$ & the ceiling of $x \in \Real$, the smallest integer $y \in \Integer$ such that $y \ge x$ \\
    $\lfloor x \rfloor$ & the floor of $x \in \Real$, the largest integer $y \in \Integer$ such that $y \le x$ \\
    $\remainder: \Integer^+ \times \Integer^+ \rightarrow \Integer^+$ & the modified remainder function. $\remainder(x, y)$ is the unique integer $n \in [1, y]$ where $x = qy + n$ for some integer $q$. For example, if $T=6W$ and $K=3W$, $\remainder(T, K) = 3W$, while if $T=6W$ and $K=4W$, $\remainder(T, K) = 2W$.\\   
    $K$-smoothed\quad loss & the loss function: \\
     & \quad $\theta \mapsto \E_{\{\bepsilon_i\}}  L([\theta + \bepsilon_1]_{\scaleto{\times K}{4pt}}, \ldots, [\theta + \bepsilon_{\lceil T/K \rceil}]_{\times \remainder(T, K)})$. \\
    $\{\{g^t_i \in \Real^d\}_{i=1}^t\}_{t=1}^T$ & the classes of sets of vectors associated with a given fixed $\theta$ defined in Assumption \ref{assumption:linearity}. For any given $\theta$, there are $T$ such sets of vectors, one for each time step $t \in \{1, \ldots, T\}$. Roughly speaking, $g^t_i$ is time step $t$'s smoothed loss's partial derivative with respect to the $i$-th application of $\theta$ \\
    $g^t$ & $g^t \coloneqq \sum_{i=1}^t g^t_i$ for any time step $t$. Roughly speaking, $g^t$ is time step $t$'s smoothed loss's total derivative with respect to the all the application of the same $\theta$. \\
    $g^t_{K, j}$ & $g^t_{K, j} \coloneqq \sum_{i=K\cdot(j-1) + 1}^{\min\{t,\; K\cdot j\}} g^t_i$ for $j \in \{1, \ldots, \lceil t/K \rceil \}$ and time step $t$. We can understand $g^t_{K, j}$ as the sum of partial derivatives of smoothed step-$t$ loss with respect to all $\theta$'s in the $j$-th noise-sharing window of size $K$. If $K$ doesn't divide $t$, the last such window will be shorter than $K$. \\
    $g^t_{c, j}$ & $g^t_{K, j}$ when $K=c$ (used in the case of $W=1$). \\
    $\tr(A)$ & the trace (sum of diagonals) of a $d \times d$ matrix $A \in \Real^{d\times d}$. \\
    $\Cov[\bX]$ & the $d \times d$ covariance matrix of random vector $\bX \in \Real^d$. \\
    $\FullES(\theta)$ & a single $\FullES$ worker's gradient estimate given in Equation \eqref{eq:fulles_estimator}. To give a gradient estimate, the worker will run a total of $2T$ steps. The randomness comes from $\bepsilon$. \\
    $\TES(\theta)$ & a single $\TES$ worker's gradient estimate given in Equation \eqref{eq:tes_estimator}. This estimator keeps track of a single saved state. It samples a new noise in each truncation window and performs antithetic unrolling from the saved state. After computing the gradient estimate using the two antithetic states, the worker will run another $W$ steps from the saved state using $\theta$ without perturbation and record this as the new saved state. This estimator takes $3W$ unroll steps in total to produce a gradient estimate. The randomness comes from $\{\bepsilon_i\}_{i=1}^{T/W}$ and $\btau$. \\
    $\PES(\theta)$ & a single $\PES$ worker's gradient estimate given in Equation \eqref{eq:pes_estimator}. This estimator keeps both a positive and negative inner state and samples a new noise perturbation at the beginning of every truncation window. It accumulates all the noise sampled in an episode to correct for bias. It runs a total of $2W$ steps to produce a gradient estimate. The randomness comes from $\{\bepsilon_i\}_{i=1}^{T/W}$ and $\btau$. \\
    $\GPES_K(\theta)$ & a single $\GPES_K(\theta)$ worker's gradient estimate given in Lemma \ref{lemma:gpesk_form}. This estimator keeps both a positive and negative inner state and samples a new noise perturbation every $K$ steps in a given episode. It also accumulates past sampled noise for bias correction. To give a gradient estimate, the worker will run a total of $2W$ steps. The randomness comes from $\{\bepsilon_i\}_{i=1}^{\lceil T/K \rceil}$ and $\btau$.\\
    $\NRES(\theta)$ & a single $\NRES$ worker's gradient estimate. It is the same as $\GPES_{K=T}(\theta)$. This estimator keeps both a positive and negative inner state, and it only samples a noise perturbation once at the beginning of each episode. To give a gradient estimate, the worker will run a total of $2W$ steps. The randomness comes from $\bepsilon$ (single noise sampled at the beginning of an episode) and $\btau$.\\
    \end{tabular}
\end{table}
\clearpage

\setcounter{theorem}{0}

\allowdisplaybreaks
\section{Additional Related Work}
\label{app_sec:related_work}

Beyond the related work we discuss in Section~\ref{sec:background} and \ref{sec:relwork} on evolution strategies methods, in this section, we discuss additional related work in gradient estimation for unrolled computation graphs, including work on automatic differentiation (AD) (reverse mode and forward mode) methods and a concurrent work on online evolution strategies.
Some of the AD methods described in this section are compared against as baselines in our experiments in Section~\ref{subsec:lorenz} and \ref{exp:lopt}. %

\paragraph{Reverse Mode Differentiation (RMD).} When the loss function and transition functions in UCG is differentiable, the default method for computing gradients is backpropagation through time ($\BPTT$). However, $\BPTT$ has difficulties when applied to  UCGs: \textbf{1)} \textit{memory issues}:
the default BPTT implementations \citep[e.g.,][]{abadi2016tensorflow, paszke2019pytorch} store all activations of the graph in memory, making memory usage scale linearly with the length of the unrolled graph. There are works that improve the memory dependency of $\BPTT$; however, they either require customized framework implementation \citep{chen2016training,tallec2017unbiasing} or specially-designed reversible computation dynamics \citep{maclaurin2015gradient, gomez2017reversible}. \textbf{2)} \textit{not online}: each gradient estimate using $\BPTT$ requires full forward and backward computation through the UCG, which is computationally expensive and incurs large latency between successive parameter updates. To alleviate the memory issue and allow online updates, a popular alternative is truncated backpropagation through time ($\TBPTT$) which estimates the gradient within short truncation windows. However, this blocks the gradient flow to the parameters applied before the current window, making the gradient estimate biased and unable to capture long-term dependencies \citep{wu2018understanding}. In contrast, $\NRES$ is memory efficient, online, and doesn't suffer from bias, while able to handle loss surfaces with extreme local sensitivity.

\paragraph{Forward Mode Differentiation (FMD).} An alternative to RMD in automatic differentiation is FMD, which computes gradient estimates through Jacobian-vector products alongside the actual forward computation, thus allowing for online applications. Among FMD methods, real-time recurrent learning ($\RTRL$) \citep{williams1989learning} requires a computation cost that scales with the dimension of the learnable parameter, making it intractable for large problems. To alleviate the computation cost, stochastic approximations of $\RTRL$ have been proposed: $\DODGE$ \citep{silver2021learning} computes directional gradient along a certain direction; $\UORO$ \citep{tallec2017unbiased} unbiasedly approximates the Jacobian with a rank-1 matrix; $\mathrm{KF}-\mathrm{RTRL}$ \citep{mujika2018approximating} and OK \citep{benzing2019optimal} uses Kronecker product decomposition to improve the gradient estimate's variance, but are specifically for RNNs. In Section~\ref{subsec:lorenz} and \ref{exp:lopt}, we experiment with forward mode methods $\DODGE$ (with standard Gaussian random directions) and $\UORO$ and demonstrate $\NRES$'s advantage over these two methods when the loss surfaces have high sensitivity to small changes in the parameter space.

\paragraph{Concurrent work on online evolution strategies.} Finally, we note that a concurrent work \cite{vicol2023low} on online evolution strategies proposes a similar algorithm (ES-Single) to the algorithm $\NRES$ proposed in our paper. However, their analyses and experiments differ from ours in a number of ways: 

\begin{enumerate}[leftmargin=*]
    \item \textit{Theoretical assumptions}. To capture the online nature of the online ES methods considered in our paper, we adopt a novel view which treats the random truncation window that an online ES method starts from as a random variable (which we denote by $\btau$). In contrast, this assumption is not made neither in the prior work on PES \cite{vicol21unbiased} nor in the concurrent work \cite{vicol2023low}. As such, these analyses cannot distinguish the theoretical difference between the estimator $\FullES(\theta)$ and $\NRES(\theta)$.
    \item \textit{Theoretical conclusions}. As a result of our novel viewpoint/theoretical assumption regarding the random truncation window used in online ES methods, we provide a precise variance characterization of our newly proposed class of $\GPES_K$ gradient estimators. Two of the main theoretical contributions of our work are thus: \textit{a)} showing that $\NRES$ provably has the lowest variance among the entire $\GPES$ class and \textit{b)} identifying the conditions under which $\NRES$ can have lower variance than $\FullES$ with the same compute budget. In contrast, \cite{vicol2023low} do not make such contributions --- in fact, in their theoretical analyses, they characterize their proposed method ES-Single as having the exact same variance as $\FullES$, and as such they are unable to draw conclusions about the variance reduction benefits of their approach as we have done in our analyses. %
    \item \textit{Experimental comparison against non-online method $\FullES$}. In \cite{vicol21unbiased, vicol2023low}, when comparing against ES methods, the authors primarily compare against online ES methods but not the canonical non-online ES method, $\FullES$. In contrast, in our paper, we experimentally show that $\NRES$ can indeed provide significant speedup benefits over its non-online counterpart $\FullES$. We believe these more complete results provide critical evidence which encourages evolution strategies users to consider switching to online methods ($\NRES$) when working with unrolled computation graphs.
    \item \textit{Identifying the appropriate scenarios for the proposed method.} In our work, we precisely identify problem scenarios that are most appropriate for the use of $\NRES$ (when the losses are extremely locally sensitive or blackbox) and provide experiments mirroring these scenarios to compare different gradient estimation (both ES and AD) methods. In contrast, \cite{vicol2023low} perform some of their experiments on problems where the loss surface might not have high sensitivity (e.g., LSTM copy task) and only show performance of ES methods (but not AD methods). However, for these scenarios where the loss surfaces are well-behaved, ES methods likely should not be used in the first place over traditional AD approaches (we discuss this further in Section~\ref{app_sec:limitations} in the Appendix). In addition, \cite{vicol2023low} treats the application of their proposed method to blackbox losses (e.g., reinforcement learning) as future work, while we provide experiments demonstrating the effectiveness of $\NRES$ over other ES methods for this important set of applications in Section~\ref{exp:rl}.
\end{enumerate}

\section{Algorithms}
In this section, we provide Python-style pseudocode for the gradient estimation algorithms discussed in this paper. We first provide the pseudocode for $\FullES$. We then provide the pseudocode for general online evolution stratgies training. We finally provide the pseudocode for Truncated Evolution Strategies ($\TES$) and Generalized Persistent Evolution Strategies ($\GPES$).

\subsection{FullES Pseudocode}
We show the pseudocode for the vanilla antithetic evolution strategies gradient estimation method $\FullES$ in Algorithm~\ref{alg:fulles}. Here we note that the FullESWorker is stateless (it has a boilerplate \init() function). In addition, to produce a single gradient estimate, it needs to run from the beginning of the UCG ($s_0$) to the end of the graph after $T$ unroll steps. This is in contrast to the online ES methods which only unroll a truncation window of $W$ steps forward for each gradient estimate.
\label{app_subsec:fulles_algorithm}
\begin{algorithm}[h]
\small
\caption{Vanilla Antithetic Evolution Strategies ($\FullES$)}
\label{alg:fulles}

\textbf{class} FullESWorker:

\quad \deff ~ \init(\self,):

\qquad \texttt{\scriptsize \# no need to initialize since $\FullES$ is stateless}

\qquad \textbf{pass}

~

\quad \deff gradient\_estimate(\self, $\; \theta$):

\qquad $\bepsilon \sim \calN(\bzero, \sigma^2I_{d\times d})$

\qquad \texttt{\scriptsize \# $\FullES$ always starts from the beginning of an episode}

\qquad ($s^+, s^{-}$) $=$ ($s_0$, $s_0$)

\qquad $L_{\mathrm{sum}}^+ = 0$; \;\; $L_{\mathrm{sum}}^- = 0$

\qquad \texttt{\scriptsize \# $\FullES$ always runs till the end of an episode}

\qquad \for $t$ \inn range(1, $T + 1$):

\qquad \quad $s^+ = f_{t}(s^+, \theta + \bepsilon)$

\qquad \quad $s^- = f_{t}(s^-, \theta - \bepsilon)$

\qquad \quad $L_{\mathrm{sum}}^+ \mathrel{+}= L^s_{t}(s^+)$

\qquad \quad $L_{\mathrm{sum}}^- \mathrel{+}= L^s_{t}(s^-)$

~

\qquad $g = (L_{\mathrm{sum}}^+ - L_{\mathrm{sum}}^-) / (2\sigma^2 \cdot T) \cdot ~ \bepsilon$

\qquad \return $g$

\end{algorithm}

\subsection{Online Evolution Strategies Pseudocode}

\label{app_subsec:oes_alg}
OnlineESWorker (Algorithm \ref{alg:oes}) is an abstract class (interface) that all the online ES methods will implement. The key functionality an OnlineESWorker provides is its worker.gradient\_estimate($\theta$) function that performs unrolls in a truncation window of size $W$ and returns a gradient estimate based on the unroll. With this interface, we can train using online Evolution Strategies workers following Algorithm \ref{alg:oes_training_algorithm}. The training takes two steps:
\begin{itemize}[leftmargin=1.5cm]
    \item[Step 1.] Constructing independent \textit{step-unlocked} workers to form a worker pool. This requires sampling different truncation window starting time step $\btau$ for different workers. During this stage, for simplicity and rigor, we only rely on the~.gradient\_estimate method call's side effect to alter the worker's saved states and discard the computed gradients. (We still count these environment steps for the reinforcement learning experiment in Experiment \ref{exp:rl}).)
    \item[Step 2.] Training using the worker pool. At each outer iteration, we average all the worker's computed gradient estimates and pass that to any first order optimizer $\textrm{OPT\_UPDATE}$ (e.g. SGD or Adam) to update $\theta$ and repeat until convergence. Each worker's gradient\_estimate method call can be parallelized.
\end{itemize}

\begin{algorithm}[h]
\small
\caption{Online Evolution Strategies (OES) (a qbstract class)}
\label{alg:oes}

\textbf{import} abc \qquad \texttt{\scriptsize \# abstract base class}

\textbf{class} OnlineESWorker(abc.ABC):

\quad $W$: int \qquad \texttt{\scriptsize  \# the size of the truncation window}

\smallskip

\quad @abc.abstractmethod

\quad \deff ~ \init(\self, $W$):

\qquad \texttt{\small """} 

\qquad \texttt{\scriptsize set up the saved states and other bookkeeping variables}

\qquad \texttt{\small """}

\smallskip

\smallskip

\quad @abc.abstractmethod

\quad \deff gradient\_estimate(\self, $\; \theta$):

\qquad \texttt{\small """} 

\qquad \texttt{\scriptsize Given a $\theta$,}

\qquad \texttt{\scriptsize \qquad perform partial unroll in a truncation window of length $W$ and return a gradient estimate for $\theta$}

\qquad \texttt{\scriptsize \qquad save the end inner state(s) and start off from the saved state(s)}

\qquad \texttt{\scriptsize \qquad \quad when \textrm{\self.gradient\_estimate} is called again}

\qquad \texttt{\scriptsize \qquad if reach the end, reset to the initial state $s_0$}

\qquad \texttt{\small """}

\end{algorithm}

\begin{algorithm}[h]
\small
\caption{Training using Online Evolution Strategies}
\label{alg:oes_training_algorithm}

\quad $\theta = \theta_{\textrm{init}}$ \texttt{\qquad \scriptsize \# start value of $\theta$ optimization.}

\bigskip

\quad \texttt{\scriptsize \# Step 1: Initialize online ES workers}

\quad worker\_list = []

\quad \for $i$ \inn range(N): \qquad \texttt{\eightquad \scriptsize \# $N$ is the number of workers, can be parallelized}

\qquad new\_worker = OnlineESWorker($W$) \texttt{\qquad \scriptsize \# replace with a real implementation of OnlineESWorker}

\qquad \texttt{\scriptsize \# the steps below make sure the workers are step-unlocked}

\qquad \texttt{\scriptsize \scriptsize \# i.e. working independently at different truncation windows}

\qquad $\btau \sim \Unif\{0, W, \ldots, T - W\}$

\qquad \for $t$ \inn range($\btau / W$):

\qquad \quad \underline{\hspace{0.5em}} = new\_worker.gradient\_estimate($\theta$)

\smallskip

\qquad worker\_list.append(new\_worker)

\bigskip

\quad \texttt{\# \scriptsize Step 2:  training}

\quad \textbf{while} not converged:

\qquad $g_\textrm{sum} = \bzero$ \texttt{\eightquad \qquad \scriptsize \qquad \# $g_\textrm{sum}$ is a vector in $\Real^d$}

\qquad \for worker \inn worker\_list: \texttt{\fourquad \qquad \scriptsize \# can be parallelized}

\qquad \quad $g_\textrm{sum}$ += worker.gradient\_estimate($\theta$) \texttt{\qquad \scriptsize \# accumulate this worker's gradient estimate}

\smallskip

\qquad $g = g_\textrm{sum} / N$ \texttt{\scriptsize \eightquad \qquad \# average all workers' gradient estimates}

\qquad $\theta = \textrm{OPT\_UPDATE}(\theta, g)$ \texttt{\fourquad \scriptsize \# updating $\theta$ with any first order optimizers}

~

\end{algorithm}

\clearpage
\subsection{Truncated Evolution Strategies Pseudocode}
\label{app_subsec:tes_algorithm}
In Section \ref{sec:background}, we have described the biased online evolution strategies method Truncated Evolution Strategies ($\TES$). We have provided its analytical form:

\begin{align}
   \frac{1}{2\sigma^2W} \sum_{i=1}^W \big [L_{\btau + i}([\theta]_{\scaleto{\times \btau}{4pt}}, [\theta + \bepsilon_{(\btau / W) + 1}]_{\scaleto{\times i}{4pt}}) - L_{\btau + i}([\theta]_{\scaleto{\times \btau}{4pt}}, [\theta - \bepsilon_{(\btau / W) + 1}]_{\scaleto{\times i}{4pt}}) \big] \bepsilon_{(\btau / W) + 1}.
\end{align}

Here we provide the algorithm pseudocode for $\TES$ in Algorithm \ref{alg:tes}. It is important to note that after the antithetic unrolling using perturbed $\theta + \bepsilon$ and $\theta - \bepsilon$ for gradient estimates, another $W$ steps of unrolling is performed starting from the saved starting state using the unperturbed $\theta$. Because of this, a TESWorker requires a total of $3W$ unroll steps to produce a gradient estimate (unlike $\PES$, $\GPES$, and $\NRES$ which requires $2W$). The algorithm in this form is first introduced in \citep{vicol21unbiased}. 
\begin{algorithm}[H]
\small
\caption{Truncated Evolution Strategies (TES)}
\label{alg:tes}
\textbf{class} TESWorker(OnlineESWorker):

\quad \deff ~ \init(\self, $W$):

\qquad \self.$\btau = 0$; \; \self.$s = s_0$

\qquad \self.$W = W$
\smallskip

~

\quad \deff gradient\_estimate(\self, $\; \theta$):

\begin{tikzpicture}[remember picture, overlay]
        \draw[line width=0pt, draw=green!30, rounded corners=2pt, fill=green!30, fill opacity=0.3]
            (0.6, -0.03) rectangle (5.2, 0.26);
\end{tikzpicture}
\quad \;\; \texttt{\scriptsize \# sample at every truncation window}

\qquad $\bepsilon \sim \calN(\bzero, \sigma^2I_{d\times d})$

~

\qquad ($s^+, s^{-}$) $=$ (\self.$s$, \self.$s$) \quad \texttt{\scriptsize \# unroll from the same state for the antithetic pair}

\qquad $L_{\mathrm{sum}}^+ = 0$; \;\; $L_{\mathrm{sum}}^- = 0$

~

\qquad \for $i$ \inn range(1, \self.$W$+1):

\qquad \quad $s^+ = f_{\textrm{\self}.\btau + i}(s^+, \theta + \bepsilon)$

\qquad \quad $s^- = f_{\textrm{\self}.\btau + i}(s^-, \theta - \bepsilon)$

\qquad \quad $L_{\mathrm{sum}}^+ \mathrel{+}= L^s_{\textrm{\self}.\btau + i}(s^+)$

\qquad \quad $L_{\mathrm{sum}}^- \mathrel{+}= L^s_{\textrm{\self}.\btau + i}(s^-)$

~

\qquad $g = (L_{\mathrm{sum}}^+ - L_{\mathrm{sum}}^-) / (2\sigma^2 \cdot \textrm{\self}.W) \cdot \bepsilon$

~

\qquad \for $i$ \inn range(1, self.$W$+1): \quad \texttt{\scriptsize \# finally unroll using unperturbed $\theta$}

\qquad \quad \self.$s = f_{\self.\btau + i}(\textrm{\self}.s, \; \theta)$;

~

\qquad \self.$\btau$ = \self.$\btau + W$

\qquad \ifff \self.$\btau \ge T$: \texttt{\scriptsize \fourquad \qquad \;\# reset at the end of an episode}

\qquad \quad \self.$\btau = 0$

\qquad \quad \self.$s$ = $s_0$

~

\qquad \return $g$

\end{algorithm}%

\clearpage%
\subsection{Generalized Persistent Evolution Strategies Pseudocode}
In section \ref{sec:gpes}, we propose a new class of unbiased online evolution strategies methods which we name \textit{Generalized Persistent Evolution Strategies} ($\GPES$). It produces an unbiased gradient estimate of the $K$-smoothed loss objective defined in Equation~\ref{eq:k-smooth-objective} in the main paper. It samples a new Gaussian noise for perturbation every $K$ unroll steps ($K$ is a multiple of the truncation window size $W$). We provide the pseudocode for $\GPES$ in Algorithm~\ref{alg:gpes}.
\label{app_subsec:gpes_algorithm}
\begin{algorithm}[h]
\small
\caption{Generalized Persistent Evolution Strategies (GPES)}
\label{alg:gpes}

\textbf{class} GPESWorker(OnlineESWorker):

\quad \deff ~ \init(\self, $W$, $K$):

\qquad \self.$\btau = 0$; \; \self.$s^+ = s_0$; \; \self.$s^- = s_0$

\begin{tikzpicture}[remember picture, overlay]
        \draw[line width=0pt, draw=orange!30, rounded corners=2pt, fill=orange!30, fill opacity=0.3]
            (4.4, -0.12) rectangle (6.5, 0.28);
\end{tikzpicture}
\quad\;\; \self.$W = W$;\; \self.$K = K$; \; \self.$\boldxi = \bzero \in \Real^d$
\smallskip

\quad \deff gradient\_estimate(\self, $\; \theta$):

\qquad \texttt{\scriptsize \# only sample a new $\bepsilon$ when \self.$\btau$ is a multiple of \self.$K$}

\begin{tikzpicture}[remember picture, overlay]
        \draw[line width=0pt, draw=green!30, rounded corners=2pt, fill=green!30, fill opacity=0.3]
            (0.71, -0.12) rectangle (4.4, 0.3);
\end{tikzpicture}
\qquad \ifff \self.$\btau \;\; \% \;\; \self.K == \; 0$: 

\qquad \quad $\bepsilon_{\textrm{new}} \sim \calN(\bzero, \sigma^2I_{d\times d})$ 

\qquad \quad \texttt{\scriptsize \# keep track of $\epsilon_\textrm{new}$ to use for the next $K$ steps}

\begin{tikzpicture}[remember picture, overlay]
        \draw[line width=0pt, draw=orange!30, rounded corners=2pt, fill=orange!30, fill opacity=0.3]
            (0.92, -0.06) rectangle (2.6, 0.24);
\end{tikzpicture}
\;\;\;\;\; \quad \self.$\bepsilon = \bepsilon_{\textrm{new}}$

\begin{tikzpicture}[remember picture, overlay]
        \draw[line width=0pt, draw=orange!30, rounded corners=2pt, fill=orange!30, fill opacity=0.3]
            (0.92, -0.06) rectangle (3.0, 0.24);
\end{tikzpicture}
\;\;\;\;\; \quad \self.$\boldxi \mathrel{+}= ~\bepsilon_{\textrm{new}}$

\qquad \texttt{\scriptsize \# after the if statement above, \self.$\bepsilon$ is now $\darkgreen{\bepsilon_{\lfloor \self.\btau/K \rfloor + 1}}$}

\qquad \texttt{\scriptsize \# and \self.$\boldxi$ is now $\sum_{i=1}^{\lfloor \btau/K \rfloor + 1} \bepsilon_i$}

~

\qquad ($s^+, s^{-}$) $=$ (\self.$s^+$, \self.$s^-$)

\qquad $L_{\mathrm{sum}}^+ = 0$; \;\; $L_{\mathrm{sum}}^- = 0$

\qquad \for $i$ \inn range(1, \self.$W$+1):

\qquad \quad $s^+ = f_{\textrm{\self}.\btau + i}(s^+, \theta + \self.\bepsilon)$

\qquad \quad $s^- = f_{\textrm{\self}.\btau + i}(s^-, \theta - \self.\bepsilon)$

\qquad \quad $L_{\mathrm{sum}}^+ \mathrel{+}= L^s_{\self.\btau + i}(s^+)$

\qquad \quad $L_{\mathrm{sum}}^- \mathrel{+}= L^s_{\self.\btau + i}(s^-)$

~

\begin{tikzpicture}[remember picture, overlay]
        \draw[line width=0pt, draw=orange!30, rounded corners=2pt, fill=orange!30, fill opacity=0.3]
            (5.6, -0.06) rectangle (6.4, 0.24);
\end{tikzpicture}
\;\;\;\;\; $g = (L_{\mathrm{sum}}^+ - L_{\mathrm{sum}}^-) / (2\sigma^2 \cdot \textrm{\self}.W) \cdot ~$\self.$\boldxi$

\qquad \self.$s^+ = s^+$; \self.$s^- = s^-$

\qquad \self.$\btau$ = \self.$\btau + W$

~

\qquad \ifff \self.$\btau \ge T$: \texttt{\scriptsize \fourquad \qquad \;\# reset at the end of an episode}

\qquad \quad \self.$\btau = 0$; \self.$s^+ = s_0$; \self.$s^- = s_0$

\begin{tikzpicture}[remember picture, overlay]
        \draw[line width=0pt, draw=orange!30, rounded corners=2pt, fill=orange!30, fill opacity=0.3]
            (0.9, -0.06) rectangle (2.3, 0.24);
\end{tikzpicture}
\;\;\;\;\; \quad \self.$\boldxi = \bzero$

~

\qquad \return $g$

\end{algorithm}

\section{Theory and Proofs}
\label{app_sec:theory}
In this section, we provide the proofs and interpretations of the lemmas, assumptions, and theorems presented in the main paper (we also restate these results for completeness). Before we begin, we set some background notation: 
\begin{itemize}[leftmargin=*]
    \item We treat all real-valued function's gradient as column vectors.
    \item We use $[\theta]_{\times t}$ to denote $t$ copies to $\theta$ stacked together to form a $t d$-dimensional column vector.
    \item For $\{v_i \in \Real^d\}_{i=1}^m$, we use $\begin{bmatrix} v_1 \\ \vdots \\ v_m \end{bmatrix}$ to denote the column vector whose first $d$-dimensions is $v_1$ and so on and so forth. Similar we use $\begin{bmatrix} v_1 & \ldots & v_m \end{bmatrix}$ to denote the transpose of the previous vector (thus a row vector).
    \item \textbf{Gradient notation} For any real-valued function whose input is more than one single $\theta$ (e.g., $L_t$ which takes in $t$ copies of $\theta$'s), we will use $\nabla_\Theta$ to describe the gradient with respect to the function's entire input dimension and similarly $\nabla^2_\Theta$ for Hessian. For such functions, we will use $\frac{\partial}{\partial \theta_i}$ to denote the partial derivative of the function with respect to the $i$-th $\theta$. We will use $\frac{d}{d\theta}$ to define the total derivative of some variable with respect to $\theta$ (e.g. when talking about $\frac{dL_{\mathrm{avg}}}{d\theta}$ with $L_{\mathrm{avg}}(\theta) \coloneqq L([\theta]_{\times t})$ if $L$ is differentiable). This operator $\frac{d}{d\theta}$ will also produce the Jacobian matrix when the applied function is vector-valued.
\end{itemize}

\subsection{Proof of Lemma \ref{lemma:gpesk_form}}
Define the $K$-smoothed loss objective as the function:
\begin{align}
\theta \mapsto \E_{\{\bepsilon_i\}}  L([\theta + \bepsilon_1]_{\times K}, \ldots, [\theta + \bepsilon_{\lceil T/K \rceil}]_{\times \remainder(T, K)}).
\end{align}

\begin{lemma}

An unbiased gradient estimator for the $K$-smoothed loss is given by
\begin{align}
   & \GPES_{K=cW}(\theta) \\
   \coloneqq & \frac{1}{2\sigma^2W}  \sum_{j=1}^W \bigg [ L_{\btau + j}([\theta + \bepsilon_1]_{\times K}, [\theta + \bepsilon_2]_{\times K}, \ldots, [\theta + \bepsilon_{\lfloor \btau/K \rfloor + 1} ]_{\times \remainder(\btau + j, K)}) \nonumber \\
   & \fourquad -  L_{\btau + j}([\theta - \bepsilon_1]_{\times K}, [\theta - \bepsilon_2]_{\times K}, \ldots, [\theta - \bepsilon_{\lfloor \btau/K \rfloor + 1} ]_{\times \remainder(\btau + j, K)}) \bigg ] \cdot \paren{\sum_{i=1}^{\lfloor \btau/K \rfloor + 1} \bepsilon_i},
\end{align}
with randomness in $\btau \sim \Unif\{0, W, \ldots, T - W\}$ and $\{\bepsilon_i\}_{i=1}^{\lceil T/K \rceil} \iid \calN(\bzero, \sigma^2 I_{d\times d})$.
\end{lemma}

\begin{proof}
~

For simplicity, we will denote $q = \lceil T/K \rceil$ and $r = \remainder(T, K)$.

First let's define a function $L^K: \Real^{d \cdot q} \rightarrow \Real $
\begin{align}
    L^K(\theta_1, \ldots, \theta_q) \coloneqq L([\theta_1]_{\times K}, \ldots, [\theta_{q-1}]_{\times K},[\theta_q]_{\times r}),
\end{align}

and a smoothed version of $L^K$ by $\widehat{L}^K: \Real^{d \cdot q} \rightarrow \Real $ with
\begin{align}
    \widehat{L}^K(\theta_1, \ldots, \theta_q) \coloneqq \E_{\{\bepsilon_i\}_{i=1}^q} L^K(\theta_1 + \epsilon_1, \ldots, \theta_q + \epsilon_q).
\end{align}

We notice that by definition, the $K$-smoothed loss function can be expressed as
\begin{align}
    \theta \mapsto \widehat{L}^K([\theta]_{\times q}).
\end{align}

In this form, the $K$-smoothed loss function is a simple composition of two functions: the first function maps $\theta$ to $q$-times repetition of $[\theta]_{\times q}$, while the second function is exactly $\widehat{L}^K$. For the first vector-valued function, we see that its Jacobian is given by:

\begin{align}
    \frac{d}{d\theta} (\theta \mapsto [\theta]_{\times q}) = \bone_q \otimes I_{d \times d} \in \Real^{qd \times d},
\end{align}
where $\bone_q \in \Real^q$ is a vector of $1$'s and $\otimes(\cdot, \cdot)$ is the Kronecker product operator. For the second function, we see that by the score function gradient estimator trick \citep{glynn1990likelihood}, 
\begin{align}
    \nabla_\Theta \widehat{L}^K(\theta_1, \ldots, \theta_q) = \frac{1}{\sigma^2} \E_{\{\bepsilon_i\}_{i=1}^q} L^K(\theta_1 + \bepsilon_1, \ldots, \theta_q + \bepsilon_q) \begin{bmatrix}
    \bepsilon_1 \\ \vdots \\ \bepsilon_q
    \end{bmatrix}.
\end{align}

Now we are ready to compute the gradient of the $K$-smoothed loss using chain rule (recall that we assume the gradients are column vectors):

\begin{align}
    & \nabla_\theta \paren{\theta \mapsto \widehat{L}^K([\theta]_{\times q})} \\
    = & \brck{\frac{d}{d\theta} (\theta \mapsto [\theta]_{\times q})}^\top \nabla_\Theta \widehat{L}^K \bigg \vert_{\Theta = [\theta]_{\times q}} \\
    =& (\bone_q \otimes I_{d \times d})^\top \E_{\{\bepsilon_i\}_{i=1}^q} \frac{1}{\sigma^2} L^K(\theta + \bepsilon_1, \ldots, \theta + \bepsilon_q) \begin{bmatrix} \bepsilon_1 \\ \vdots \\ \bepsilon_q \end{bmatrix} \\
    =&  \frac{1}{\sigma^2} \E_{\{\bepsilon_i\}_{i=1}^q} L^K(\theta + \bepsilon_1, \ldots, \theta + \bepsilon_q) \paren{(\bone_q \otimes I_{d \times d})^\top \begin{bmatrix} \bepsilon_1 \\ \vdots \\ \bepsilon_q \end{bmatrix}} \\
    =& \frac{1}{\sigma^2} \E_{\{\bepsilon_i\}_{i=1}^q} L^K(\theta + \bepsilon_1, \ldots, \theta + \bepsilon_q) (\sum_{i=1}^q \bepsilon_i).
\end{align}

Here the last step is by the algebra of Kronecker product. With this, we now consider the structure of $L^K$ as an average of losses over all time steps by converting this average into an expectation over truncation windows starting at $\btau$:

\begin{align}
     & \nabla_\theta \paren{\theta \mapsto \widehat{L}^K([\theta]_{\times q})} \\
    =& \frac{1}{\sigma^2} \E_{\{\bepsilon_i\}_{i=1}^q} L^K(\theta + \bepsilon_1, \ldots, \theta + \bepsilon_q) (\sum_{i=1}^q \bepsilon_i) \\
    =& \frac{1}{\sigma^2} \E_{\{\bepsilon_i\}_{i=1}^q} \frac{1}{T} \sum_{t=1}^T L_t([\theta + \bepsilon_1]_{\times K}, [\theta + \bepsilon_2]_{\times K}\ldots, [\theta + \bepsilon_{\lceil t/K \rceil}]_{\times \remainder(t, K)}) (\sum_{i=1}^q \bepsilon_i) \\
    =& \frac{1}{\sigma^2} \E_{\{\bepsilon_i\}_{i=1}^q} \frac{1}{T/W} \sum_{b=0}^{T/W - 1} \nonumber \\
    & \fourquad \bigg [ \frac{1}{W} \sum_{j=1}^W L_{Wb + j}([\theta + \bepsilon_1]_{\times K}, [\theta + \bepsilon_2]_{\times K}\ldots, [\theta + \bepsilon_{\lceil (Wb + j)/K \rceil}]_{\times \remainder(Wb + j, K)}) (\sum_{i=1}^q \bepsilon_i) \bigg ] \\
    =& \frac{1}{\sigma^2} \E_{\{\bepsilon_i\}_{i=1}^q} \frac{1}{T/W} \sum_{b=0}^{T/W - 1} \nonumber\\
    & \fourquad \bigg [ \frac{1}{W} \sum_{j=1}^W L_{Wb + j}([\theta + \bepsilon_1]_{\times K}, [\theta + \bepsilon_2]_{\times K}\ldots, [\theta + \blue{\bepsilon_{\lfloor Wb/K \rfloor + 1}}]_{\times \remainder(Wb + j, K)}) (\sum_{i=1}^q \bepsilon_i) \bigg ] \label{eqn:causality} 
\end{align}
Here the last step we observe that $\lceil (Wb + j) / K \rceil = \lfloor Wb/K \rfloor + 1$ for any $1\le j \le W \le K$.

Here we notice that in Equation \ref{eqn:causality}, for $i > \lfloor Wb/K \rfloor + 1$, there is independence between the random vector $\bepsilon_i$ and the term $L_{Wb + j}([\theta + \bepsilon_1]_{\times K}, [\theta + \bepsilon_2]_{\times K}\ldots, [\theta + \bepsilon_{\lfloor Wb/K \rfloor + 1} ]_{\times \remainder(Wb + j, K)})$. The expectation of the product between these independent terms is then $\bzero$ because $\E[\bepsilon_i] = \bzero$. As a result, we have the further simplification:
\begin{align}
     & \nabla_\theta \paren{\theta \mapsto \widehat{L}^K([\theta]_{\times q})} \\
    =& \frac{1}{\sigma^2} \E_{\{\bepsilon_i\}_{i=1}^q} \frac{1}{T/W} \sum_{b=0}^{T/W - 1}  \nonumber \\
    & \quad \brck{\frac{1}{W} \sum_{j=1}^W L_{Wb + j}([\theta + \bepsilon_1]_{\times K}, [\theta + \bepsilon_2]_{\times K}\ldots, [\theta + \bepsilon_{\lfloor Wb/K \rfloor + 1} ]_{\times \remainder(Wb + j, K)}) (\sum_{i=1}^{\lfloor Wb/K \rfloor + 1} \bepsilon_i)}\label{eqn:gpes_single_avg_form}\\
    =& \frac{1}{\sigma^2} \E_{\{\bepsilon_i\}_{i=1}^q} \blue{\E_{\btau}} \frac{1}{W} \sum_{j=1}^W L_{\btau + j}([\theta + \bepsilon_1]_{\times K}, [\theta + \bepsilon_2]_{\times K}\ldots, [\theta + \bepsilon_{\lfloor \btau/K \rfloor + 1} ]_{\times \remainder(\btau + j, K)}) (\sum_{i=1}^{\lfloor \btau/K \rfloor + 1} \bepsilon_i) \label{eqn:gpes_single_exp_form}.
\end{align}

Here the last step converts the average in Equation \ref{eqn:gpes_single_avg_form} into an expectation in Equation \ref{eqn:gpes_single_exp_form} by treating $Wb$ as the random variable $\btau$. By additionally averaging over the antithetic samples of the random variable in Equation \ref{eqn:gpes_single_exp_form}, we arrive at the unbiased estimator given in the Lemma.

\end{proof}

\subsection{Interpretation of Assumption \ref{assumption:linearity}}

\begin{assumption}
For a given fixed $\theta \in \Real^d$, for any $t \in [T]$, there exists a set of vectors $\{g^t_i \in \Real^d\}_{i=1}^t$, such that for any $\{v_i \in \Real^d\}_{i=1}^t$, the following equality holds:
{
\small
\begin{align}
    L_t(\theta + v_1, \theta + v_2, \ldots, \theta + v_t)
    - L_t(\theta - v_1, \theta - v_2, \ldots, \theta - v_t) = 2\sum_{i=1}^t (v_i)^\top (g^t_i)
\end{align}
}
\end{assumption}

\bigskip 

Here we show that when $L_t: \Real^{dt} \rightarrow \Real$ is a quadratic function, this assumption would hold with $g^t_i = \frac{\partial L_t}{\partial \theta_i}$.

If $L_t$ is a quadratic (assumption made in \citep{vicol21unbiased}), it can be expressed exactly as its second-order Taylor expansion. Then we have
\begin{align}
    L_t(\theta + v_1, \theta + v_2, \ldots, \theta + v_t) =& L_t([\theta]_{\times t}) + \paren{\nabla_\Theta L_t \bigg \vert_{\Theta=[\theta]_{\times t}}}^\top \begin{bmatrix} v_1 \\ \vdots \\ v_t \end{bmatrix} + \frac{1}{2} \begin{bmatrix} v_1 & \ldots & v_t \end{bmatrix} \nabla^2_\Theta L_t \begin{bmatrix} v_1 \\ \vdots \\ v_t \end{bmatrix} \\
    L_t(\theta - v_1, \theta - v_2, \ldots, \theta - v_t) =& L_t([\theta]_{\times t}) - \paren{\nabla_\Theta L_t \bigg \vert_{\Theta=[\theta]_{\times t}}}^\top  \begin{bmatrix} v_1 \\ \vdots \\ v_t \end{bmatrix} + \frac{1}{2} \begin{bmatrix} v_1 & \ldots & v_t \end{bmatrix} \nabla^2_\Theta L_t \begin{bmatrix} v_1 \\ \vdots \\ v_t \end{bmatrix} \label{eqn:quadratic_assumption}
\end{align}

Taking the difference of the above two equations, we have that 
\begin{align}
    L_t(\theta + v_1, \theta + v_2, \ldots, \theta + v_t) - L_t(\theta - v_1, \theta - v_2, \ldots, \theta - v_t) = 2 \paren{\nabla_\Theta L_t \bigg \vert_{\Theta=[\theta]_{\times t}}}^\top 
 \begin{bmatrix} v_1 \\ \vdots \\ v_t \end{bmatrix}.
\end{align}

We note that 
\begin{align}
    \nabla_\Theta L_t \bigg \vert_{\Theta=[\theta]_{\times t}} = \begin{bmatrix}
        \frac{\partial L_t}{\partial \theta_1} \\ \ldots \\ \frac{\partial L_t}{\partial \theta_t}
    \end{bmatrix}.
\end{align}

Plugging it into Equation \eqref{eqn:quadratic_assumption}, we have 
\begin{align}
    L_t(\theta + v_1, \theta + v_2, \ldots, \theta + v_t) - L_t(\theta - v_1, \theta - v_2, \ldots, \theta - v_t) = 2 \sum_{i=1}^t (v_i)^\top \frac{\partial L_t}{\partial \theta_i}.
\end{align}

Thus we see in the case of quadratic $L_t$, $g^t_i$ is just the partial derivative of $L_t$ with respect to $\theta_i$ (i.e., $\frac{\partial L_t}{\partial \theta_i}$). Hence we see that our assumptions generalize those made in \citep{vicol21unbiased}.

\subsection{Proof of Lemma \ref{lemma:gpesk_form_under_assumption}}
\setcounter{theorem}{3}
\begin{lemma}
Under Assumption \ref{assumption:linearity}, when $W=1$, {\small $\GPES_{K=c}(\theta) = \frac{1}{\sigma^2}\sum_{j=1}^{\lfloor \btau/c \rfloor + 1} \paren{\sum_{i=1}^{\lfloor \btau/c \rfloor + 1} \bepsilon_i} \bepsilon_j^\top g_{c,j}^{\btau + 1}$}, where the randomness lies in $\btau \sim \Unif\{0, 1, \ldots, T - 1\}$ and $\{\bepsilon_i\}\iid \calN(\bzero, \sigma^2 I_{d\times d})$.
\end{lemma}

\begin{proof}
We see that when $W=1$, we have $K=cW = c$ and
\begin{align}
& \GPES_{K=c}(\theta) \\
= & \frac{1}{2\sigma^2W} \sum_{j=1}^{\blue{W=1}} \bigg [ L_{\btau + j}([\theta + \bepsilon_1]_{\times K}, [\theta + \bepsilon_2]_{\times K}, \ldots, [\theta + \bepsilon_{\lfloor \btau/K \rfloor + 1} ]_{\times \remainder(\btau + j, K)})  \nonumber\\
& \qquad \quad  -  L_{\btau + j}([\theta - \bepsilon_1]_{\times K}, [\theta - \bepsilon_2]_{\times K}, \ldots, [\theta - \bepsilon_{\lfloor \btau/K \rfloor + 1} ]_{\times \remainder(\btau + j, K)}) \bigg ] \cdot \paren{\sum_{i=1}^{\lfloor \btau/K \rfloor + 1} \bepsilon_i} \\
= & \frac{1}{2\sigma^2} \bigg [ L_{\btau + 1}([\theta + \bepsilon_1]_{\times c}, [\theta + \bepsilon_2]_{\times c}, \ldots, [\theta + \bepsilon_{\lfloor \btau/c \rfloor + 1} ]_{\times \remainder(\btau + 1, c)}) \nonumber \\
& \qquad \quad -  L_{\btau + 1}([\theta - \bepsilon_1]_{\times c}, [\theta - \bepsilon_2]_{\times c}, \ldots, [\theta - \bepsilon_{\lfloor \btau/c \rfloor + 1} ]_{\times \remainder(\btau + 1, c)}) \bigg ] \cdot \paren{\sum_{i=1}^{\lfloor \btau/c \rfloor + 1} \bepsilon_i} \label{eqn:gpesk_form_simplified_w=1}
\end{align}
Applying Assumption \ref{assumption:linearity} to the difference in the last equation, we have
\begin{align}
    & L_{\btau + 1}([\theta + \bepsilon_1]_{\times c}, [\theta + \bepsilon_2]_{\times c}, \ldots, [\theta + \bepsilon_{\lfloor \btau/c \rfloor + 1} ]_{\times \remainder(\btau + 1, c)}) \nonumber \\
    & \qquad -  L_{\btau + 1}([\theta - \bepsilon_1]_{\times c}, [\theta - \bepsilon_2]_{\times c}, \ldots, [\theta - \bepsilon_{\lfloor \btau/c \rfloor + 1} ]_{\times \remainder(\btau + 1, c)}) \\
    =& 2\sum_{j=1}^{\lfloor \btau/c \rfloor + 1} \brck{\bepsilon_j^\top \paren{\sum_{i=c\cdot (j-1) + 1}^{\min \{\btau + 1, c\cdot j\}} g^{\btau + 1}_i}} \\
    =& 2 \sum_{j=1}^{\lfloor \btau/c \rfloor + 1} \bepsilon_j^\top g^{\btau + 1}_{c, j}
\end{align}

Plugging this equality in to Equation \ref{eqn:gpesk_form_simplified_w=1}, we get
\begin{align*}
    & \GPES_{K=c}(\theta) = \frac{1}{\sigma^2}  \sum_{j=1}^{\lfloor \btau/c \rfloor + 1} 
 \paren{\sum_{i=1}^{\lfloor \btau/c \rfloor + 1} \bepsilon_i}\bepsilon_j^\top g^{\btau + 1}_{c, j} 
\end{align*}
\end{proof}

\subsection{Proof of Theorem \ref{thm:variance}}
\begin{theorem}
When $W=1$ and under Assumption \ref{assumption:linearity}, the total variance of $\GPES_{K=c}(\theta)$ has the following form for any integer $c \in [1, T]$,
\small
\begin{align}
\tr(\Cov[\GPES_{K=c}(\theta)]) = \frac{(d+2)}{T} \sum_{t=1}^T \paren{ \|g^t\|_2^2} - \norm{\frac{1}{T} \sum_{t=1}^T g^t}_2^2 + \frac{1}{T} \sum_{t=1}^T \paren{\frac{d}{2}\sum_{j=1, j'=1}^{\lceil t/c \rceil} \norm{g^t_{c,j} - g^t_{c,j'}}_2^2}.
\end{align}
\end{theorem}

\begin{proof}
First we break down the total variance:
\begin{align}
    &\tr(\Cov[\GPES_{K=c}(\theta)]) \\
    =& \tr(\E \brck{\paren{\GPES_{K=c}(\theta) - \E[\GPES_{K=c}(\theta)]}\paren{\GPES_{K=c}(\theta) - \E[\GPES_{K=c}(\theta)]}^\top}) \\
    =& \tr (\E \GPES_{K=c}(\theta)\GPES_{K=c}(\theta)^\top) - \tr([\E \GPES_{K=c}(\theta)][\E \GPES_{K=c}(\theta)]^\top) \\
    =& (\E \tr[\GPES_{K=c}(\theta)\GPES_{K=c}(\theta)^\top]) - [\E \GPES_{K=c}(\theta)]^\top[\E \GPES_{K=c}(\theta)] \\
    =& \E \|\GPES_{K=c}(\theta)\|_2^2 - \|\E\GPES_{K=c}(\theta)\|_2^2 \\
    =& \E_{\btau} \E_{\{\bepsilon_i\} \mid \btau=t-1} \|\GPES_{K=c}(\theta)\|_2^2 - \|\E\GPES_{K=c}(\theta)\|_2^2 \\
    =& \frac{1}{T} \sum_{t=1}^{T} \underbrace{\E_{\{\bepsilon_i\} \mid \btau=t-1} \|\GPES_{K=c}(\theta)\|_2^2 }_{\circled{1}} - \underbrace{\|\E\GPES_{K=c}(\theta)\|_2^2}_{\circled{2}}
\end{align}

Thus to analytically express the trace of covariance, we need to separately derive $\circled{1}$ for any $t \in \{1, \ldots, T\}$ and $\circled{2}$.
\bigskip

\textbf{\circled{1} Expressing $\E_{\{\bepsilon_i\} \mid \btau=t-1} \|\GPES_{K=c}(\theta)\|_2^2$}.

Here we see that for a given $t \in [T]$, conditioning on $\btau = t-1$, we have
\begin{align}
& \GPES_{K=c}(\theta)\\
  =& \frac{1}{\sigma^2}  \sum_{j=1}^{\lfloor \btau/c \rfloor + 1} 
 \paren{\sum_{i=1}^{\lfloor \btau/c \rfloor + 1} \bepsilon_i}\bepsilon_j^\top g^{\btau + 1}_{c, j} \\
 =& \frac{1}{\sigma^2}  \sum_{j=1}^{\lfloor (t-1)/c \rfloor + 1} 
 \paren{\sum_{i=1}^{\lfloor (t-1)/c \rfloor + 1} \bepsilon_i}\bepsilon_j^\top g^{t}_{c, j} \\
 =& \frac{1}{\sigma^2}  \sum_{j=1}^{\lceil t/c \rceil} 
 \paren{\sum_{i=1}^{\lceil t/c \rceil} \bepsilon_i}\bepsilon_j^\top g^{t}_{c, j},\label{eqn:gpesk_analytical_form_conditioning_on_t}
\end{align}

where in the last step we use the fact that for any integer $t$, $\lfloor (t-1)/c \rfloor + 1 = \lceil t/c \rceil$. Because we are deriving expression \circled{1} for every value of $t$ separately, to simplify the notation, we define $n \coloneqq \lceil t/c \rceil$ and $a_j \coloneqq g^t_{c, j}$ for a given fixed $t$.

Then the expression in Equation \eqref{eqn:gpesk_analytical_form_conditioning_on_t} can be simplified as $\frac{1}{\sigma^2}  \sum_{j=1}^{n} \paren{\sum_{i=1}^{n} \bepsilon_i}\bepsilon_j^\top a_j$.

As a result, term \circled{1} can be expressed as
\begin{align}
    & \E_{\{\bepsilon\} \mid \btau=t-1} \|\GPES_{K=c}(\theta)\|_2^2 \\
=& \E_{\bepsilon} \norm{\frac{1}{\sigma^2}  \sum_{j=1}^{n} \paren{\sum_{i=1}^{n} \bepsilon_i}\bepsilon_j^\top a_j}_2^2 \\
= & \frac{1}{\sigma^4} \E_{\bepsilon} \brck{\sum_{i=1}^n \paren{\sum_{k=1}^n{\bepsilon_k}} \bepsilon_i^\top a_i}^\top \brck{\sum_{j=1}^n \paren{\sum_{l=1}^n{\bepsilon_l}} \bepsilon_j^\top a_j} \\
=& \frac{1}{\sigma^4} \sum_{i=1}^n \sum_{j=1}^n a_i^\top \E_{\bepsilon} \brck{\bepsilon_i \paren{\sum_{k=1}^n{\bepsilon_k}}^\top \paren{\sum_{l=1}^n{\bepsilon_l}} \bepsilon_j^\top} a_j \label{eqn:gpesk_conditioning_squared_2_norm_quadratic} \\
\end{align}

From Equation \eqref{eqn:gpesk_conditioning_squared_2_norm_quadratic}, we see that term \circled{1} is a quadratic in $\{a_i\}_{i=1}^n$ with each bilinear form's matrix determined by an expectation. Thus we break into different cases to evaluate the expectation $\brck{\sum_{k=1}^n \sum_{l=1}^n \E_{\bepsilon} \bepsilon_i \bepsilon_k^\top \bepsilon_l \bepsilon_j^\top}$ for different values of $i, j, k, l$.

\textbf{Begin of Cases}

\textbf{Case (I)} $i = j$.

\quad \textbf{(I.1)} If $k \neq i, l \neq i$,
\begin{align*}
& \E_{\bepsilon} \bepsilon_i \bepsilon_k^\top \bepsilon_l \bepsilon_j^\top \\
=& \E_{\bepsilon_i, \bepsilon_k, \bepsilon_l} \bepsilon_i \bepsilon_k^\top \bepsilon_l \bepsilon_i^\top \\
=& \E_{\bepsilon_i} \E_{\bepsilon_k, \bepsilon_l} \bepsilon_i \bepsilon_k^\top \bepsilon_l \bepsilon_i^\top \\
=& \E_{\bepsilon_i}  \bepsilon_i [\E_{\bepsilon_k, \bepsilon_l} \bepsilon_k^\top \bepsilon_l] \bepsilon_i^\top
\end{align*}

\qquad \textbf{(I.1.a)} If $k \neq i, l \neq i$, and $k \neq l$, then $\E_{\bepsilon_k, \bepsilon_l} \bepsilon_k^\top \bepsilon_l=0$ and $\E_{\bepsilon} \bepsilon_i \bepsilon_k^\top \bepsilon_l \bepsilon_j^\top = \bzero$.

\qquad \textbf{(I.1.b)} If $k \neq i, l \neq i$, and $k = l$.
\begin{align*}
     &\E_{\bepsilon_i}  \bepsilon_i [\E_{\bepsilon_k, \bepsilon_l} \bepsilon_k^\top \bepsilon_l] \bepsilon_i^\top \\
     =& \E_{\bepsilon_i}  \bepsilon_i [\E_{\bepsilon_k} \bepsilon_k^\top \bepsilon_k] \bepsilon_i^\top \\
     =& (d\sigma^2) \sigma^2 I_d \\
     =& d\sigma^4 I_d
\end{align*}

\quad \textbf{(I.2)} If $k = i, l \neq i$,
\begin{align*}
    & \E_{\bepsilon} \bepsilon_i \bepsilon_k^\top \bepsilon_l \bepsilon_j^\top \\
    =& \E_{\bepsilon} \bepsilon_i \bepsilon_i^\top \bepsilon_l \bepsilon_i^\top \\
    =& \E_{\bepsilon_i} \bepsilon_i \bepsilon_i^\top \E_{\bepsilon_l} [\bepsilon_l] \bepsilon_i^\top \\
    =& \bzero
\end{align*}

\clearpage
\quad \textbf{(I.3)} If $k \neq i, l = i$, similarly as \textbf{(I.2)},
\begin{align*}
    & \E_{\bepsilon} \bepsilon_i \bepsilon_k^\top \bepsilon_l \bepsilon_j^\top \\
    =& \E_{\bepsilon} \bepsilon_i \bepsilon_k^\top \bepsilon_i \bepsilon_i^\top \\
    =& \E_{\bepsilon_i} \bepsilon_i \E_{\bepsilon_k} [\bepsilon_k^\top] \bepsilon_i \bepsilon_i^\top \\
    =& \bzero
\end{align*}

\quad \textbf{(I.4)} If $k = i, l = i$, by Isserlis’ theorem (derivation see Supplementary material A.2 in \citep{maheswaranathan2019guided}),
\begin{align*}
& \E_{\bepsilon} \bepsilon_i \bepsilon_k^\top \bepsilon_l \bepsilon_j^\top \\
=& \E_{\bepsilon_i} \bepsilon_i \bepsilon_i^\top \bepsilon_i \bepsilon_i^\top \\
=& (d + 2) \sigma^4 I_d
\end{align*}

\quad Combining \textbf{(I.1)} to \textbf{(I.4)}, we see that for the case of $i = j$,
\begin{align}
    \sum_{k=1}^n \sum_{l=1}^n \E_{\bepsilon} \bepsilon_i \bepsilon_k^\top \bepsilon_l \bepsilon_j^\top = \brck{\sum_{k\neq i} d\sigma^4 I_d} + (d + 2)\sigma^4 I_d = (nd + 2) \sigma^4 I_d \label{eq:sum_product_ieqj}
\end{align}

\textbf{Case (II)} $i \neq j$.

\quad \textbf{(II.1)} $k \neq i, k \neq j$.

\qquad \textbf{(II.1.a)} If $k \neq i, k \neq j$, and additionally $k \neq l$, 
\begin{align*}
& \E_{\bepsilon} \bepsilon_i \bepsilon_k^\top \bepsilon_l \bepsilon_j^\top \\
=& \E_{\bepsilon_i,\bepsilon_j, \bepsilon_l} \bepsilon_i (\E_{\bepsilon_k} [\bepsilon_k])^\top \bepsilon_l \bepsilon_j^\top \\
=& \bzero
\end{align*}

\qquad \textbf{(II.1.b)} If $k \neq i, k \neq j$, and $k = l$,

\qquad \qquad \qquad After pulling out $\E_{\bepsilon_k} [\bepsilon_k^\top \bepsilon_k]$, because $i \neq j$, we still have the expectation $=0$.
\begin{align*}
& \E_{\bepsilon} \bepsilon_i \bepsilon_k^\top \bepsilon_l \bepsilon_j^\top \\
=& \E_{\bepsilon} \bepsilon_i \E[\bepsilon_k^\top \bepsilon_k] \bepsilon_j^\top \\
=& [\E_{\bepsilon_i,\bepsilon_j} \bepsilon_i \bepsilon_j^\top] \cdot \E[\bepsilon_k^\top \bepsilon_k] \\
=& \bzero \cdot \E[\bepsilon_k^\top \bepsilon_k] \\
=& \bzero
\end{align*}

\quad \textbf{(II.2)} $k = i$,

\qquad \textbf{(II.2.a)} if $l \neq i, l \neq j$. This is similar to \textbf{(II.1)} as we can swap the position of $k$ and $j$:
\begin{align*}
    \bepsilon_i \bepsilon_k^\top \bepsilon_l \bepsilon_j^\top = \bepsilon_i \bepsilon_l^\top \bepsilon_k \bepsilon_j^\top.
\end{align*}
\qquad \qquad \quad \; Thus we have $\E_{\bepsilon} \bepsilon_i \bepsilon_k^\top \bepsilon_l \bepsilon_j^\top = \bzero$.

\qquad \textbf{(II.2.b)} if $l = i$,
\begin{align*}
    & \E_{\bepsilon} \bepsilon_i \bepsilon_k^\top \bepsilon_l \bepsilon_j^\top \\
    =& E_{\bepsilon} \bepsilon_i \bepsilon_i^\top \bepsilon_i \bepsilon_j^\top \\
    =& E_{\bepsilon_i} \bepsilon_i \bepsilon_i^\top \bepsilon_i \E_{\bepsilon_j}\bepsilon_j^\top \\
    =& \bzero
\end{align*}

\qquad \textbf{(II.2.c)} if $l = j$,
\begin{align*}
    & \E_{\bepsilon} \bepsilon_i \bepsilon_k^\top \bepsilon_l \bepsilon_j^\top \\
=& \E_{\bepsilon_i, \bepsilon_j} \bepsilon_i \bepsilon_i^\top \bepsilon_j \bepsilon_j^\top \\
=& \E_{\bepsilon_i} [\bepsilon_i \bepsilon_i^\top] \E_{\bepsilon_j} [\bepsilon_j \bepsilon_j^\top] \\
=& \sigma^4 I_d
\end{align*}

\clearpage
\quad \textbf{(II.3)} $k = j$,

\qquad \textbf{(II.3.a)} if $l \neq i, l \neq j$. Again similar to \textbf{(II.1)} by swapping $k$ and $j$, we have $\E_{\bepsilon} \bepsilon_i \bepsilon_k^\top \bepsilon_l \bepsilon_j^\top = \bzero$.

\qquad \textbf{(II.3.b)} if $l = i$. By swapping the position, we have
\begin{align*}
& \E_{\bepsilon} \bepsilon_i \bepsilon_k^\top \bepsilon_l \bepsilon_j^\top \\
=& \E_{\bepsilon_i, \bepsilon_j} \bepsilon_i (\bepsilon_j^\top \bepsilon_i) \bepsilon_j^\top \\
=& \E_{\bepsilon_i, \bepsilon_j} \bepsilon_i (\bepsilon_i^\top \bepsilon_j) \bepsilon_j^\top \\
=& \E_{\bepsilon_i} [\bepsilon_i \bepsilon_i^\top] \E_{\bepsilon_j} [\bepsilon_j \bepsilon_j^\top] \\
=& \sigma^4 I_d
\end{align*}

\qquad \textbf{(II.3.c)} $l = j$.
\begin{align*}
    & \E_{\bepsilon} \bepsilon_i \bepsilon_k^\top \bepsilon_l \bepsilon_j^\top \\
    =& \E_{\bepsilon} \bepsilon_i \bepsilon_j^\top \bepsilon_j \bepsilon_j^\top \\
    =& \E_{\bepsilon_i} \bepsilon_i \E_{\bepsilon_j} \bepsilon_j^\top \bepsilon_j \bepsilon_j^\top \\
    =& \bzero
\end{align*}

As a result, when we have $i \neq j$, the total sum over all the cases \textbf{(II.1)} - \textbf{(II.3)} is
\begin{align}
    \sum_{k=1}^n \sum_{l=1}^n \E_{\bepsilon} \bepsilon_i \bepsilon_k^\top \bepsilon_l \bepsilon_j^\top = 2 \sigma^4 I_d \label{eq:sum_product_ineqj}
\end{align}
\textbf{End of all Cases}.

Using the result in \eqref{eq:sum_product_ieqj} and \eqref{eq:sum_product_ineqj}, we see that for the fixed time step $t$ and $n \coloneqq \lceil t/c \rceil$, we have
{
\allowdisplaybreaks
\begin{align}
        & \E_{\{\bepsilon\} \mid \btau=t-1} \|\GPES_{K=c}(\theta)\|_2^2 \\
    =& \frac{1}{\sigma^4} \paren{\sum_{i=1}^n (nd + 2) \sigma^4 \|a_i\|_2^2 + \sum_{i\neq j} 2\sigma^4 a_i^{\top} a_j} \\
    =& \sum_{i=1}^n (nd + 2) \|a_i\|_2^2 + \sum_{i\neq j} 2 a_i^\top a_j \\
    =& \paren{\sum_{i=1}^n (d + 2) \|a_i\|_2^2 + \sum_{i=1}^n (n-1)d \|a_i\|_2^2} + \paren{\sum_{i\neq j} (d+2) [a_i]^\top a_j - \sum_{i\neq j} d a_i^\top a_j} \\
    =& \paren{\sum_{i=1}^n (d + 2) \|a_i\|_2^2 + \sum_{i\neq j} (d+2) [a_i]^\top a_j} + 
    \paren{\sum_{i=1}^n (n-1)d \|a_i\|_2^2 - \sum_{i\neq j} d [a_i]^\top a_j} \\
    =& (d+2) \norm{\sum_{i=1}^n a_i}_2^2 + \frac{d}{2} \paren{\sum_{i=1}^n 2(n-1) \|a_i\|_2^2 - \sum_{i\neq j} 2a_i^\top a_j} \\
    =& (d+2) \norm{\sum_{i=1}^n a_i}_2^2 + \frac{d}{2} \paren{\sum_{i=1}^n (n-1) \|a_i\|_2^2 + \sum_{j=1}^n (n-1) \|a_j\|_2^2 - \sum_{i\neq j} 2a_i^\top a_j} \\
    =& (d+2) \norm{\sum_{i=1}^n a_i}_2^2 + \frac{d}{2} \paren{\sum_{i=1}^n \sum_{j\neq i} \|a_i\|_2^2 + \sum_{j=1}^n \sum_{i\neq i} \|a_j\|_2^2 - \sum_{i\neq j} 2a_i^\top a_j} \\
    =& (d+2) \norm{\sum_{i=1}^n a_i}_2^2 + \frac{d}{2} \paren{\sum_{i=1, j=1, i\neq j}^n \|a_i\|_2^2 + \|a_j\|_2^2 - 2a_i^\top a_j} \\
    =& (d+2) \norm{\sum_{i=1}^n a_i}_2^2 + \frac{d}{2} \sum_{i=1, j=1, i\neq j}^n \|a_i - a_j\|_2^2
\end{align}
}

Now we substitute $a_i$ with $g^t_{c,i}$ and $n$ back with $\lceil t/c \rceil$, we have
\begin{align}
& \E_{\{\bepsilon\} \mid \btau=t-1} \|\GPES_{K=c}(\theta)\|_2^2 \\
=& (d+2) \norm{\sum_{i=1}^{\lceil t/c \rceil} g^t_{c,i}}_2^2 + \frac{d}{2} \sum_{i=1, j=1, i\neq j}^{\lceil t/c \rceil} \|g^t_{c,i} - g^t_{c,j}\|_2^2 \\
=& (d+2) \norm{g^t}_2^2 + \frac{d}{2} \sum_{j=1, j'=1}^{\lceil t/c \rceil} \|g^t_{c,j} - g^t_{c,j'}\|_2^2     \label{eqn:circle1}
\end{align}
This completes the derivation for \circled{1}.

\textbf{\circled{2} Expressing $\|\E\GPES_{K=c}(\theta)\|_2^2$}

This amounts to computing the expectation of the $\GPES_{K=c}(\theta)$ estimator:
{
\allowdisplaybreaks
\begin{align}
    & \E \GPES_{K=c}(\theta) \\
    =& \E_{\btau} \E_{\bepsilon} \GPES_{K=c}(\theta) \\
    =& \frac{1}{T} \sum_{t=1}^{T} \E_{\bepsilon \mid \btau = t-1} \GPES_{K=c}(\theta) \\
    =& \frac{1}{T}  \sum_{t=1}^{T} \frac{1}{\sigma^2} \E_{\bepsilon} \sum_{j=1}^{\lceil t/c \rceil} 
 \paren{\sum_{i=1}^{\lceil t/c \rceil} \bepsilon_i}\bepsilon_j^\top g^{t}_{c, j} \\
    =& \frac{1}{T}  \sum_{t=1}^{T} \frac{1}{\sigma^2} \sum_{j=1}^{\lceil t/c \rceil} 
 \sum_{i=1}^{\lceil t/c \rceil} \E_{\bepsilon} [\bepsilon_i\bepsilon_j^\top] g^{t}_{c, j} \\
    =& \frac{1}{T}  \sum_{t=1}^{T} \frac{1}{\sigma^2} \sum_{j=1}^{\lceil t/c \rceil} \sigma^2 g^{t}_{c, j} \\
    =& \frac{1}{T}  \sum_{t=1}^{T}  \sum_{j=1}^{\lceil t/c \rceil} g^{t}_{c, j} \\
    =& \frac{1}{T}  \sum_{t=1}^{T}  g^t
\end{align}
}

Thus we have 
\begin{align}
    \|\E\GPES_{K=c}(\theta)\|_2^2 = \norm{ \frac{1}{T}  \sum_{t=1}^{T}  g^t}_2^2.
    \label{eqn:circle2}
\end{align}
This completes the derivation for \circled{2}.

Combining the result we have for \circled{1} (Equation \eqref{eqn:circle1}) and \circled{2} (Equation \eqref{eqn:circle2}), we have
\begin{align}
    & \tr(\Cov[\GPES_{K=c}(\theta)]) \\
    =& \frac{1}{T} \sum_{t=1}^{T} \E_{\{\bepsilon_i\} \mid \btau=t-1} \|\GPES_{K=c}(\theta)\|_2^2 - \|\E\GPES_{K=c}(\theta)\|_2^2 \\
    =& \frac{(d+2)}{T} \sum_{t=1}^T \paren{ \|g^t\|_2^2} - \norm{\frac{1}{T} \sum_{t=1}^T g^t}_2^2 + \frac{1}{T} \sum_{t=1}^T \paren{\frac{d}{2}\sum_{j=1, j'=1}^{\lceil t/c \rceil} \norm{g^t_{c,j} - g^t_{c,j'}}_2^2} \label{app_eqn:variance}
\end{align}
This completes the proof for Theorem \ref{thm:variance}.
\end{proof}

\subsection{Proof of Corollary \ref{corollary:w=1}}
\begin{corollary}
Under Assumption \ref{assumption:linearity}, when $W=1$, the gradient estimator $\GPES_{K=T}(\theta)$ has the smallest total variance among all $\{\GPES_K: K \in [T]\}$ estimators.
\end{corollary}

\begin{proof}
We notice that in Equation \ref{app_eqn:variance}, the only term that depends on $c$ is the third term:
\begin{align}
\frac{1}{T} \sum_{t=1}^T \paren{\frac{d}{2}\sum_{j=1, j'=1}^{\lceil t/c \rceil} \norm{g^t_{c,j} - g^t_{c,j'}}_2^2}.
\end{align}
This term can be made zero by having $c=T$: in this case, $\lceil t/T\rceil = 1$ for any $t \in [T]$, hence there is only one $g^t_{T, 1}$ to compare against itself. Thus having $c=T$ minimizes the total variance among the class of $\GPES_{K=c}$ estimators.
\end{proof}

\subsection{Proof of Corollary \ref{corollary:w>1}}
\begin{corollary}
    Under Assumption \ref{assumption:linearity}, when $W$ divides $T$, the NRES gradient estimator has the smallest total variance among all $\GPES_{K=cW}$ estimators $c \in \Integer \cap [1, T/W]$.
\end{corollary}

\begin{proof}
Here the key idea is that because $W$ divides $T$, we can define a mega unrolled computation graph and apply Corollary \ref{corollary:w=1}.

Specifically, we consider a mega UCG where the inner state is represented by all the states within a size $W$ truncation window. The total horizon length of this mega UCG is $T' = T/W$. 
The initial state of the mega UCG is given by
\begin{align}
    S_0 \coloneqq (s_0, \ldots, s_0) \in \Real^{Wp}.
\end{align}
For time step $t' \in \{1,\ldots,  T/W\}$ in this mega UCG, the mega-state $S_{t'}$ is given by the concatenation of all the states in a given original truncation window
\begin{align}
    S_{t'} \coloneqq (s_{(t'-1)\cdot W + 1}, \ldots, s_{t'\cdot W}) \in \Real^{Wp},
\end{align}

The learnable parameter in this problem is still $\theta \in \Real^d$. The transition dynamics $u_{t'}: \Real^{Wp} \times \Real^d \rightarrow \Real^{Wp}$ works as follows:
\begin{align}
    u_{t'}: \paren{S_{t'} = \begin{bmatrix}  s_{(t'-1)\cdot W + 1}\\ s_{(t'-1)\cdot W + 2} \\ \vdots \\ s_{t'\cdot W}\end{bmatrix}, \; \theta} \mapsto S_{t'+1} \coloneqq \begin{bmatrix} f_{t'W + 1}( s_{t'\cdot W}, \theta) \\ f_{t'W + 2}(f_{t'W + 1}( s_{t'\cdot W}, \theta), \theta) \\ \vdots\\  f_{(t'+1)W}(\ldots(f_{t'W + 2}(f_{t'W + 1}( s_{t'\cdot W}, \theta), \theta)\ldots, \theta).\end{bmatrix}
\end{align}
Here, the mega transition only look at the last time step's original state in $S_{t'}$ and unroll it forward $W$ steps to form the next mega state.

The mega loss function for step $t'$ is defined as $\ell^s_{t'}: \Real^{Wp} \rightarrow \Real$:
\begin{align}
    \ell^s_{t'}((s_{(t'-1)\cdot W + 1}, \ldots, s_{t'\cdot W})) = \frac{1}{W} \sum_{i=1}^WL^s_{(t'-1)W + i}(s_{(t'-1)W + i}),
\end{align}
which simply averages each original inner state's loss within the truncation window.

Here, when we consider a truncation window of size $W$ in the original graph, it is equivalent to a truncation window of size 1 in the mega UCG. To apply Corollary \ref{corollary:w=1} to this graph, we only need to make sure Assumption \ref{assumption:linearity} holds for this mega graph. Here we see that, by Assumption \ref{assumption:linearity} on this original graph, we have

\begin{align*}
    & \ell_{t'}(\theta + v_1, \ldots, \theta + v_{t'}) - \ell_{t'}(\theta - v_1, \ldots, \theta - v_{t'}) \\
    =& \frac{1}{W} \sum_{j=1}^W \brck{L_{W(t'-1) + j}([\theta + v_1]_{\times W}, \ldots, [\theta + v_{t'}]_{\times j}) - L_{W(t'-1) + j}([\theta - v_1]_{\times W}, \ldots, [\theta - v_{t'}]_{\times j})} \\
    =& \frac{1}{W} \sum_{j=1}^W  2 \sum_{i'=1}^{t'} [v_{i'}]^\top g^{W(t'-1) + j}_{W, i'}  \\
    =&  2 \sum_{i'=1}^{t'} [v_{i'}]^\top \paren{ \frac{1}{W} \sum_{j=1}^W g^{W(t'-1) + j}_{W, i'}}
\end{align*}

Thus for this mega graph, the set of vectors satisfying Assumption \ref{assumption:linearity} is $g^{t'}_{i'} = \paren{ \frac{1}{W} \sum_{j=1}^W g^{W(t'-1) + j}_{W, i'}}$.

With the mega graph satisfying the assumption in Corollary \ref{corollary:w=1}, we see that when the truncation window is of size $1$ in the mega graph, the gradient estimator $\NRES$ on this mega graph has smaller trace of covariance than other $\GPES_K=c$ estimators on this mega graph. Here, running $\NRES$ on the mega-graph is equivalent to runing $\NRES$ on the original graph, while the gradient estimator $\GPES_{K=c}$ on the mega-graph is equivalent to the gradient estimator $\GPES_{K=cW}$ on the original UCG. Thus we have proved that the $\NRES$ gradient estimator on the original graph has the smallest total variance among all $\GPES_{K=cW}$ estimators, thus completing the proof.

\end{proof}

\subsection{Proof of Theorem \ref{thm:nres_fulles_comparison}}
\begin{theorem}
Under Assumption \ref{assumption:linearity}, for any $W$ that divides $T$, if
{\small
\begin{align}
\label{eq:direction_assumption_app}
    \sum_{k=1}^{T/W} \norm{\sum_{t=W\cdot (k-1) + 1}^{W\cdot k} g^t}_2^2 \le \frac{d+1}{d+2}\norm{\sum_{j=1}^{T/W} \sum_{t=W\cdot (k-1) + 1}^{W\cdot k} g^t}_2^2,
\end{align}
}
then {$\tr(\Cov(\frac{1}{T/W}  \sum\limits_{i=1}^{T/W} \NRES_i(\theta)) \le \tr(\Cov(\FullES(\theta))$} where $\NRES_i(\theta)$ are iid $\NRES$ estimators.
\end{theorem}

\begin{proof}
We first analytically express the $\FullES$ gradient estimator. Under the Assumption \ref{assumption:linearity},

\begin{align}
    & \FullES(\theta) \\
    =& \frac{1}{2\sigma^2} \big [ L([\theta + \bepsilon]_{\times T}) - L([\theta - \bepsilon]_{\times T}) \big] \bepsilon \\
    =& \frac{1}{\sigma^2}  \bepsilon \bepsilon^\top \paren{\frac{1}{T} \sum_{t=1}^T g^t}
\end{align}

Because $\E_{\bepsilon} \bepsilon \bepsilon^T \bepsilon \bepsilon^T = (d+2) \sigma^2 I_{d\times d}$, we can see that 
\begin{align}
    \tr(\Cov[\FullES(\theta)]) =& (d+2) \norm{\frac{1}{T} \sum_{t=1}^T g^t}_2^2 -  \norm{\frac{1}{T} \sum_{t=1}^T g^t}_2^2 \\
    =& (d+1) \norm{\frac{1}{T} \sum_{t=1}^T g^t}_2^2
\end{align}

From Theorem \ref{thm:variance} and Corollary \ref{corollary:w>1}, we can see that the trace of covariance for a single $\NRES$ when $W \ge 1$ is given by
\begin{align}
    \tr(\Cov(\NRES(\theta)) = \frac{(d+2)}{T/W} \sum_{k=1}^{T/W} \paren{ \norm{\frac{1}{W}\sum_{t=W\cdot(k-1) + 1}^{W\cdot k} g^t}_2^2} - \norm{\frac{1}{T} \sum_{t=1}^T g^t}_2^2
\end{align}

When we average over $T/W$ \textit{i.i.d.} $\NRES$ workers, the trace of covariance of the average is scaled by $\frac{1}{T/W}$:
{
\allowdisplaybreaks
\begin{align}
& \tr(\Cov(\frac{1}{T/W}  \sum\limits_{i=1}^{T/W} \NRES_i(\theta))\\
=& \frac{1}{T/W} \tr(\Cov(\NRES(\theta))\\
=& \frac{1}{T/W} \frac{(d+2)}{T/W} \frac{1}{W^2} \sum_{k=1}^{T/W} \paren{ \norm{\sum_{t=W\cdot(k-1) + 1}^{W\cdot k} g^t}_2^2} - \frac{1}{T/W}\norm{\frac{1}{T} \sum_{t=1}^T g^t}_2^2 \\
\le& \frac{(d+2)}{T^2} \sum_{k=1}^{T/W} \paren{ \norm{\sum_{t=W\cdot(k-1) + 1}^{W\cdot k} g^t}_2^2} \\
\le& \frac{(d+2)}{T^2} \frac{d+1}{d+2}\norm{\sum_{k=1}^{T/W} \sum_{t=W\cdot (k-1) + 1}^{W\cdot k} g^t}_2^2 \textrm{(using the condition in Theorem~\ref{thm:nres_fulles_comparison})}\\
=& \frac{(d+1)}{T^2} \norm{\sum_{k=1}^{T/W} \sum_{t=W\cdot (k-1) + 1}^{W\cdot k} g^t}_2^2 \\
=& \frac{(d+1)}{T^2} \norm{\sum_{t=1}^T g^t}_2^2 \\
=& (d+1) \norm{\frac{1}{T} \sum_{t=1}^T g^t}_2^2 \\
=& \tr(\Cov[\FullES(\theta)])
\end{align}
}

This completes the proof.

\end{proof}

\section{Experiments}
In this section, we first provide additional experiment results for each of the three applications we consider in the main paper. We next describe the experiment details and hyperparameters used for these experiments. We finally describe the computation resources needed to run these experiments and our implementation.

\subsection{Additional Experiment Results}
\label{app_subsec:experiment_result}

\subsubsection{Learning dynamical system parameters}
\paragraph{Visualizing the Lorenz system parameter learning loss surface.}

We have plotted the training loss surface of the Lorenz system parameter learning problem in the left panel of Figure~\ref{fig:lorenz}(a) in the main paper. Here we provide a larger version of the same figure in Figure~\ref{app_fig:lorenz_loss_surface}(a). In addition, we plot the losses along the line segment connecting the groundtruth $\theta_{\mathrm{gt}}$ and $\theta_{\mathrm{init}}$ in Figure~\ref{app_fig:lorenz_loss_surface}(b). (Notice that this is just a demonstration of the sensitivity of the loss surface; the optimization of $\theta$ are not constrained to this line segment.) We see that the loss surface have many local fluctuations and suboptimal local minima. In this case, the direction of the negative gradient is often non-informative for global optimization --- it would frequently point in a direction opposite to the global direction of loss decrease.
\begin{figure}[h]
\centering
    \begin{overpic}[width=0.33\textwidth,percent]{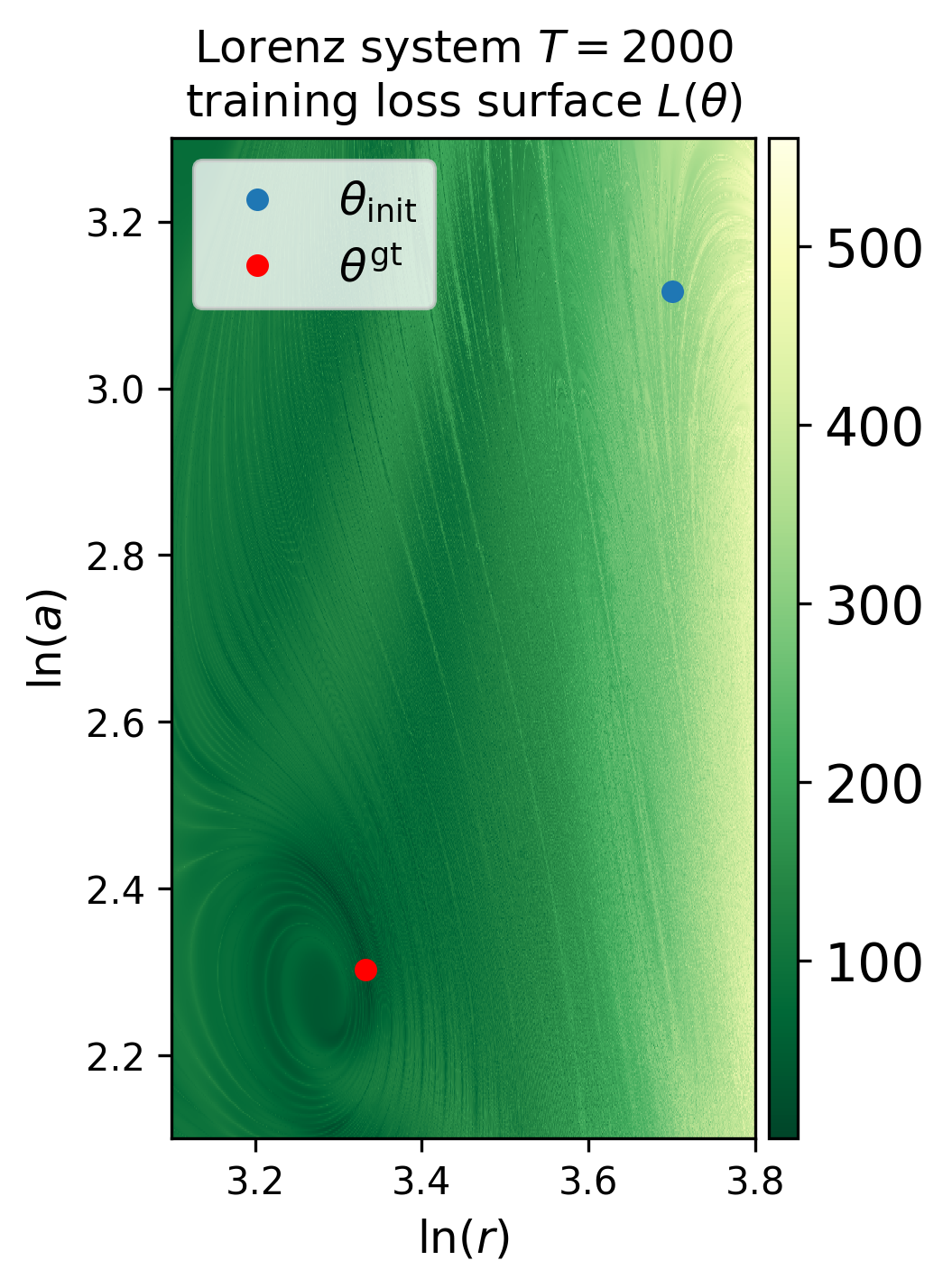}
        \put(0.0,90){\textbf{(a)}}
    \end{overpic}
    \;\;
    \raisebox{0.1\height}{
        \begin{overpic}[width=0.43\textwidth,percent]{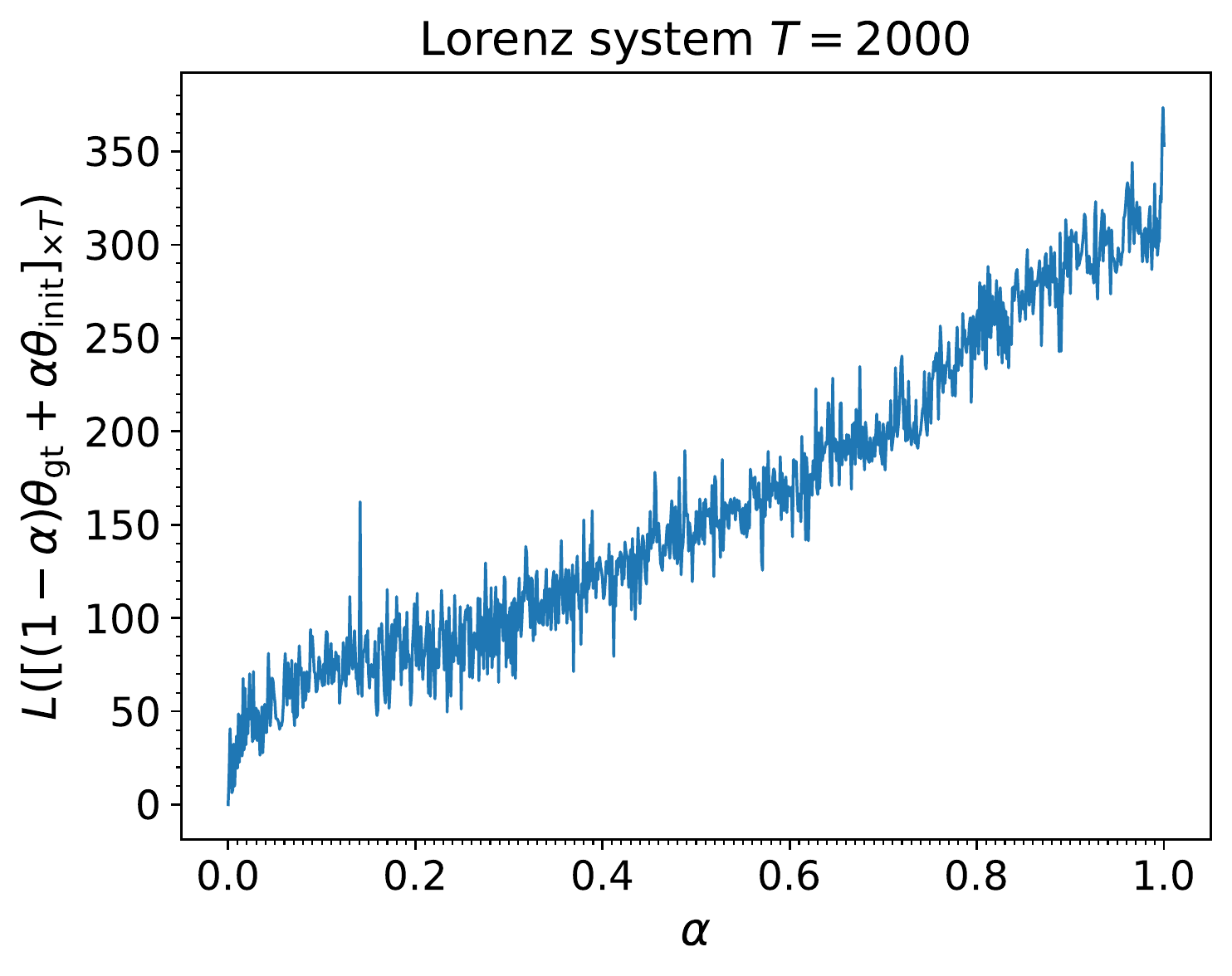}
            \put(10.0,75){\textbf{(b)}}
        \end{overpic}
    }
    \caption{(a) The extremely locally sensitive training loss surface in the Lorenz system parameter learning problem. (b) The training loss on the line segment in the parameter space connecting the groundtruth $\theta_{\mathrm{gt}}$ and the initialization $\theta_{\mathrm{init}}$. Because of the high sharpness and existence of many suboptimal local minima, automatic differentiation methods are ineffective to optimize this loss.}
    \label{app_fig:lorenz_loss_surface}
\end{figure}

\paragraph{Measuring progress on the Lorenz system's loss surface.}
As we have seen in Figure~\ref{app_fig:lorenz_loss_surface} that the training loss surface is highly non-smooth, it is difficult to visually compare different gradient estimation methods' performance through their non-smoothed training loss convergences as the losses have significant fluctuation for all methods. Instead, we measure the test loss (denoted by \texttt{loss} in Figure~\ref{fig:lorenz}(b)) instead of the non-smoothed training loss by sampling the random initial state $s_0 \sim \calN((1.2, 1.3, 1.6),\, 0.01\,I_{3\times3})$. Because this test loss considers a distribution of initial states, it is much smoother and helps with better visual comparisons. Besides, the test loss is a better indicator of predictive generalization to novel initial state conditions.

\paragraph{AD methods perform worse than ES methods.}
In Figure~\ref{fig:lorenz}(b) in the main paper, we have shown that $\NRES$ outperforms other evolution strategies baselines on the Lorenz system learning task. Here we additionally include the performance of four popular automatic differentiation methods $\BPTT$, $\TBPTT$, $\UORO$, and $\DODGE$ (discussed in Section~\ref{app_sec:related_work} in the Appendix) on the same task in Figure~\ref{app_fig:lorenz_ad_es_comparison}. We notice that these 4 AD methods all perform worse than the 4 ES methods. Thus, $\NRES$ is still the best among all the methods considered in this paper.

\begin{figure}[h]
    \centering
    \includegraphics[width=0.75\textwidth]{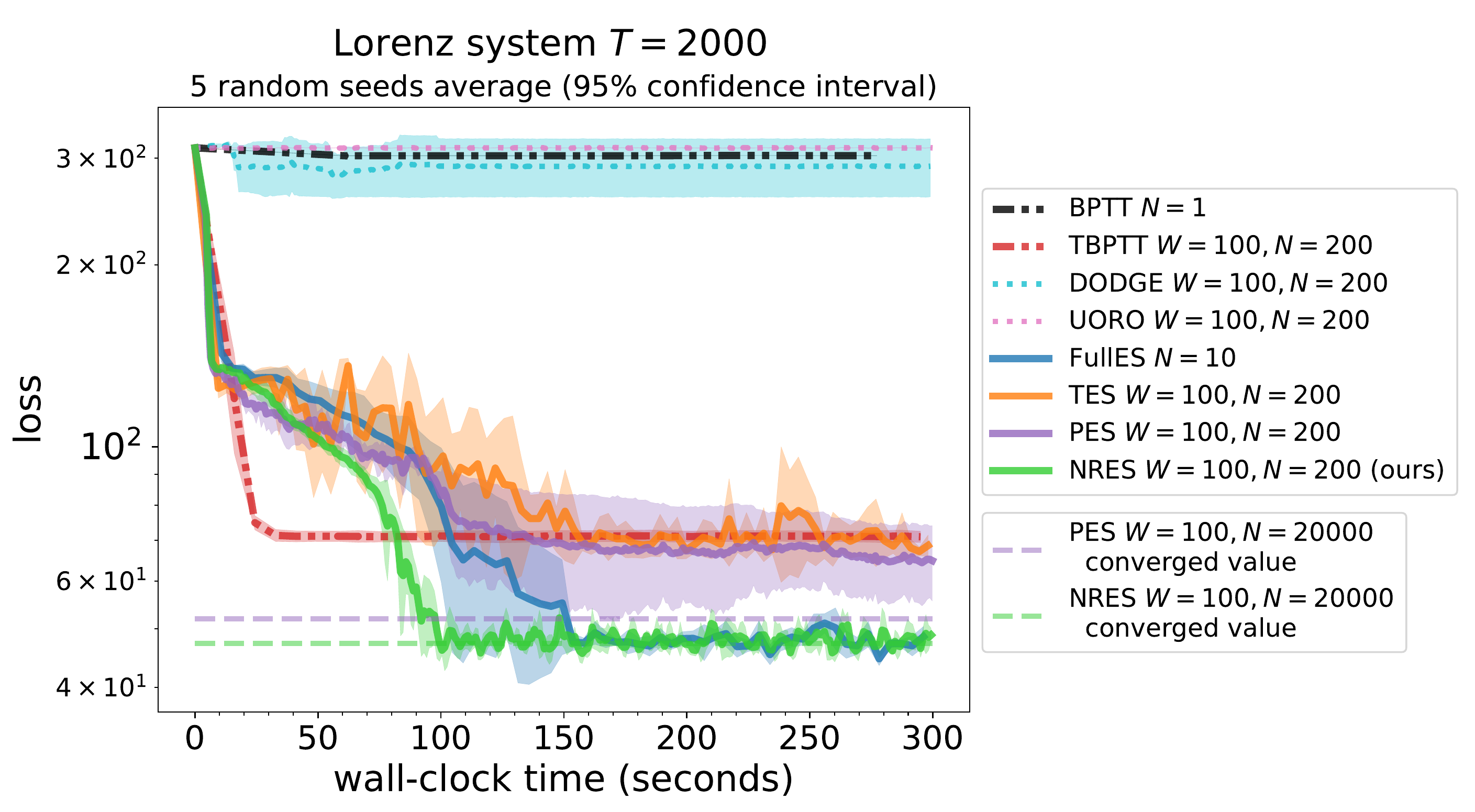}
    \caption{Different automatic differentiation and evolution strategies methods' loss convergence on the Lorenz system parameter learning problem. AD methods ($\BPTT$, $\TBPTT$, $\UORO$, $\DODGE$) all perform worse than ES methods, justifying the need for ES methods on problems with chaotic loss surfaces. Among all the methods, our proposed method $\NRES$ converges the fastest.}
    \label{app_fig:lorenz_ad_es_comparison}
\end{figure}

\subsubsection{Meta-training learned optimizers}
\label{app_subsubsec:lopt}

\paragraph{AD methods perform worse than ES methods.} In Figure~\ref{fig:lopt}(b) in the main paper, we have shown that $\NRES$ outperforms other ES baselines on the task to meta-learn a learned optimizer model to train a 3-layer MultiLayer Perceptron (MLP) on the Fashion MNIST dataset. Here we additionally include the performance of four popular automatic differentiation methods $\BPTT$, $\TBPTT$, $\UORO$, and $\DODGE$ (discussed in Section~\ref{app_sec:related_work} in the Appendix) on the same task in Figure~\ref{app_fig:lopt_ad_es_comparison_fashion_mnist}. It is worth noting that we focus on the training loss range $[0.5, \ln(10)]$ in Figure~\ref{fig:lopt}(a) in the main paper to make the comparisons among different ES methods more visually obvious. ($\ln(10)$ is chosen because random guessing on Fashion MNIST yields a validation loss of $\ln(10)$.) However, the non-smoothed training loss of the initialization $\theta_{\textrm{init}}$ for this problem is at around $84.4$. Thus we show the loss range of $[0.5, 100]$ on the y-axis in Figure~\ref{app_fig:lopt_ad_es_comparison_fashion_mnist}. Regardless of the y-axis range choice, our proposed method $\NRES$ performs the best among all the gradient estimation methods considered in this paper.

\clearpage
\begin{figure}[h]
    \centering
    \includegraphics[width=0.5\textwidth]{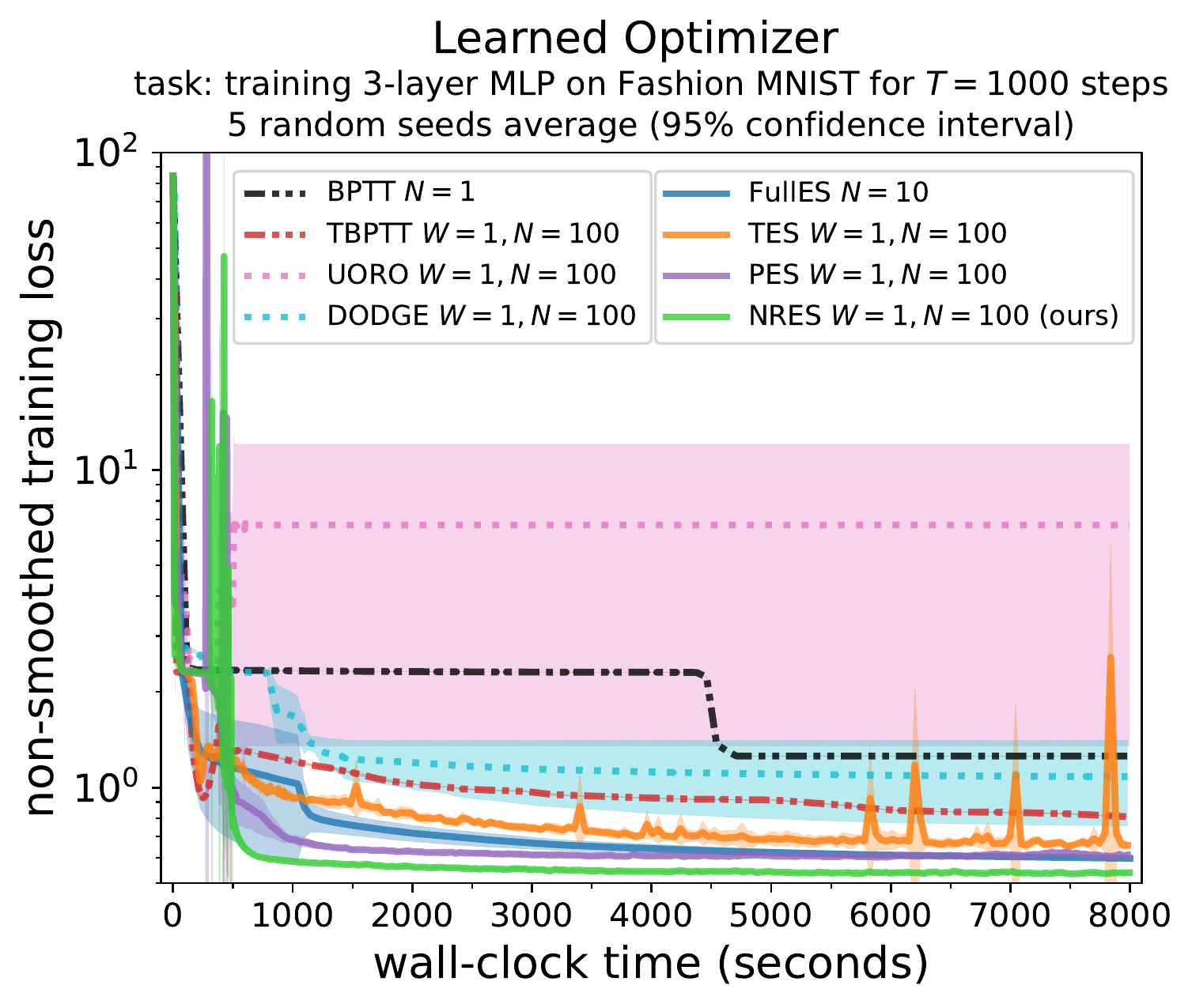}
    \caption{Different automatic differentiation and evolution strategies methods' loss convergence on the meta-learning a learned optimizer to train a 3-layer MLP on Fashion MNIST for $T=1000$ steps. AD methods ($\BPTT$, $\TBPTT$, $\UORO$, $\DODGE$) all perform worse than ES methods, justifying the need for ES methods on problems with such highly-sensitive loss surfaces. Among all the methods, our proposed method $\NRES$ converges the fastest. We notice that the two best performing methods $\PES$ and $\NRES$ experience some loss fluctuation right below $500$ seconds. Because such fluctuation never occurs for any other tasks we experiment with, we believe this is a property of this specific application itself but not the problem of the two methods.}
    \label{app_fig:lopt_ad_es_comparison_fashion_mnist}
\end{figure}

\paragraph{Results on the learned optimizer task with a higher dimension.} In Section~\ref{exp:lopt}, the learned optimizer we have considered has a parameter dimension $d=1762$. Here we compare ES gradient estimators on 
the same meta-training task but with an $11\times$ larger learnable parameter ($\theta$) dimension in Figure~\ref{app_fig:lopt_es_large_d_long_T}(a). To increase the parameter dimension, we increase the width of the multilayer perceptron used by the learned optimizer. On this higher-dimensional problem, $\NRES$ can still provide a $3.8\times$ wall clock time speed up over $\PES$ and a $9.7\times$ wall clock time speed up over $\FullES$ to reach a loss value which $\NRES$ reaches early on during the meta-training.

\paragraph{Results on the learned optimizer task with a longer horizon.} In Section~\ref{exp:lopt}, the learned optimizer we have considered has an inner-problem training steps of $T=1000$. Here we make the problem horizon $10 \times $ longer and show the training convergence of different ES estimators in Figure~\ref{app_fig:lopt_es_large_d_long_T}(b). For this problem, $\NRES$ still achieves a $2.1\times$ wall clock speed up over PES and a $3.5 \times$ speed up over FullES to reach a loss value which $\NRES$ reaches early on during the meta-training. 

\begin{figure}[h]
     \centering
    \begin{overpic}[width=0.45\textwidth, percent]{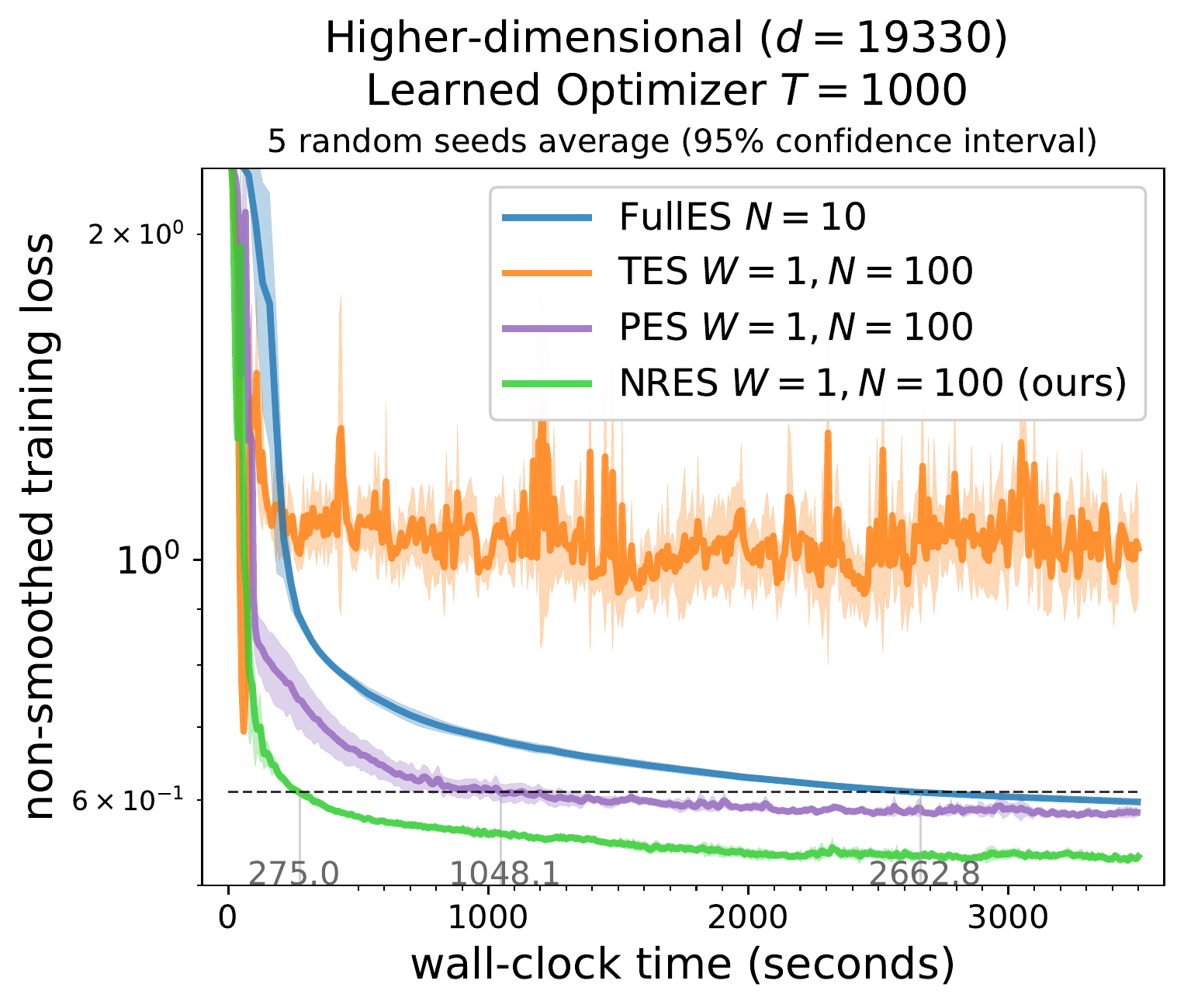}
     \put(7.0,74){\textbf{(a)}}
    \end{overpic}
    \begin{overpic}[width=0.45\textwidth, percent]{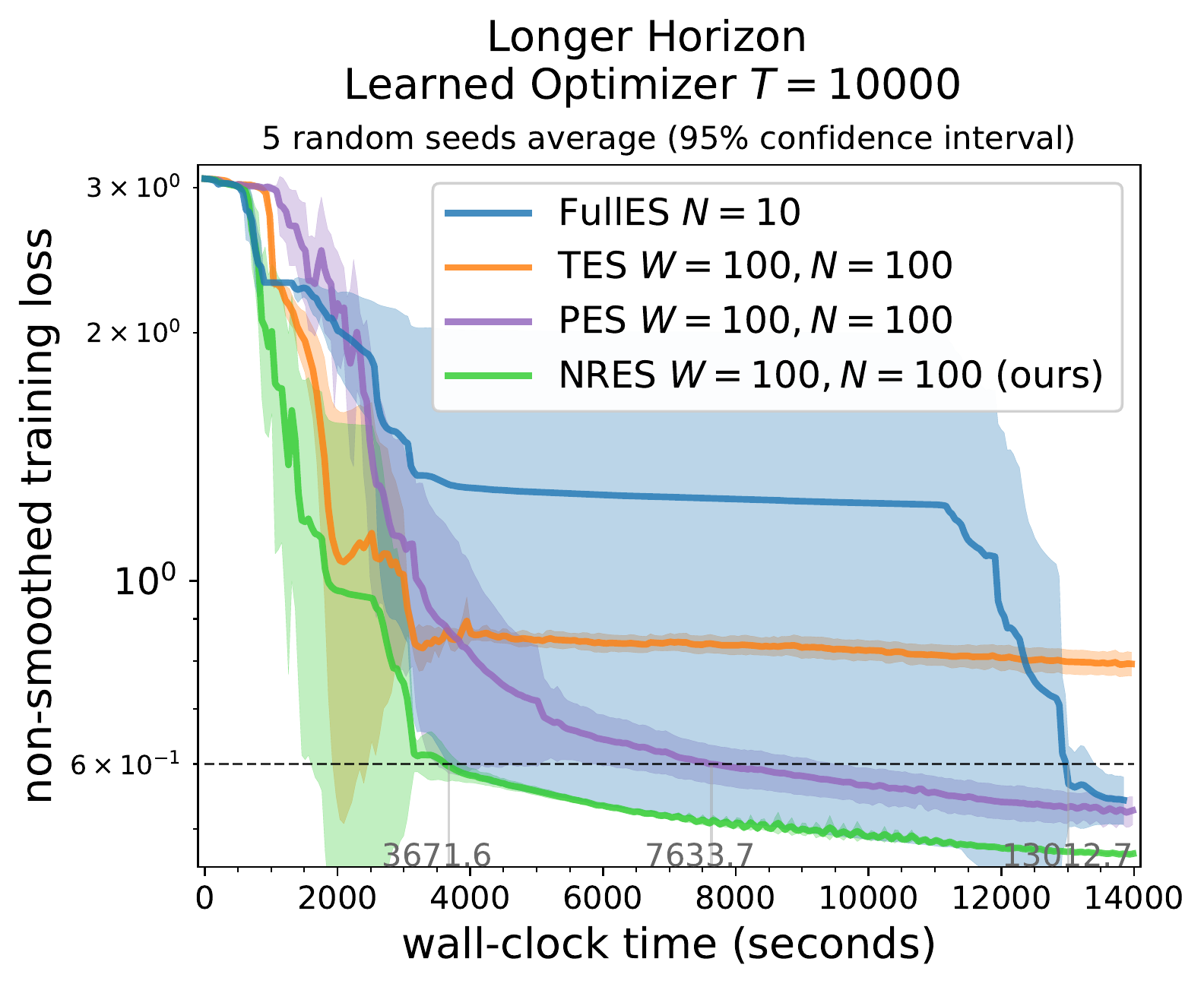}
     \put(7.0,74){\textbf{(b)}}
    \end{overpic}
    \caption{(a) Comparing different ES gradient estimators' training loss convergence (in wall-clock time) on a $11 \times$ higher-dimensional learned optimizer task ($d=19330$) than Figure~\ref{fig:lopt}(a) in the main paper. (b) Comparing different ES gradient estimators' training loss convergence (in wall-clock time) on a learned optimizer task with $10\times$ longer horizon $T=10000$. For both cases $\NRES$ still converges much faster than other methods.}
    \label{app_fig:lopt_es_large_d_long_T}
\end{figure}

\paragraph{Results on another learned optimizer task.} In addition to the learned optimizer task shown in the main paper, we experiment with another task from the training task distribution of VeLO~\cite{metz2022velo} to further compare the performance of different evolution strategies methods. Here we meta-learn the learned optimizer model from~\cite{metz2019understanding} to train a 4-layer fully Convolutional Neural Network on CIFAR-10 \citep{krizhevsky2009learning}. As shown in Figure~\ref{app_fig:lopt_es_comparison_cifar}, we see that $\TES$ fails to converge to to the same loss level as the other three methods. Among the rest of the methods, $\NRES$ can achieve a speed up of more than $2\times$ and $5\times$ over $\PES$ and $\FullES$ respectively to reach the non-smoothed training loss of $1.55$ given perfect parallel implementations of each method.
\begin{figure}[h]
    \centering
    \includegraphics[width=0.5\textwidth]{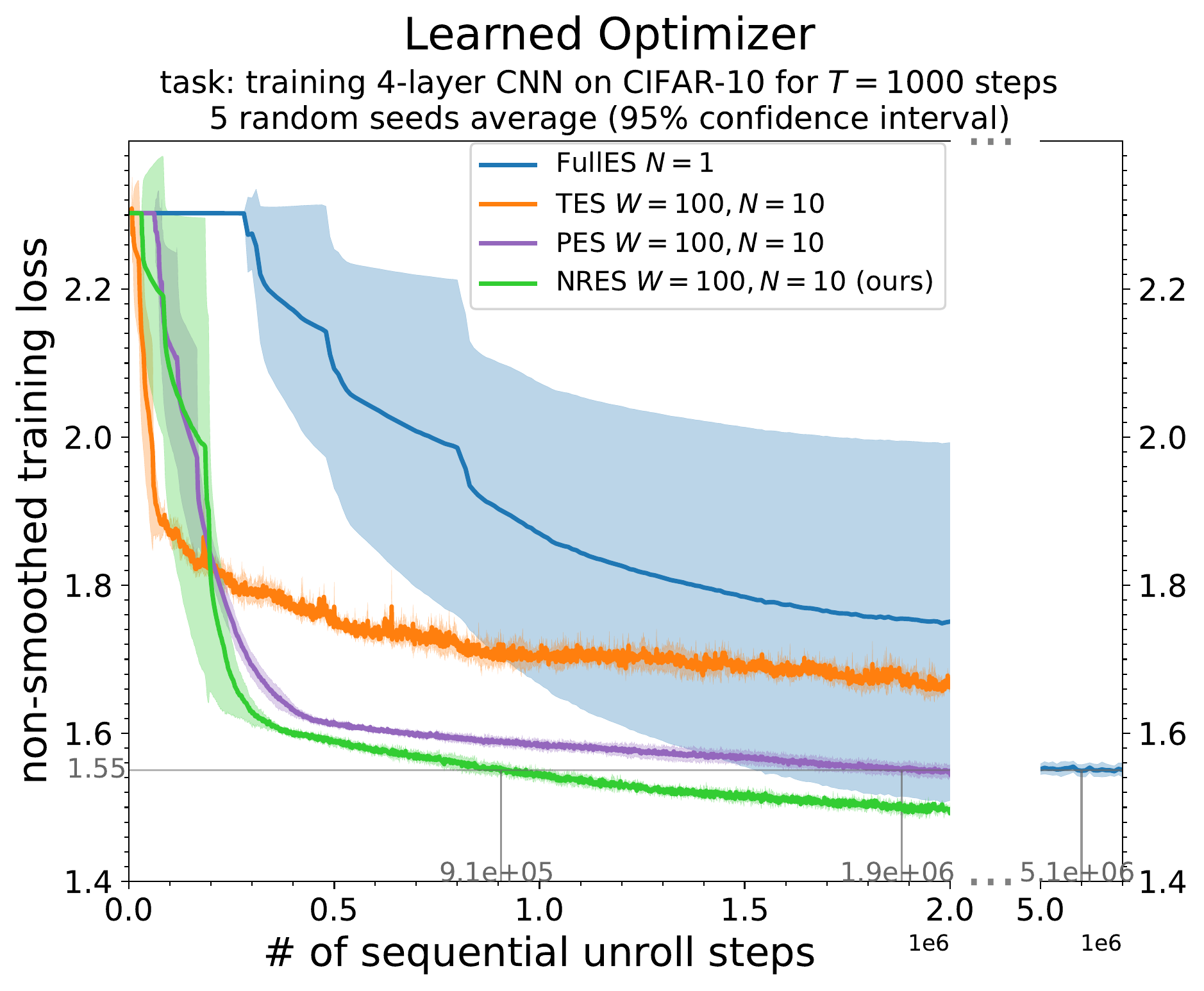}
    \caption{Different automatic differentiation and evolution strategies methods' loss convergence on the meta-learning learned optimizer to train a 4-layer fully convolutional neural network on CIFAR-10 for $T=1000$ steps. Because $\FullES$ takes a lot more sequential steps than $\PES$ and $\NRES$ to reach the loss value of $1.55$, we make the x-axis skip the values between $[2.0 \times 10^6, 5.0 \times 10^6]$ to show the number of steps $\FullES$ reaches the same value. Among all the ES methods, our proposed method $\NRES$ converges the fastest by using the smallest number of sequential unroll steps to reach the same loss value.}
    \label{app_fig:lopt_es_comparison_cifar}
\end{figure}

\vspace{-0.5em}
\subsubsection{Reinforcement Learning}
\label{app_subsubsec:rl}

\paragraph{Total number of environment steps to solve the Mujoco tasks.} We show in Figure~\ref{fig:rl_reward_progression} that $\NRES$ can solve the two Mujoco tasks using the least number of sequential environment steps, which indicates $\NRES$ can solve the tasks using the shortest amount of wall clock time under perfect parallelization. Here we additionally show the \textit{total} number of environment steps used by each ES methods to solve the two tasks in Table~\ref{app_tab:rl_total_env_steps}. We see that $\NRES$ uses the least total number of environment steps and thus is the most sample efficient among the methods compared.
{
\setlength{\extrarowheight}{2pt} %
\begin{table}[h]
    \centering
    \small
    \caption{Total number of environment steps used by each ES method considered in this paper. $\TES$ is unable to solve either task, while $\PES$ struggles to solve the Half Cheetah task despite we allowing it to use a significantly larger number of environment steps. In contrast, $\NRES$ (ours) solve both tasks using the least number of total environment steps, being $26\%$ and $22\%$ more sample efficient than $\FullES$ for the two tasks respectively.}
    \label{app_tab:rl_total_env_steps}
    \begin{tabular}{r|c|c|c|c}
        & \multicolumn{4}{c}{total number of environment steps} \\
        & \multicolumn{4}{c}{used to solve the Mujoco task} \\
        & \multicolumn{4}{c}{{\scriptsize (averaged over 5 random seeds)}} \\
        \cline{2-5}
        Mujoco task & $\FullES$ & $\TES$ & $\PES$ & $\NRES$ (ours)\\
        \hline
        \hline
        Swimmer &  $1.50 \times 10^5$ &  not solved & $5.85 \times 10^5$ & $\mathbf{1.11 \times 10^5}$ \\
        \hline
        Half Cheetah & $7.54 \times 10^6$ & not solved & $> 3.60 \times 10^8$ & $\mathbf{5.81 \times 10^6}$ \\
    \end{tabular}
\end{table}
}

\paragraph{Results on non-linear policy learning on the Half Cheetah task.} In Section~\ref{exp:rl}, we compare ES gradient estimators on training linear policies on the Mujoco tasks. Here we compare these estimators' performance in training a non-linear ($d=726$) policy network on the Mujoco Half Cheetah task under a fixed budget of total number of environment steps in Figure~\ref{app_fig:rl_half_cheetah_mlp_policy}. Only our proposed method NRES solves the task under the computation budget.

\begin{figure}[h]
     \centering
     \includegraphics[width=0.5\textwidth]{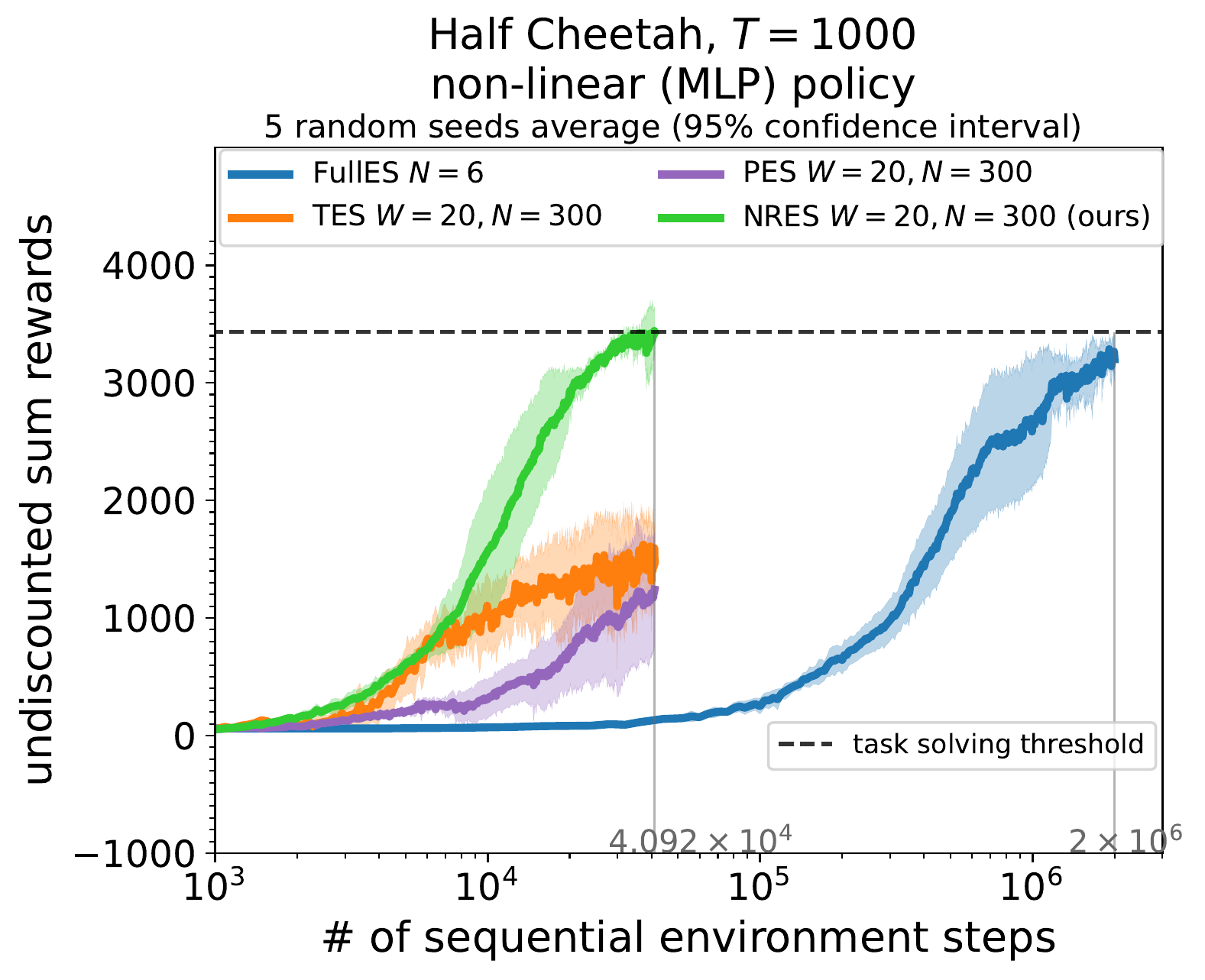}
    \caption{Comparing different ES methods’ performance in training the nonlinear MLP policy architecture used in TRPO \cite{pmlr-v37-schulman15} on the Mujoco Half Cheetah task under a fixed number of total environment steps. Among all the ES methods, only $\NRES$ solves the task within the budget while also offering close to $50\times$ parallelization speed up compared to $\FullES$.}
    \label{app_fig:rl_half_cheetah_mlp_policy}
\end{figure}

\paragraph{Ablation on the impact of $N$ (number of workers) on $\NRES$'s performance.} We perform an ablation study on the impact of the number of $\NRES$ workers on its performance in solving the Mujoco Swimmer task in Figure~\ref{app_fig:rl_swimmer_ablation_on_N}. Here, increasing $N$ can help $\NRES$ use fewer sequential steps to solve the task but at a larger per sequential step compute cost.
 \begin{figure}[h]
     \centering
    \includegraphics[width=0.48\textwidth, percent]{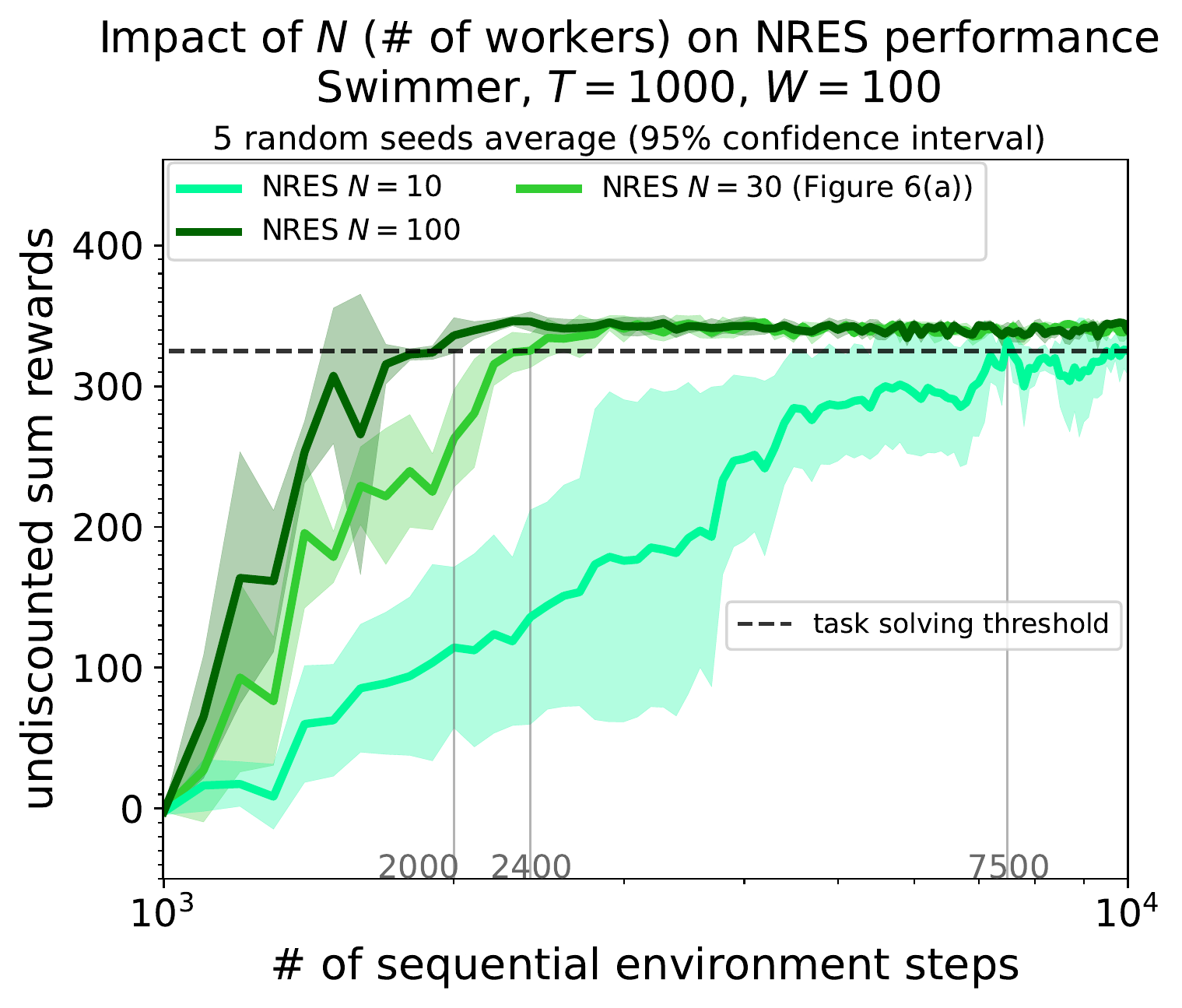}
    \caption{The impact of the number of parallel workers $N$ on $\NRES$'s performance on the Mujoco Swimmer task. We see that using more parallel workers can reduce the number of sequential environment steps needed to solve the task because the optimization benefits from a reduced gradient estimate's variance which scales with $1/N$.}
    \label{app_fig:rl_swimmer_ablation_on_N}
\end{figure}

\paragraph{Ablation on the impact of $\sigma$ on $\NRES$ and $\FullES$'s performance.} We perform an ablation study on the impact of the noise variance $\sigma^2$ on the performance of our proposed method $\NRES$ and $\FullES$ in solving the Mujoco Half Cheetah task in Figure~\ref{app_fig:rl_half_cheetah_ablation_on_sigma}. While setting $\sigma$ too small provides insufficient amount of smoothing and makes both methods fail to to solve the task, there still exists a range of larger $\sigma$ values under which both methods can solve the task successfully. For these cases, NRES always achieves a more than $50\times$ reduction in the number of sequential steps used over FullES.

\begin{figure}[h]
     \centering
    \begin{overpic}[width=0.48\textwidth, percent]{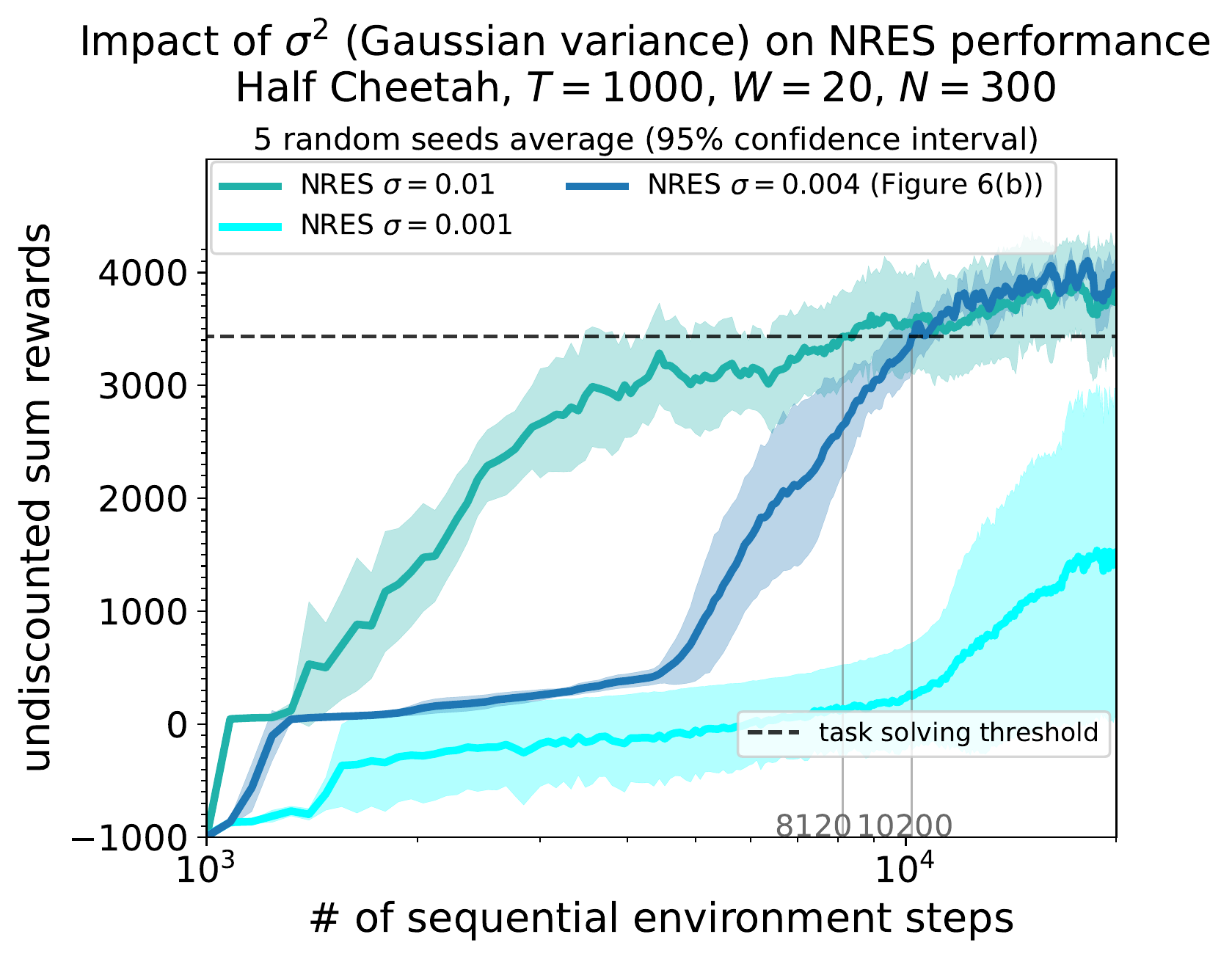}
     \put(-3.0,68){\textbf{(a)}}
    \end{overpic}
    \begin{overpic}[width=0.48\textwidth, percent]{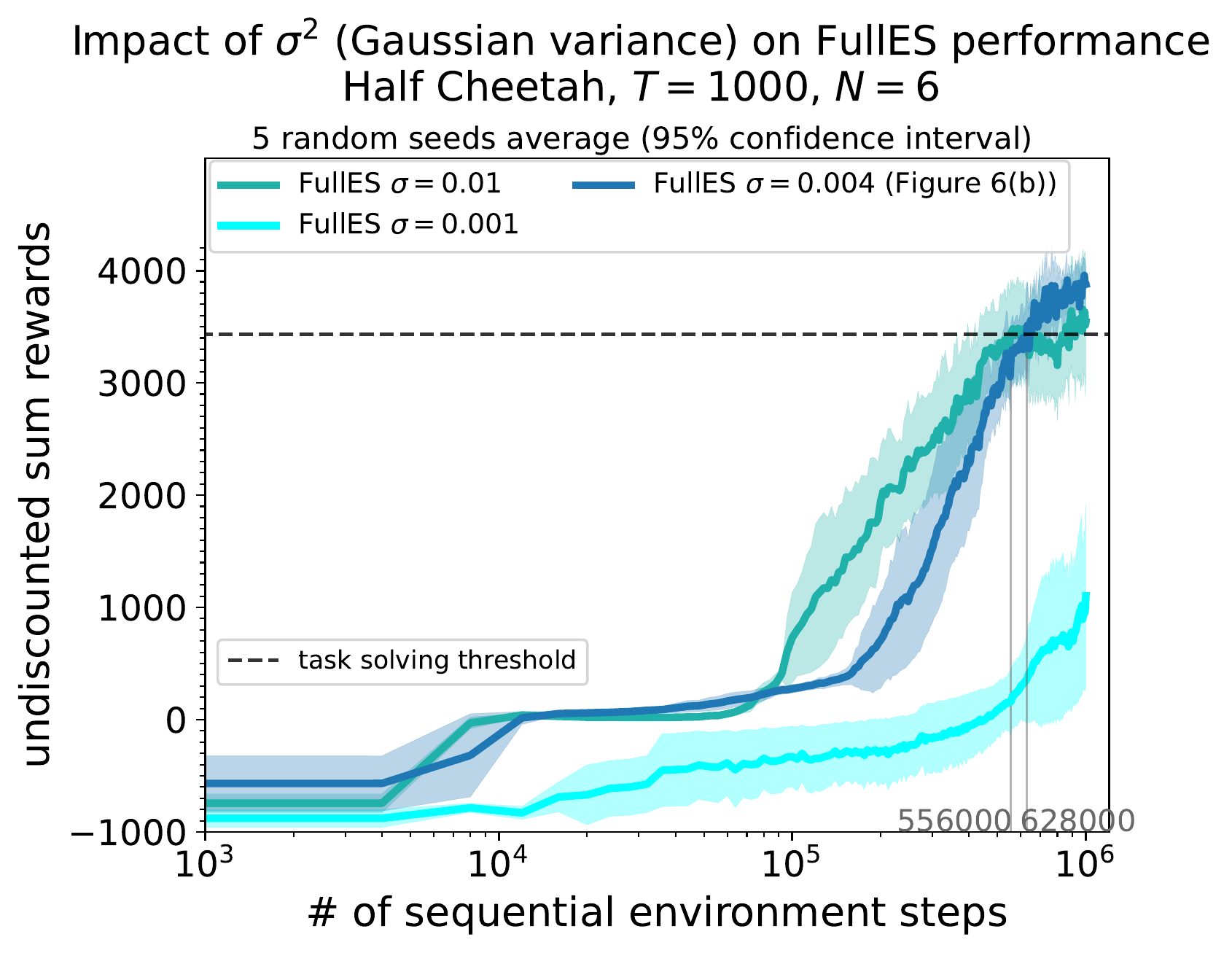}
     \put(-3.0,68){\textbf{(b)}}
    \end{overpic}
    \caption{The impact of the hyperparameter $\sigma^2$ (smoothing Gaussian distribution's isotropic variance) on (a)$\NRES$'s and (b)$\FullES$'s performance on the Mujoco Half Cheetah task. Although insufficient amount of loss smoothing ($\sigma = 0.001$) could lead to slow convergence, there exists a range of hyperparameters ($\sigma=0.004$ and $\sigma=0.01$) that can allow both ES methods to solve the task. For both such cases, $\NRES$ improves significantly over $\FullES$ by more than $50\times$.}
    \label{app_fig:rl_half_cheetah_ablation_on_sigma}
\end{figure}

\clearpage
\subsection{Experiment details and hyperparameters}

\label{app_subsec:hypermeters}
\subsubsection{Learning Lorenz dynamical system parameters}
On the Lorenz system parameter learning task, we use the vanilla SGD optimizer to update the parameter $\theta = (\ln(r), \ln(a))$ starting at $\theta_{\textrm{init}}=(\ln(r_{\textrm{init}}), \ln(a_{\textrm{init}})) = (3.7, 3.116)$ with the episode length at $T=2000$.

\begin{itemize}
    \item For non-online methods, we use only $N=1$ worker for $\BPTT$ to compute the non-smoothed true gradient since we only have one example sequence to learn from. For $\FullES$, we use $N=10$ workers.
    \item For all the online methods, we use $N=200$ workers and a truncation window of size $W=100$.  This relationship exactly matches the condition considered in Theorem \ref{thm:nres_fulles_comparison}, and the total amount of computation for $\FullES$ and $\NRES$ to produce one gradient estimate is roughly the same.
    \item For all the ES methods, we use the smoothing standard deviation $\sigma = 0.04$ chosen by first tuning it on $\FullES$.
\end{itemize}
  For each gradient estimation methods, we tune its SGD constant learning rate from the following set
\begin{align}
    \{10^{-3}, 3\times 10^{-4}, 10^{-4}, 3\times 10^{-5}, 10^{-5}, 10^{-6}, 10^{-7}, 10^{-8}, 10^{-9}, 10^{-10}, 10^{-11}, 10^{-12}, 10^{-13}\}.
\end{align}
Here we choose the learning rate to ensure that \textbf{1)} there is no NaN in gradients/loss value due to the exploding gradient from the loss surface; \textbf{2)} the learning rate doesn't result in a significant increase in the loss. The learning rate range is made so wide because for all the unbiased AD methods, \textit{these two issues would occur unless we use a trivially small learning rate}. Among the learning rates that pass these two requirements, we choose the best learning rate such that the optimization metric of interest decreases the fastest. For each method, the tuned learning rate used for Figure \ref{fig:lorenz}(a) is given in Table~\ref{tab:lorenz_hyperparameters}. As $\PES$ couldn't afford to use the same learning rate as $\NRES$ because of the unstable convergence (see Figure \ref{fig:lorenz}(a)), we hand-tuned a learning rate decay schedule for $\PES$ to maximally allow for its convergence.

\begin{table}[h!]
    \centering
    \small
    \caption{Learning rates (schedules) used for different gradient estimators on the Lorenz system parameter learning task}
    \begin{tabular}{l|r}
        method name & SGD learning rate (schedule) \\
            \hline
        $\BPTT \; N=1$ &  $10^{-8}$\\
        $\TBPTT \; N=200, W=100$ &  $3\times 10^{-4}$\\
        $\DODGE \; N=200, W=100$ & $10^{-10}$\\
        $\UORO \; N=200, W=100$ & $10^{-13}$ \\
        $\FullES \; N=10$ & $3\times 10^{-5}$ \\
        $\TES \; N=200, W=100$ & $3\times 10^{-4}$\\
        $\PES \; N=200, W=100$ & $10^{-5}$ for the first 1000 updates; $10^{-6}$ afterwards\\
        $\NRES \; N=200, W=100$ & $10^{-5}$\\
    \end{tabular}

    \label{tab:lorenz_hyperparameters}
\end{table}

To generate Figure \ref{fig:lorenz}(a), we use a constant SGD learning rate of $10^{-5}$ to isolate the variance properties of different $\GPES_K$ estimators.

\subsubsection{Meta-training learned optimizers}
\label{app_subsubsec:lopt_hyperparameters}

\paragraph{Meta-training learned optimizer to train 3-layer MLP on Fashion MNIST.}
The inner problem model is a three-layer GeLU-activated \citep{hendrycks2016gaussian} MLP with 32 dimensions for each hidden layer. We choose GeLU instead of ReLU because GeLU is infinitely differentiable. By this design choice, the meta-training loss function is infinitely differentiable with respect to $\theta$, making it possible to consider automatic Differentiation methods for this task. We downsize the dataset Fashion MNIST \citep{xiao2017fashion} to 8 by 8 images to make both the unrolled computation graph small enough to fit onto a single GPU and also to make training fast. We split the Fashion MNIST training's first 80\% image for inner problem training and the last 20\% images for inner problem validation. During the inner training, we randomly sample a training batch of size 128 to compute the training gradient using the cross entropy loss. After updating the inner model with the learned optimizer, we evaluate the meta-loss (our optimization objective) using the cross entropy on a randomly sampled validation batch of size 128. We fix the data sequence in inner training and inner evaluation to focus on the optimization problem on a single unrolled computation graph. One inner problem training lasts $T=1000$ steps of updates using the learned optimizer. We aim to minimize the average meta loss over $T=1000$ steps. 

\begin{itemize}
    \item For offline methods, we only need $N=1$ worker for $\BPTT$ to compute the gradient exactly. For $\FullES$, we tuned its number of workers among the set $N \in \{1, 3, 10\}$ and choose $N=10$. 
    \item For all the online methods, we use the smallest possible truncation window size of $W=1$ to test the online algorithms to the extreme. We use a total of $N=100$ workers for all the online methods. Therefore, to produce one gradient update, our proposed online ES method $\NRES$ only runs $2 \cdot N \cdot W / T = 100 \cdot 1 / 1000 = 20\%$ of an episode in total.
    \item  For all the ES methods, we use the smoothing standard deviation at $\sigma = 0.01$ inspired by the hyperparameter choice made in \cite{metz2019understanding}.
\end{itemize}

When meta-training the learned optimizer, we use Adam with the default parameters (other than a tuneable learning rate) to optimize the meta-training loss. For each gradient estimation method, we tune their Adam learning rate in the following range
\begin{align}
    \{3 \times 10^{-2}, 1\times 10^{-2}, 3\times 10^{-3}, 10^{-3}, 3\times 10^{-4}, 10^{-4}, 3\times 10^{-5}\}
\end{align}

After choosing a learning rate for a particular method over a specific random seed, we also check its performance on another seed to ensure its consistency (otherwise we lower the learning rate and try again). For each method, the tuned learning rate is given in Table~\ref{tab:lopt_fashion_mnist_hyperparameters}.
\begin{table}[h!]
    \centering
    \small
    \caption{Tuned Learning rates used for different AD and ES gradient estimators on the meta-training learned optimizer task to train 3-layer MLP on Fashion MNIST for $T=1000$ steps.}
    \begin{tabular}{l|r}
        method name & Adam learning rate (schedule)  \\
        \hline
        $\BPTT \; N=1$ &  $3 \times 10^{-4}$\\
        $\TBPTT \; N=100, W=1$ &  $10^{-2}$\\
        $\DODGE \; N=100, W=1$ & $10^{-5}$\\
        $\UORO \; N=100, W=1$ & $3 \times 10^{-5}$ \\
        $\FullES \; N=10$ & $3\times 10^{-2}$ \\
        $\TES \; N=100, W=1$ & $10^{-3}$\\
        $\PES \; N=100, W=1$ & $3 \times 10^{-4}$ \\
        $\GPES_K \; N=100, W=1, K=4$ & $3 \times 10^{-4}$ \\
        $\GPES_K \; N=100, W=1, K=16$ & $3 \times 10^{-4}$ \\
        $\GPES_K \; N=100, W=1, K=64$ & $3 \times 10^{-4}$ \\
        $\GPES_K \; N=100, W=1, K=256$ & $3 \times 10^{-4}$ \\
        $\NRES \; N=100, W=1$ & $3 \times 10^{-4}$\\
    \end{tabular}
    \label{tab:lopt_fashion_mnist_hyperparameters}
\end{table}

\paragraph{Meta-training a learned optimizer to train 4-layer CNN on CIFAR-10.}

The inner problem model is a 4-layer ReLU-activated Fully Convolutional Neural Networks with 32 channels per layer produced by $(3, 3)$ filters with stride size of $1$ and same padding. The convolutional layer is followed with a $(2, 2)$ average pooling layer with stride size of $2$ and valid padding. After the 4-th convolutional block, we flatten the features and apply an affine transformation to predict the logits over the 10 classes in CIFAR-10. We split the CIFAR-10 \citep{krizhevsky2009learning} dataset with the first 80\% images from its training set for inner problem training and the last 20\% images for inner problem validation. During the inner training, we randomly sample a training batch of size 32 to compute the training gradient using the cross entropy loss. After updating the inner model with the learned optimizer, we evaluate the meta-loss (our objective) using cross entropy on a randomly sampled validation batch of size 32. Following \cite{metz2022practical}, we apply thresholding to each step's meta loss to cap it at $1.5 \times \ln(10)$. Without this thresholding, the meta losses could become greater than $10^4$. This choice of thresholding introduces discontinuity to the loss function, making automatic differentiation inappropriate for this task. One inner problem training lasts $T=1000$ steps of updates using the learned optimizer. We aim to minimize the average meta validation loss over $T=1000$ steps. We fix the sequence of training and validation batches to isolate the problem to learning on a single unrolled computation graph instead of a distribution of UCGs. We compare among different ES methods on this task:

\begin{itemize}
    \item For $\FullES$, we set the number of workers $N = 1$ as we observe that it could already optimize the meta losses efficiently. 
    \item For all the online ES methods, we use the truncation window size of $W=100$. To match the total number of unroll steps used by $\NRES$ with $\FullES$, we use $N=10$ for all online ES methods ($\TES$, $\PES$, $\NRES$).
    \item  For all the ES methods, we use the smoothing standard deviation at $\sigma = 0.01$ inspired by the hyperparameter choice made in \cite{metz2019understanding}.
\end{itemize}

When meta-training the learned optimizer, we use Adam with the default parameters (other than a tuneable learning rate) to optimize the meta-training loss. For each gradient estimation method, we tune the learning rate in the following range
\begin{align}
    \{10^{-3}, 3\times 10^{-4}, 10^{-4},\}
\end{align}
For each ES method, the tuned learning rate is given in Table~\ref{tab:lopt_cifar_hyperparameters}.

\begin{table}[h!]
    \centering
    \small
    \caption{Tuned Learning rates used for different ES gradient estimators on the meta-training learned optimizer task to train 4-layer CNN on CIFAR-10 for $T=1000$ steps}
    \begin{tabular}{l|r}
        method name & Adam learning rate (schedule)  \\
        \hline
        $\FullES \; N=1$ & $10^{-3}$ \\
        $\TES \; N=10, W=100$ & $10^{-3}$\\
        $\PES \; N=10, W=100$ & $3 \times 10^{-4}$ \\
        $\NRES \; N=10, W=100$ & $3 \times 10^{-4}$\\
    \end{tabular}
    \label{tab:lopt_cifar_hyperparameters}
\end{table}

\subsubsection{Reinforcement Learning}
We use the Mujoco tasks Swimmer-v4 and Half Cheetah-v4 from Open AI gym (gymnasium version) \citep{brockman2016openai} with default settings. %
Here, following \citep{mania2018simple}, we learn a deterministic linear policy that maps the observation space to the action space: for Swimmer-v4, this amounts to mapping an 8-dimensional observation space to a 2-dimensional action space ($d = 16)$, while for Half Cheetah-v4, this amounts to mapping a 17-dimensional observation space to a 6-dimensional action space ($d = 102$). Because we can't differentiate through the dynamics, we only consider ES methods for this task.

Regarding the hyperparameters used for the Swimmer task,
\begin{itemize}
    \item We use $N=3$ $\FullES$ workers.
    \item For all the online ES methods, we fix the truncation window size at $W=100$ and use $N = 30$ workers. This choice makes $\NRES$ and $\PES$ use the same number of total environment steps as $\FullES$.
    \item We choose $\sigma=0.3$ for all ES methods after tuning it for $\FullES$.
\end{itemize}

For Swimmer-v4, we tune the learning rate used by the SGD algorithm in the following range:
\begin{align}
    \{10^{2}, 3 \times 10^1, 10^{1}, 3 \times 10^{0}, 10^{0}, \}
\end{align}

The tuned learning rates for each method on the Swimmer task are given in Table~\ref{tab:swimmer_hyperparameters}. 

\begin{table}[h!]
    \centering
    \caption{Tuned learning rates used for different ES gradient estimators on the Mujoco Swimmer-v4 task}
    \begin{tabular}{l|r}
        method name & SGD learning rate \\
        \hline
        $\FullES \; N=3$ & $1 \times 10^{0}$ \\
        $\TES \; N=30, W=100$ & $3\times 10^{1}$\\
        $\PES \; N=30, W=100$ & $1 \times 10^{0}$ \\
        $\NRES \; N=30, W=100$ & $3 \times 10^{0}$\\
    \end{tabular}
    \label{tab:swimmer_hyperparameters}
\end{table}

Regarding the hyperparameters used for the Half Cheetah task,
\begin{itemize}
    \item We use $N=6$ $\FullES$ workers.
    \item For the online ES methods, we fix the truncation window size at $W=20$ and use $N = 300$ workers. This choice makes $\NRES$ and $\PES$ use the same number of total environment steps as $\FullES$.
    \item We choose $\sigma=0.004$ for all ES methods after tuning it for $\FullES$.
\end{itemize}

For Half Cheetah-v4, we tune the learning rate used by the SGD algorithm in the following range:
\begin{align}
    \{10^{-4}, 3 \times 10^{-5}, 10^{-5}, 3\times 10^{-6}\}
\end{align}
The tuned learning rates for each method on the Half Cheetah task are given in Table~\ref{tab:halfcheetah_hyperparameters}. We notice that $\PES$'s learning rate is tuned to be much smaller than the other methods because it suffers from high variance.
\clearpage
\begin{table}[h!]
    \centering
    \small
    \caption{Tuned learning rates used for different ES gradient estimators on the Mujoco Half Cheetah-v4 task}
    \begin{tabular}{l|r}
        method name & SGD learning rate \\
        \hline
        $\FullES \; N=6$ & $3 \times 10^{-5}$ \\
        $\TES \; N=100, W=100$ & $3 \times 10^{-5}$\\
        $\PES \; N=100, W=100$ & $3 \times 10^{-6}$ \\
        $\NRES \; N=100, W=100$ & $3 \times 10^{-5}$\\
    \end{tabular}
    \label{tab:halfcheetah_hyperparameters}
\end{table}

\subsection{Computation resources}
For any gradient estimation method on each of the three applications, we can train on a GPU machine with a single Nvidia GeForce RTX $3090$ GPU. As shown in the Figure~\ref{fig:lorenz} (Lorenz) and Figure~\ref{fig:lopt_full_comparison}(a) (learned optimizer) , $\NRES$ can converge in wall-clock time in less than $150$ and $2500$ seconds respectively. For the reinforcement learning mujoco tasks Swimmer and Half Cheetah, because the environment transitions are computed on CPUs in OpenAI Gym, we can run the entire experiment on CPU and it takes less than 2 hours and 10 hours respectively to finish training for all the ES methods. We believe these tasks are valuable experiment set ups to analyze new online evolution strategies methods in the future.

\subsection{Experiment implmentation}

Our codebase is inspired by the high level logic in the codebase by \citep{metz2022practical}. We provide our code implementation in JAX \cite{jax2018github} at \href{https://github.com/OscarcarLi/Noise-Reuse-Evolution-Strategies}{https://github.com/OscarcarLi/Noise-Reuse-Evolution-Strategies}.

\section{Broader Impacts and Limitations of NRES}
\label{app_sec:limitations}
As we have shown through multiple applications in Section \ref{sec:exps}, $\NRES$ can be an ideal choice of gradient estimator for unrolled computation graphs when the loss surface has extreme local sensitivity or is black-box/non-differentiable. In addition to being simple to implement, we show that NRES results in faster convergence than existing AD and ES methods in terms of wall-clock time and number of unroll steps across a variety of applications, including learning dynamical systems, meta-training learned optimizers, and reinforcement learning. These efficiency gains and lower resource requirements can in turn reduce the environmental impact and cost required to enable such applications.

However, just as any other useful tool in the modern machine learning toolbox, we note that $\NRES$ should only be used in  appropriate scenarios. Here, we additionally discuss limitations of $\NRES$ (which are often limitations shared by a class of methods $\NRES$ belongs to) and in what scenarios it might not be the first choice to use.

\textbf{Limitations of $\NRES$ as an evolution strategies method.} Here we discuss the limitations of $\NRES$ as an ES method (both limitations below are common to all ES methods).
\begin{itemize}[leftmargin=*]
    \item \textbf{[Dependence on $d$]} As all the evolution strategies methods only have access to zeroth-order information, their variance has a linear dependence on the parameter dimension $d$ (see the first term in Equation~\eqref{eq:variance} in Theorem~\ref{thm:variance}). As a result, the most ideal use cases of ES methods (including $\NRES$) are when the parameter dimension $d$ is lower than $100,000$. As we have seen in Section \ref{sec:exps}, many interesting applications would have parameter dimensions in this range. However, even in cases when $d$ is larger than this range, if the loss surface is non-differentiable or has high local sensitivity, one might still need to use $\NRES$ because AD methods might either be unable to compute the gradient or give noninformative gradients for effective optimization. On the other hand, AD methods (specifically Reverse Mode Differentiation methods $\BPTT$ and $\TBPTT$) should be the first choice when the parameter dimension is large and the loss surface is differentiable and well-behaved, because their variances do not suffer from a linear dependence on $d$.
    
    \item \textbf{[Dependence on $\sigma$]} The variance $\sigma$ of the smoothing isotropic Gaussian distribution for $\bepsilon$ is an extra hyperparameter ES methods introduce. In contrast, stochastic forward mode method $\DODGE$ doesn't have such a hyperparameter and can instead always sample from the standard Gaussian distribution (i.e. $\sigma = 1$). Assuming the loss surface is locally quadratic, differentiable, and well-behaved, $\DODGE$ (with standard Gaussian distribution) would have similar variance properties as $\NRES$ and is a preferable choice than $\NRES$ among stochastic gradient estimation methods due to its lack of hyperparameters to tune. On the other hand, as we have seen in Figure~\ref{app_fig:lorenz_ad_es_comparison} and \ref{app_fig:lopt_ad_es_comparison_fashion_mnist}, when the loss surface are non-smooth and have many local fluctuations, $\DODGE$ would perform much worse than $\NRES$, because $\DODGE$ unbiasedly estimates the pathological loss surface's true gradient without smoothing.
\end{itemize}

\textbf{Limitations of $\NRES$ as an online method.} Next we discuss two limitations of $\NRES$ as an \textit{online} gradient estimation method for unrolled computation graphs. The limitations are not specific to $\NRES$ but apply more broadly to classes of online methods.
\begin{itemize}[leftmargin=*]
    \item \textbf{[Hysteresis]}
    Any online gradient estimation methods (including $\TBPTT$, $\DODGE$, $\UORO$, $\TES$, $\PES$, and $\NRES$) would have bias due to hysteresis (see Section~\ref{sec:background}) because the parameter $\theta$ changes across adjacent partial online unrolls in the same UCG episode. Although we do not observe much impact of this bias in the applications and hyperparamter combinations considered in this paper, it is conceivable for some scenarios it could become an issue. In those cases, one should consider the tradeoffs of using $\NRES$ versus using $\FullES$ -- $\NRES$ allows more frequent gradient updates than $\FullES$ while $\FullES$ avoids the bias from hysteresis.
    
    \item \textbf{[$T$'s potential dependence on $\theta$]} In this paper, we assume the unrolled computation graph is of length $T$, where $T$ doesn't depend on the value of the parameter $\theta$. However, for some UCGs, an episode's length might be conditioned on what value of $\theta$ is used. For example, consider a robotic control problem where the robot could physically break when some action is taken. In this case, a prespecified episode length might not be reached. In Mujoco, these types of termination are called unhealthy conditions. This is potentially problematic for antithetic online ES methods (including $\TES$, $\PES$, and $\NRES$) when only one of the two antithetic trajectories terminate early but not the other. In this case, one heuristic would be to directly reset both particles, but by doing this, one would ignore this useful information about the difference in antithetic directions. Besides, the early termination of some workers might also make the distribution of workers' truncation windows not truly uniform over the entire episode's length $[T]$. How to gracefully handle these cases (through better heuristics/principled approaches) would be an interesting question of future work to make online antithetic ES methods (especially $\NRES$) more applicable for these problems.
\end{itemize}

\end{document}